\documentclass[letterpaper]{article} % DO NOT CHANGE THIS
\usepackage{aaai2026}  % DO NOT CHANGE THIS
\usepackage{times}  % DO NOT CHANGE THIS
\usepackage{helvet}  % DO NOT CHANGE THIS
\usepackage{courier}  % DO NOT CHANGE THIS
\usepackage[hyphens]{url}  % DO NOT CHANGE THIS
\usepackage{graphicx} % DO NOT CHANGE THIS
\usepackage{natbib}  % DO NOT CHANGE THIS AND DO NOT ADD ANY OPTIONS TO IT
\usepackage{caption} % DO NOT CHANGE THIS AND DO NOT ADD ANY OPTIONS TO IT
\usepackage{algorithm}
\usepackage{algorithmic}

\usepackage{newfloat}
\usepackage{listings}
\DeclareCaptionStyle{ruled}{labelfont=normalfont,labelsep=colon,strut=off} % DO NOT CHANGE THIS
\floatstyle{ruled}
\newfloat{listing}{tb}{lst}{}
\floatname{listing}{Listing}
\graphicspath{{./pics}}

\usepackage[caption=false]{subfig}
\usepackage{amsthm,amsmath,amssymb}
\usepackage{booktabs}
\usepackage{multirow}
\usepackage[capitalise,nameinlink]{cleveref}
\usepackage{makecell}
\usepackage{thm-restate}

\usepackage{tabularray}
\UseTblrLibrary{booktabs}

\declaretheorem[
name=Theorem,
numberwithin=section,
]{theorem}

\declaretheorem[%
name=Lemma,
sibling=theorem
]{lemma}

\declaretheorem[%
name=Corollary,
sibling=theorem
]{corollary}

\theoremstyle{definition}
\newtheorem{definition}[theorem]{Definition}

\theoremstyle{remark}

\newcommand{\E}[1]{\mathbb{E}(#1)}

\newcommand{\OurMethod}{SCS\ }

\newcommand{\Real}{\mathbb{R}}
\newcommand{\dd}{\mathrm{d}}
\newcommand{\CV}{\mathcal{V}}

\newcommand{\CN}{\mathcal{N}}
\newcommand{\bX}{\mathbf{X}}

\title{A Theoretical Analysis of Detecting Large Model-Generated Time Series}
\author{
    Junji Hou,
    Junzhou Zhao\thanks{Corresponding Author.},
    Shuo Zhang,
    Pinghui Wang
}
\affiliations{
    MoE KLINNS Lab,
    Xi'an Jiaotong University, Xi'an 710049, P.~R.~China\\
    \{15955192, zs412082986\}@stu.xjtu.edu.cn, \{junzhou.zhao,phwang\}@xjtu.edu.cn
}

\begin{document}

\maketitle

\begin{abstract}
Motivated by the increasing risks of data misuse and fabrication, we investigate
the problem of identifying synthetic time series generated by Time-Series Large
Models (TSLMs) in this work.
While there is extensive research on detecting model generated text, we find
that these existing methods are not applicable to time series data due to the
fundamental modality difference, as time series usually have lower information
density and smoother probability distributions than text data, which limit the
discriminative power of token-based detectors.
To address this issue, we examine the subtle distributional differences between
real and model-generated time series and propose \emph{contraction
  hypothesis}, which states that model-generated time series, unlike real ones,
exhibit progressively decreasing uncertainty under recursive forecasting.
We formally prove this hypothesis under theoretical assumptions on model
behavior and time series structure.
Model-generated time series exhibit progressively concentrated distributions
under recursive forecasting, leading to uncertainty contraction.
We provide empirical validation of the hypothesis across diverse datasets.
Building on this insight, we introduce the \emph{Uncertainty Contraction
  Estimator (UCE)}, a white-box detector that aggregates uncertainty metrics
over successive prefixes to identify TSLM‑generated time series.
Extensive experiments on $32$ datasets show that UCE consistently outperforms
state-of-the-art baselines, offering a reliable and generalizable solution for
detecting model-generated time series.

\end{abstract}

\section{Introduction}

Recent advances in time series forecasting have given rise to Time Series Large
Models (TSLMs), which are pre-trained on massive multi-domain time series
datasets with billions of parameters~\cite{Chronos,timer,TimeMoE}.
The vast training data and enormous parameter scale enable TSLMs to achieve
remarkable long-term zero-shot forecasting on entirely new datasets or domains
without any labeled examples or task-specific fine-tuning, as evidenced by
recent scaling law analyses~\cite{scalinglaw}.
Leveraging these sophisticated forecasting capabilities, TSLMs have demonstrated
strong performance in domains such as finance, Internet of Things (IoT), and
climate science.

The powerful capability of TSLMs raises significant concerns about potential
data fabrication or misuse.
Unlike classical approaches (e.g., ARIMA, LSTM), which degrade in performance
outside their training domains, TSLMs can generate coherent long sequences even
for unfamiliar domains.
If maliciously exploited, this capability could enable the systematic
fabrication of time series and pose severe threats in scenarios where data
authenticity is critical.

\begin{itemize}
\item \textbf{Finance:} TSLMs can synthesize long transaction histories that
  mirror real trading patterns, facilitating fraudulent activities such as
  inflated valuations or hidden manipulations, such as those seen in the 2012 LIBOR scandal~\cite{gupta2024defending, rose2013barclays}.

\item \textbf{Scientific Research:} Highly realistic counterfeit measurement
  time series (e.g., signal traces or biological data) can distort experimental
  outcomes, similar to Sch\"on and Wakefield data
  forgeries~\cite{Bellfabric,godlee2011wakefield}.

\item \textbf{Environmental Governance:} Attackers can generate counterfeit
  metrics (e.g., air quality indices or emissions) that reproduce genuine
  diurnal and seasonal cycles to conceal pollution spikes or overstate
  improvements, misleading policy makers and obscuring real
  hazards~\cite{Wang2021}.
\end{itemize}

\begin{figure*}
  \centering
  \includegraphics[width=.8\linewidth]{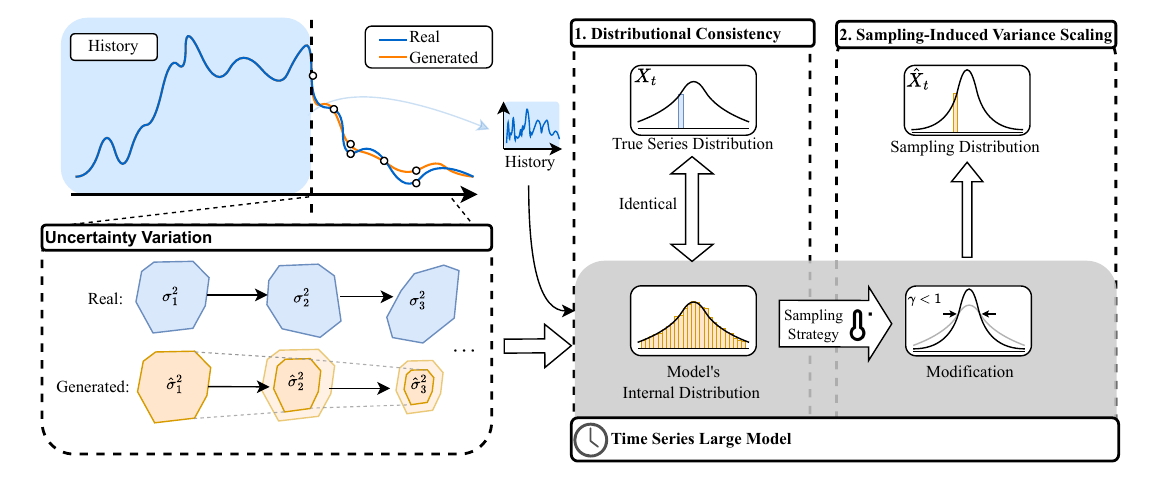}
  \caption{Illustration of the variation in uncertainty for real and model-generated
    time series.}
  \label{fig:overview}
\end{figure*}

To address these threats, we introduce a theoretical framework for white-box
detection of TSLM-generated time series.
Since there are no dedicated detection methods for time series data, we adapt
text-based detection methods to time series data.
Textual detectors exploit the observation that LLMs exhibit different
token-level probability patterns on human-written and model-generated text, and
therefore typically build zero-shot classifiers from each token’s probability or
its rank within the model’s vocabulary.
However, these methods face significant challenges when applied to time series
because of fundamental modality differences.

Textual data exhibit structural semantics and carry rich token-level
information~\cite{textInformation}, which creates greater semantic distances
among different tokens.
In any given context, only a small subset of tokens is semantically plausible
and with higher probability (e.g., probabilities might concentrate on
``apple'',``orange'' given the prefix ``I eat an'').
This produces sharper, low-entropy probability distributions over the vocabulary, with a narrower range of token selection.

In contrast, time series carry less information at each point~\cite{64words} and
inherently contain greater intrinsic uncertainty, yielding smoother probability
distributions.
Notably, despite the greater entropy of these smooth distributions, they reflect
the information at each time point rather than specific values.
Since the adjacent values in time series are highly similar (e.g.
temperature $25.1^\circ\mathrm{C}$ and $25.2^\circ\mathrm{C}$), resulting in
large mutual information, time series values convey less information.
Consequently, text-based detectors that rely on token probabilities
perform poorly on time series given the relatively lower probability gaps
between values.

Since point-wise probabilities are insufficiently discriminative in time
series, we instead analyze the full distributions from the model.
We show that a TSLM’s internal distributions, conditioned on the full history,
accurately capture the true series distributions and their inherent uncertainty.
As TSLMs are trained to minimize prediction errors, the internal distributions
are concentrated via model's sampling strategies for forecasting.
In recursive forecasting, the uncertainty is cumulatively decreasing, leading to
progressively concentrated internal distributions, as illustrated
in~\cref{fig:overview}.

We therefore propose the \textbf{contraction hypothesis}, namely, TSLM-generated
time series exhibit progressively decreasing uncertainty, whereas real time
series do not.
To validate this hypothesis, we provide a theoretical analysis under idealized
assumptions on model behavior and time series structure
(see~\cref{ss:analysis}) with empirical evidence through long-horizon
forecasting experiments (see~\cref{fig:confidence_oveall}).
Grounded in this detailed analysis of properties unique to time series, we introduce the Uncertainty Contraction Estimator
(UCE), a white-box model generation detection method for time series data.
UCE captures uncertainty dynamics from internal prediction distributions and
identifies sequences with lower uncertainty levels as model-generated.

\begin{figure*}[t]
  \centering
  \subfloat[Entropy\label{fig:entropy_trend}]{%
    \includegraphics[width=.3\linewidth]{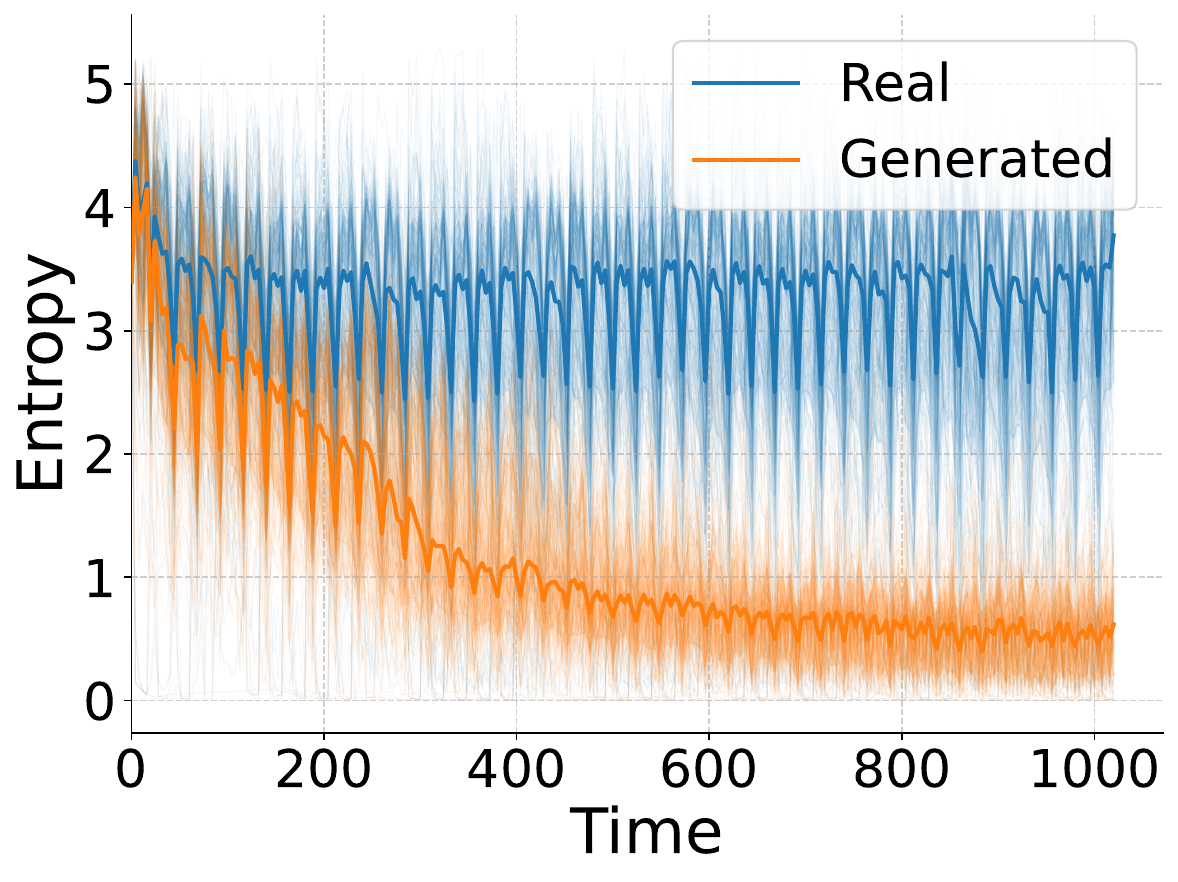}}\quad
  \subfloat[Max-Probability\label{fig:max_prob_trend}]{%
    \includegraphics[width=.3\linewidth]{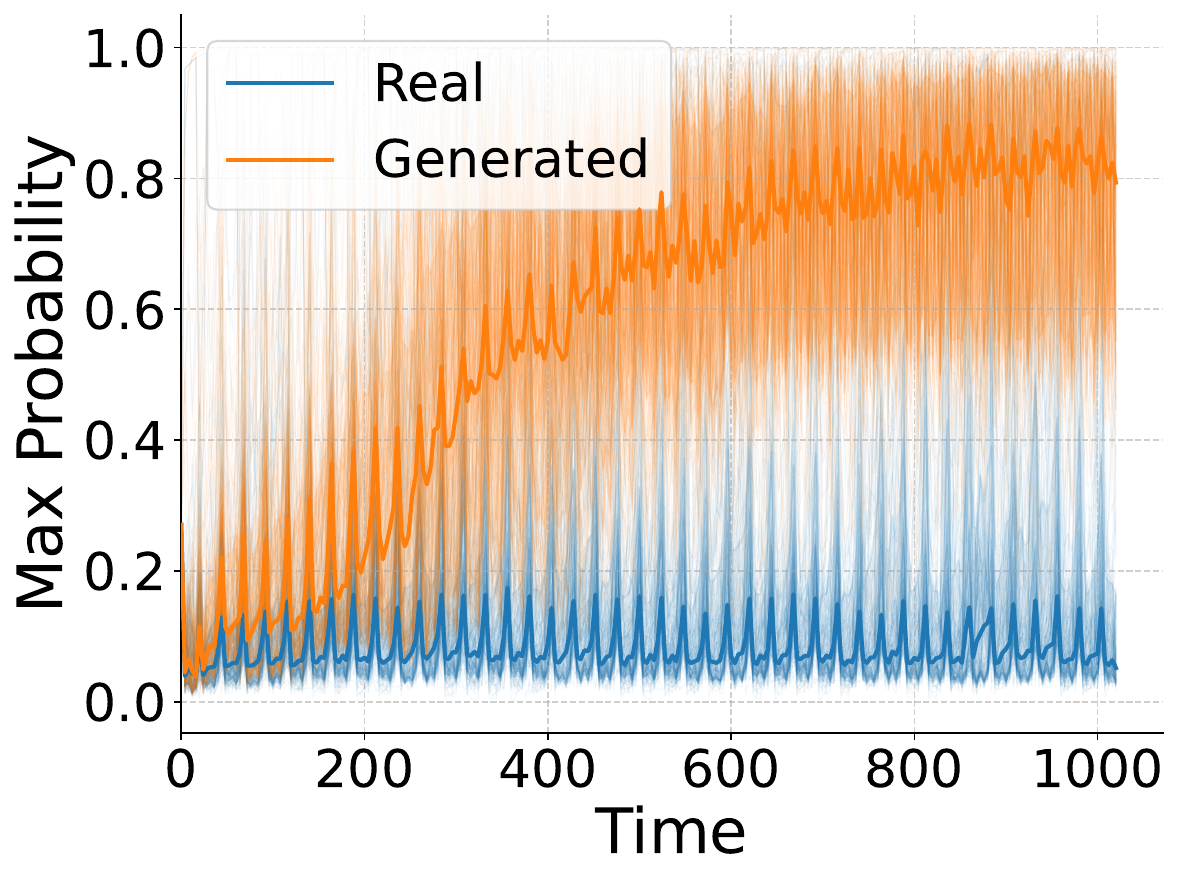}}\quad
  \subfloat[Variance\label{fig:var_trend}]{%
    \includegraphics[width=.3\linewidth]{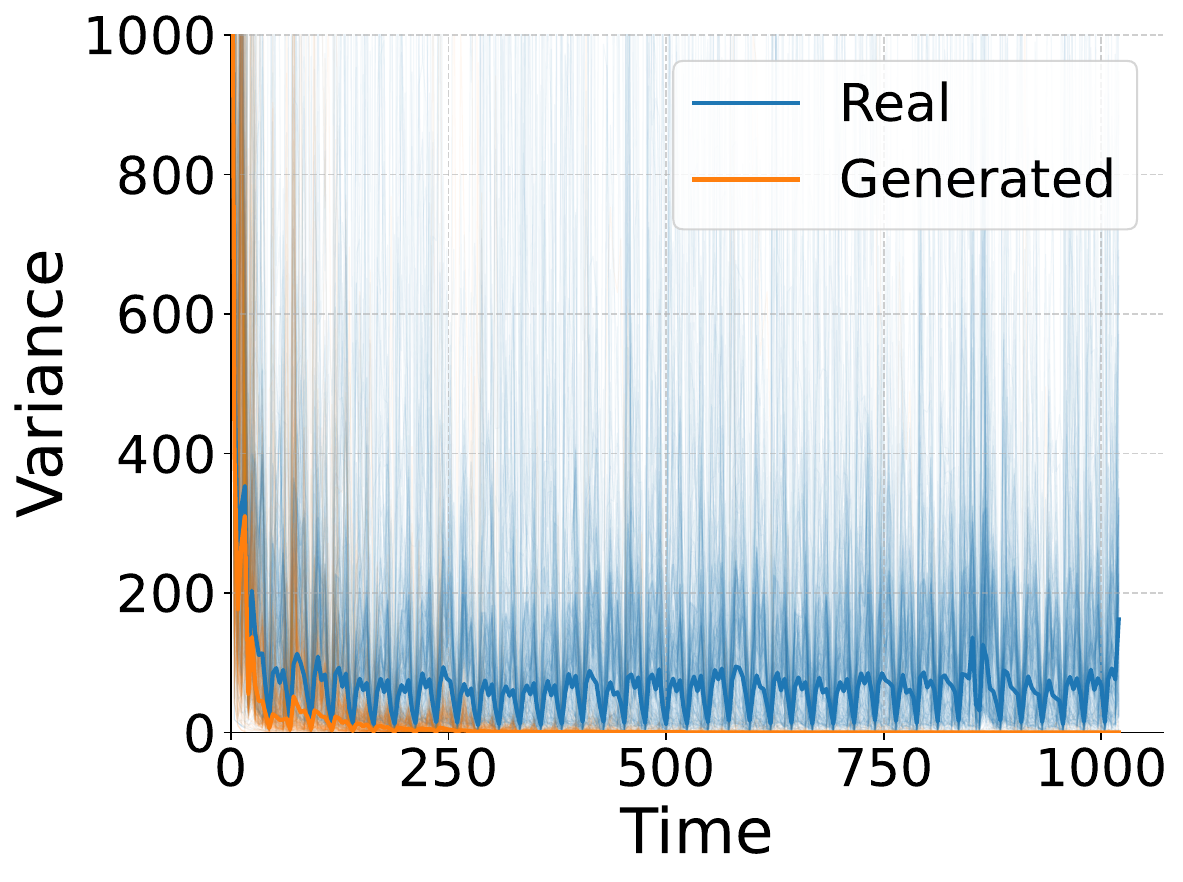}}
  \caption{The empirical results show the trajectories of uncertainty metrics,
    including entropy (\ref{fig:entropy_trend}), max-probability
    (\ref{fig:max_prob_trend}) and variance (\ref{fig:var_trend}) of both {real}
    and {model-generated} time series data, illustrating reduction in
    uncertainty for generated data.
  }
  \label{fig:confidence_oveall}
\end{figure*}

Our main contributions are summarized as follows.
\begin{itemize}
% \item To the best of our knowledge, we present the first framework for white-box
%   detection of TSLM-generated time series by extending textual detection methods
%   to time series and addressing the challenge of sparse information in time
%   series data.
\item To the best of our knowledge, we present the first framework for white-box
  detection of TSLM-generated time series. 
  We analyze the detailed properties unique to time series in contrast to textual data and address the challenge of low information density in time
  series.

\item We propose the contraction hypothesis, which states that model-generated
  time series exhibit progressively decreasing uncertainty during recursive
  forecasting, whereas real series do not.
  We provide a theoretical analysis under idealized assumptions of time series and model and empirically
  validate it.
  
\item Based on this hypothesis, we develop the Uncertainty
  Contraction Estimator (UCE), which captures uncertainty dynamics over
  successive prefixes using TSLMs to distinguish model-generated from real time series.
\end{itemize}

\section{Related Work}
\label{sec:related}

\subsection{Time Series Large Models}

Time series large models have emerged to significantly advance zero-shot time
series forecasting.
Earlier methods such as PromptCast~\cite{PromptCast} and
LLMTime~\cite{Time_Series_Forecasters} directly leverage LLMs for time series.
They convert time series data into text-based prompts and use pretrained LLMs
with little to no task-specific adaptation.
Such methods require dataset-specific templates and rely heavily on model scale.

PatchTST~\cite{64words} introduces splitting time series into patch embeddings,
a technique later adopted in models such as GPT4TS~\cite{OneFitsAll} and
Time-LLM~\cite{TimeLLM}.
These methods use patch embeddings as inputs to LLMs for prediction, but require
extensive fine-tuning.
Moreover, their performance and inference efficiency have been
questioned~\cite{AreLLMUseful}.

Recent work concentrates on models trained on vast and diverse time series
datasets.
MOMENT~\cite{moment} employs uniform random masking to generate patch embedding
for self-supervised pretraining.
Chronos~\cite{Chronos} discretizes real-valued time series through scaling and
quantization for forecasting and introduces data augmentation to improve
generalization.
Timer~\cite{timer} adopts a decoder-only transformer architecture, enabling
multiple temporal tasks such as forecasting, anomaly detection, and imputation.
Time-MoE~\cite{TimeMoE} leverages a Mixture-of-Experts framework to improve
forecasting performance and enhance cross-domain generalization while
maintaining computational efficiency.
These methods collectively represent the emerging trend of general-purpose time
series models capable of handling diverse tasks with minimal or no fine-tuning.

\subsection{Model Generation Detection}

The growing generative capabilities of LLMs have raised concerns about
distinguishing model-generated content, particularly in the textual domain.
Early work on model generation detection conduct supervised classification using bag of words~\cite{Release_Strategies} or neural
representations~\cite{automatic,authorship}, but often overfit and underperform
on out-of-distribution data~\cite{deepfaketext}.
To overcome this limitation, zero-shot detectors use LLM output statistics, such
as perplexity~\cite{relative_entropy} and log rank~\cite{gltr,unifying}.

Recent work analyzes token probabilities and uses the discrepancies between
human-written and model-generated text to develop classifiers.
DetectGPT~\cite{detectgpt}, Fast-DetectGPT~\cite{fast_detectgpt}, and
NPR~\cite{detectllm} compare probability differences between perturbed texts,
while DNA-GPT~\cite{dnagpt} leverages regeneration to compute divergence.
FourierGPT~\cite{FourierGPT} performs a spectral analysis on token probability
sequences to extract linguistic features that distinguish human-authored from
model-generated text.
Binocular~\cite{Binocular} leverages both an observer and a performer to compute
cross-perplexity and isolate intrinsic prefix-induced uncertainty to
improve detection accuracy.
Black-box methods such as intrinsic dimension~\cite{Intrinsic_dimension} perform
topological data analysis over token embeddings and find that human-written
texts are more complicated in topology.

Despite their success in text, these methods rely on the dense and
discrete structure of the language, making them unsuitable for continuous and
information-sparse time series data.
In this work, we extend zero-shot model generation detection to time series by
introducing uncertainty-based metrics derived from model's internal probability
distributions.

\section{Problem Formulation and Preliminaries}

In this section, we formally define the model-generated time series detection
problem.
Then we provide some preliminaries about our framework.

\subsection{Problem Formulation}

Let $\bX_t=(X_1,\ldots, X_t)$ denote a univariate time series with
\textit{unknown} history.
We perform a zero-shot classification of $\bX_t$ as real or
model-generated without labeled examples.

We assume a white-box setting with access to the model's internal probability
distribution $p_\theta(\cdot|\bX_t)$ rather than to the model architecture or
parameters $\theta$.
Our detection employs point-wise probabilistic TSLMs over an equidistant
vocabulary $\CV=\{v_i|v_{i+1}-v_i=\Delta,-R\leq v_i\leq R\}$ for new token
sampling.
Therefore, sampling token $v_i$ yields the numerical forecast $\hat X_t=v_i$.
In the limit $\Delta\to0$ and $R\to\infty$, $\CV$ converges to the real set
$\Real$ and a probability density function $f_\theta$ is obtained.
Such a formulation allows for detection evaluation across diverse model
architectures.

Some frequently used symbols throughout the paper are given
in~\cref{tab:notation}.
In the following sections, we use probability density functions to refer to the
distributions.

\begin{table}[ht]
  \centering\small
  \begin{tblr}{
      colspec={X[1,c,m]X[2.5,l,m]},
      colsep=2pt,
    }
    \toprule
    Symbol & Description  \\
    \midrule
    ${X}_t,\hat{X}_t$ & Real / forecast time series process \\
    $T_t$ & Deterministic trend component at time $t$ \\
    $n_t, \hat{n}_t$ & True / forecast (Gaussian) noise at time $t$ \\
    $\bX_{-H:t}, \hat{\bX}_{-H:t}$
    & History $X_{-H},\ldots,X_t$ (or $\hat{X}_{-H},\ldots,\hat{X}_t$) \\
    $f_t(X_t), \hat{f}_\theta(\hat{X}_t)$
    & Probability density of $X_t$ / $\hat{X}_t$ at time $t$ \\
    $\sigma_t^2, \hat{\sigma}_t^2$
    & Variance of true noise $n_t$ / forecast noise $\hat{n}_t$ \\
    $f_\theta(z_t|\bX_{-H:t-1})$
    & Model's internal probability density with history $\bX_{-H:t-1}$,
    abbreviated as $f_\theta(z_t)$ \\
    $\tilde{\sigma}_t^2$ & Variance of internal probability distribution at time $t$ \\
    $\gamma_t$ & Scaling factor of sampling strategy to $f_\theta(z_t)$ at time $t$ \\
    \bottomrule
  \end{tblr}
  \caption{Some frequently used symbols}
  \label{tab:notation}
\end{table}

\subsection{Preliminaries}

We first introduce some assumptions and premises about time series and idealized
TSLMs.
Detailed definitions and assumptions are provided in the Appendix.

We formally model time series as realizations of stochastic processes, where
each observation sequence is a sample path of this process.
We decompose the real time series process at time point $t$ into two
components~\cite{TimeSeriesAnalysis}:
\[
  X_t=T_t+n_t, \quad n_t\sim \CN\left(0,\sigma_t^2\right),
\]
% where $T_t$ denotes the trend sequence, representing predictable components, and
% $n_t$ is a Gaussian process with zero mean and variance satisfying
% $\sigma_t^2=\sum_{i=1}^{l}\alpha_{i}\sigma_{t-i}^2$ where
% $\sum_{i=1}^{l}\alpha_{i}=1$.
% The noise process represents the unpredictable or uncertainty components.
where $T_t$ denotes the trend sequence, representing predictable components, and
$n_t$ is a Gaussian process with zero mean and variance such that $\sigma_t^2=\sum_{i=1}^{l}\alpha_{i}\sigma_{t-i}^2$, $\sum_{i=1}^{l}\alpha_{i}=1$.
The noise process represents the unpredictable or uncertainty components.

We introduce \emph{Ideal Model}, which, for any real-valued time series in
$\Real$, predicts probability distributions to minimize the expectation of
cross-entropy loss.
Conceptually, it generalizes practical TSLMs by letting both training data and
model capacity grow without bound.
This abstraction allows us to derive fundamental detection principles regardless
of specific architectures or training constraints.

We formalize \emph{model evaluation function} as \cref{eq:Eva}:
\begin{equation}\label{eq:Eva}
    \mathrm{Eva}=
  \left[\sum_{i=1}^{\tau}g\left(\left|X_i-\hat{X}_i\right|^{p}\right)\right]^\frac{1}{p},
  1\leq p\leq \infty,
\end{equation}
where $g$ is any non-negative strictly increasing mapping.
Specifically, if $g$ is linear, $\mathrm{Eva}$ reduces to the standard $\ell_p$
norm of point-wise errors (e.g., RMSE when $p=2$).
Since $X_{i}$ and $\hat{X}_i$ are variables, $\mathrm{Eva}$ is also a variable.
$\mathrm{Eva}$ is a unified form of various model prediction evaluators (e.g.,
MSE, MAE) which in turn guides the model generation process.

% Detailed definitions and assumptions for the aforementioned concepts are
% provided in \cref{apdx:definition}.
% Detailed definitions and assumptions are provided in the Supplementary Material (Section A).

\section{Methodology}

In this section, we introduce the Uncertainty Contraction Estimator (UCE), a
novel time series model generation detection method.
We first establish the theoretical foundation by proposing the \emph{contraction
  hypothesis} in~\cref{subsec:Contraction Hypothesis}, followed by a formal
analysis in~\cref{ss:analysis}.
Finally, the implementation details of UCE are presented in~\cref{subsec:UCE}.

\subsection{Contraction Hypothesis}
\label{subsec:Contraction Hypothesis}

UCE hinges on two observations: (1) the model's internal distributions
faithfully reproduce true series distributions at each time; (2) through sampling
strategies, the model systematically modifies these distributions, creating a
self-reinforcing process of exponentially decreasing uncertainty.
We formalize these insights as the contraction hypothesis.

\textit{Contraction hypothesis: TSLM-generated time series exhibit progressively
  concentrating internal distributions with decreasing uncertainty, whereas real
  series do not.}

The term ``contraction'' refers to the contraction mapping nature of the
uncertainty reduction.
We systematically substantiate the hypothesis through a
tripartite theoretical framework.
% Formal proofs of all following propositions are provided in~\cref{apdx:proof}.
% Formal proofs of all the following propositions are provided in the Supplementary Material (Section B).
Formal proofs of all the following propositions are provided in the Appendix.

\subsection{Theoretical Analysis}
\label{ss:analysis}

\subsubsection{Distributional Consistency.}

We begin with the property of ideal models.
Given any history $\mathbf{X}_{-H:t-1}$, TSLM generates an
\emph{internal probability distribution} $f_{\theta}$.
From the perspective of this distribution, we revisit the ideal models and
conclude the distribution-perfect prediction property.
\begin{lemma}\label{theorem:continuous}
  For $\sigma_t^2\geq0$, we have $f_{\theta}\equiv
  f_{t}\; a.e.$
\end{lemma}
\cref{fig:perfect prediction} illustrates the insight
of~\cref{theorem:continuous}.
Since the probability density of $X_{t} = T_t+d$ is $f_t(T_t+d)$, the loss
expectation is therefore the cross entropy $-\int f_{t}(v)\log
f_{\theta}(v)\mathrm{d}v$ between $f_{\theta}$ and the true series
distribution $f_{t}$.
As illustrated in \cref{fig:perfect prediction}, it is minimized when
$f_{\theta}\equiv f_{t}$ almost everywhere by Gibbs’ inequality.
Specifically, if the true sequence is entirely deterministic (i.e., zero
uncertainty), ideal models make value-perfect predictions, leading
to~\cref{theorem:dirac}.
\begin{restatable}{corollary}{MainBodyDirac}\label{theorem:dirac}
    If $\sigma_t^2=0$, we have $f_{\theta}(z=v)=\delta(v-T_t)$, where $\delta$ is the Dirac-$\delta$ function.
\end{restatable}

% \begin{corollary}\label{theorem:dirac}
%     If $\sigma_t^2=0$, we have $f_{\theta}(z=v)=\delta(v-T_t)$, where $\delta$ is the Dirac-$\delta$ function.
% \end{corollary}

% \begin{lemma}
% \end{lemma}

In summary, the ideal model's internal distributions coincide with true series distributions and therefore preserve the inherent uncertainty in the series.

\begin{figure}
    \centering
    \includegraphics[width=.95\linewidth]{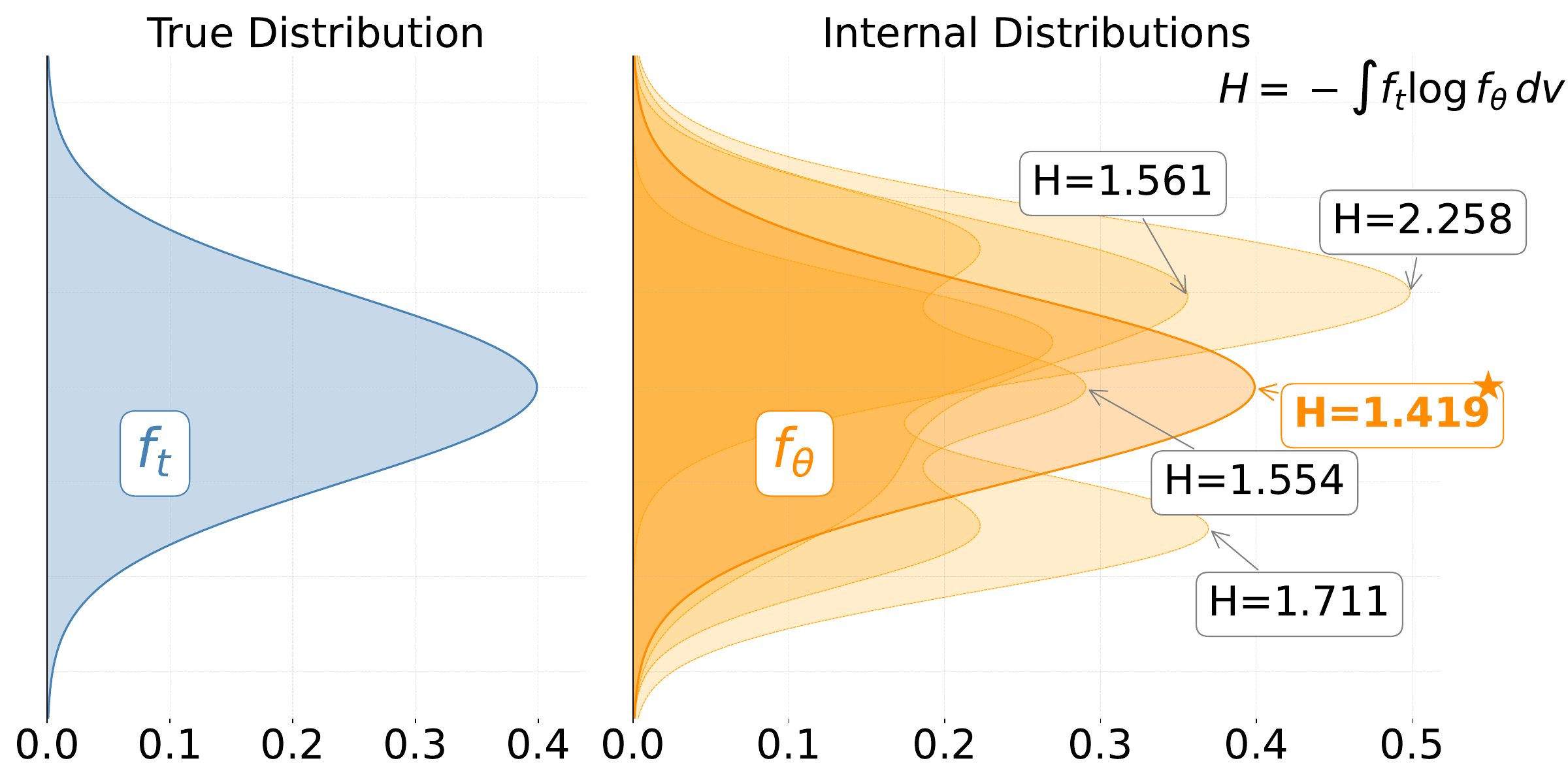}
    \caption{Comparison of the true distribution (blue) and model's internal distributions (orange). For an ideal model, its internal distribution coincides with the true distribution.}
    \label{fig:perfect prediction}
\end{figure}

\subsubsection{Sampling-Induced Variance Scaling.}
Using the property of ideal models to reproduce true series distributions, we analyze the prediction of $\mathbf{\hat{X}}_{1:\tau}$ with history $\mathbf{X}_{-H:0}$, and focus on the variation in uncertainty.
In this section, we explain that the sampling probability distribution $\hat{f}_{t}$ of each $X_{t}$ is a modified version of $f_{\theta}$, which tends to reduce its uncertainty.

TSLMs modify internal distributions $f_{\theta}$ through sampling strategies to modulate output diversity~\cite{GPT2}.
We summarize common sampling strategies as follows: (1) scaling (e.g., temperature sampling); and (2) symmetric truncation (with normalization, e.g., top-$k$).
These modifications preserve the means of $f_{\theta}$ but alter their variances, as shown in~\cref{fig:modification}, generating sampling probability distributions $\hat{f}_\theta(\hat{X}_t=v)$ to sample $\hat{X}_{t}$ with variance $\hat{\sigma}^2_{t}$.

Let the uncertainty (variance) of $f_{\theta}$ be $\tilde{\sigma}^2_{t}$ at
time $t$, we denote the modified uncertainty by sampling strategies as
$\hat{\sigma}^2_{t}= \gamma_{t}\cdot\tilde{\sigma}^2_{t}$.
Intuitively, $\gamma_t$ controls uncertainty expansion ($\gamma_t>1$),
preservation ($\gamma_t=1$), or contraction ($\gamma_t<1$).
Since modifications of sampling strategies preserve the expectation, we
recursively obtain $\E{\hat{X}_t}=T_t$ for all $1\leq t\leq\tau$.
We therefore treat the generation of $\hat{\mathbf{X}}_{1:\tau}$ as $\tau$
independent, one-step predictions given their respective histories.
We evaluate the differences between $\mathbf{X}_{1:\tau}$ and
$\hat{\mathbf{X}}_{1:\tau}$ using the model evaluation function, leading to
\cref{theorem:evaluation}.
\begin{restatable}{lemma}{MainBodyEva}\label{theorem:evaluation}
  The expectation of the model evaluation function $\E{\mathrm{Eva}}$ increases
  monotonically with respect to all $\hat{\sigma}_{t}$.
\end{restatable}
% \begin{lemma}\label{theorem:evaluation}
%   The expectation of the model evaluation function $\E{\mathrm{Eva}}$ increases monotonically in all $\hat{\sigma}_{t}$.
% \end{lemma}
\cref{theorem:evaluation} indicates that a smaller $\hat{\sigma}_t^2$ yields a lower $\mathrm{Eva}$ value, suggesting a higher quality of generation.
Given $\hat{\sigma}_{t}^2=\gamma_{t}\cdot\tilde{\sigma}_t^2$, when $\tilde{\sigma}_t^2$ is fixed, a smaller $\gamma_t$ yields lower $\E{\mathrm{Eva}}$, which validates our intuition.

\begin{figure}
    \centering
    \includegraphics[width=0.85\linewidth]{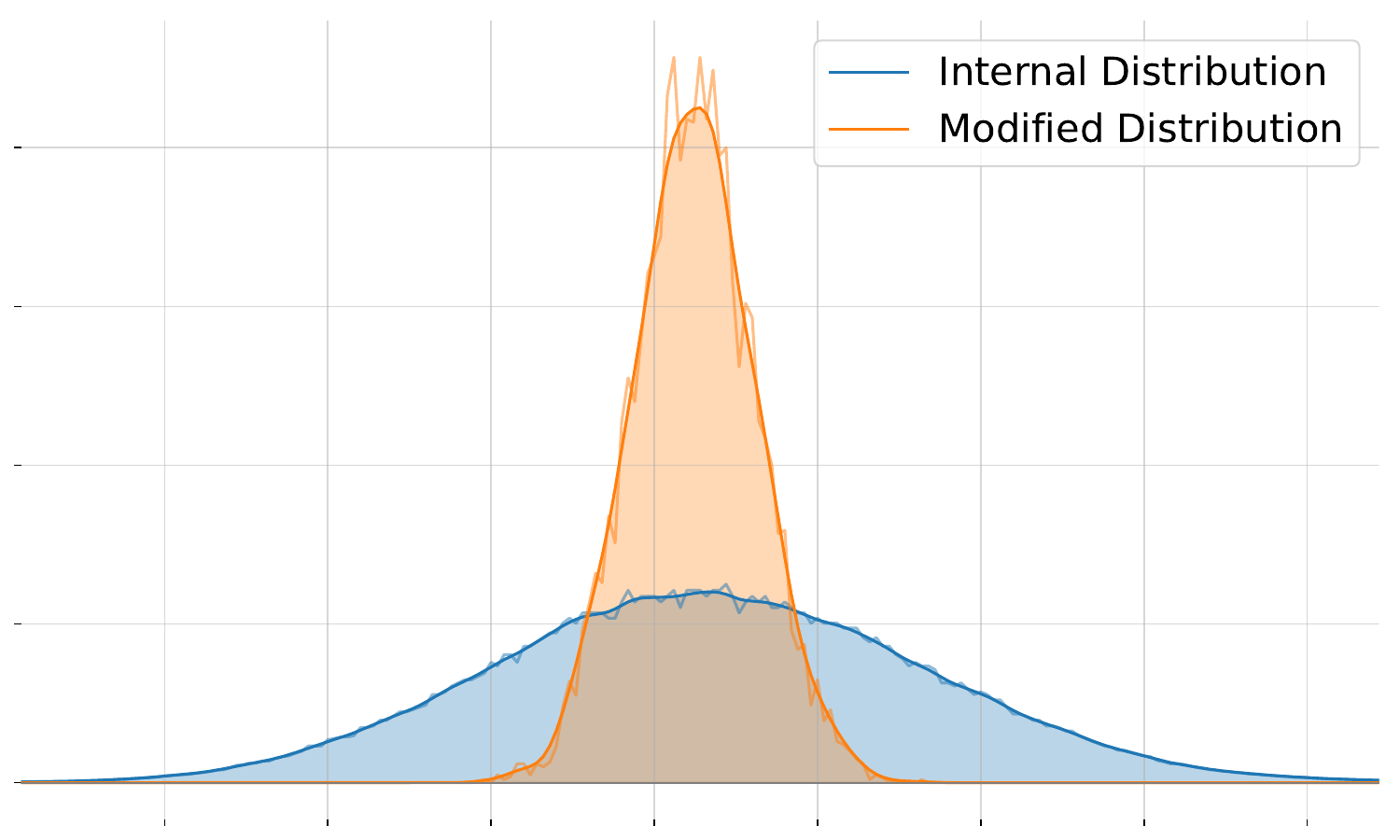}
    \caption{Comparison of the internal distribution (blue, from the model logits) and modified sampling distribution (orange, from 10,000-step Monte Carlo) for a single prediction step.
    The model’s internal probability distribution becomes sharper under our sampling strategy.
    }
    \label{fig:modification}
\end{figure}

In particular, we derive a corollary from \cref{theorem:evaluation} that if metrics such as MSE are used as training loss, the contraction of uncertainty is immediate.
\begin{corollary}
    \label{corollary:MSE_converge}
    If the loss function is defined in the same form as the model evaluation function $\mathrm{Eva}$ in \cref{eq:Eva} (e.g., MSE),
    % \begin{equation}
    %     \mathcal{L}=
    % \left[\sum_{i=1}^{\tau}g\left(\left|X_i-z_i\right|^{p}\right)\right]^\frac{1}{p},
    % \end{equation}
    we have $f_{\theta}(z=v)=\delta(v-T_t)$.
\end{corollary}

% \cref{theorem:evaluation} shows that given fixed ${X}_t$, the model evaluation function in reality evaluates the distance between the generated observation $\hat{X}_t$ and the trend sequence $T_t$.
% In a fixed underlying time series $\mathbf{X}_{1:\tau}$, the difference in ``quality'' between different forecasts stems solely from their expected distance to $T_t$, and is irrelevant to the true noise $n_t$.
% Given this characterization of quality, TSLM minimizes the uncertainty in  forecast noise $\hat{n}_t$ regardless of true noise $n_t$.

Based on the analysis, we intuitively expect that TSLMs employ $\gamma_t<1$ in practice.
In the following, we further analyze this intuition of $\gamma_t$ in terms of performance.

\subsubsection{Recursive Variance Reduction.}
\label{subsubsec:recursive variance reduction}
With the intuition of $\gamma_t$ above, we examine how the sampling distribution (i.e., the internal distribution modified by $\gamma_t$) affects the subsequent generation.
In this section, we demonstrate the recursive evolution of the variance and analyze its behavior under different values of the scaling factor $\gamma_t$.

According to~\cref{theorem:continuous}, the ideal model reproduces any distribution with variance satisfying $\sigma_{t}^2=\sum_{i=1}^{l}\alpha_{i}\sigma_{t-i}^2$.
When forecasts $\hat{X}_1, \hat{X}_2,\ldots$ serve as histories, their variances $\hat{\sigma}_{1}^2,\hat{\sigma}_2^2,\ldots$ also follow the same recurrence relation.
\begin{equation}\label{eq:recursion}
    \tilde{\sigma}_t^2=\sum_{i=1}^{l}\alpha_{i}\hat{\sigma}_{t-i}^2=\sum_{i=1}^{l}\alpha_{i}\gamma_{t-i}\tilde{\sigma}_{t-i}^2,\;\sum_{i=1}^{l}\alpha_i=1.
\end{equation}

Typically, the uncertainty scaling behavior is consistent and $\gamma_t< 1$ or $\gamma_t> 1$ for all $t$.
When $\gamma_t < 1$, the internal uncertainty, according to~\cref{eq:recursion}, is strictly lower than the weight sum of historical uncertainties, leading to an exponential decay in uncertainty towards $0$.
By~\cref{theorem:dirac}, the forecast series $\hat{X}_t$ converges to the trend $T_t$.
Conversely, when $\gamma_t > 1$, the uncertainty grows exponentially towards infinity, introducing excessive noise into the forecast that undermines trend capture and violates forecasting objectives.
When $\gamma_t\equiv 1$, the uncertainty in forecast series is relatively stable for ideal models, which is unattainable for practical models due to noise accumulation.

This theoretical analysis also helps explain why many practical TSLMs exhibit similar behaviors of uncertainty contraction.
For example, Chronos~\cite{Chronos} adopts contractive sampling strategies by using top-$k$ with median sampling, and therefore $\gamma_t<1$ holds consistently.
Timer~\cite{timer} and Time-MoE~\cite{TimeMoE} utilize point-wise errors (e.g., MSE) as training loss, sharing the form as $\mathrm{Eva}$ defined in \cref{eq:Eva}.
According to \cref{corollary:MSE_converge}, the forecast uncertainty converges rapidly, indicating their equivalence to $\gamma_t<1$ in terms of performance.
We formalize the aforementioned analysis as \emph{contraction hypothesis}, where variance can be substituted for other metrics.
\begin{theorem}\label{theorem:contraction hypothesis}
  For an Ideal Model's forecast time series $\mathbf{\hat{X}}_t$ with history $\mathbf{X}_{-H:0}$, let the internal probability distribution at $t$ be $f_{\theta}(z_t)$, and define its uncertainty measure by variance $\tilde{\sigma}_t^2$.
  Therefore we have that $\tilde{\sigma}_t^2$ is monotonically decreasing with $t$ and $\lim_{t\to\infty}\tilde{\sigma}_t^2\approx 0$.
\end{theorem}

We provide empirical verification of~\cref{theorem:contraction hypothesis} in \Cref{fig:confidence_oveall}, which illustrates this uncertainty contrast over 1024 tokens using three metrics, namely entropy, max-probability, and variance, which are used in our detection method.
In model-generated series, max-probability steadily approaches $1$, while both entropy and variance decay towards $0$.
In contrast, real sequences remain comparatively stable in uncertainty.
This observed discrepancy in series uncertainty, combined with the ability of TSLMs to capture it, forms the basis of our detection methodology.

\begin{figure}[t!]
    \centering
    \includegraphics[width=1\linewidth]{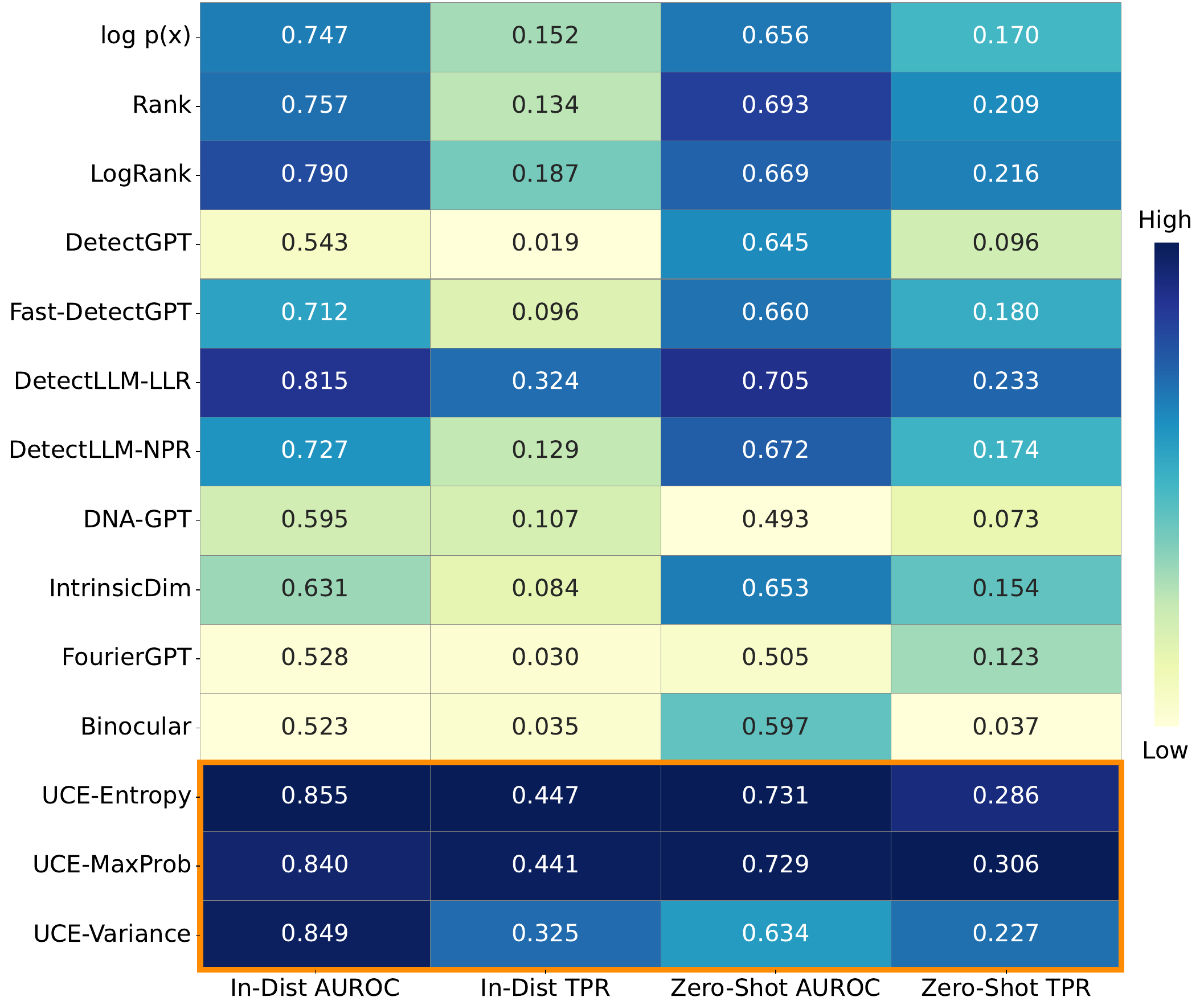}
    \caption{Average AUROC and TPR (at 1\% FPR) for model generation detection on In-Distribution (12 datasets) and Zero-Shot (20 datasets) scenarios.}
    \label{fig:exp_result}
\end{figure}

\subsection{Uncertainty Contraction Estimator}
\label{subsec:UCE}

The Uncertainty Contraction Estimator (UCE) operationalizes the Contraction
Hypothesis by quantifying the
uncertainty contraction in model-generated time series.
Given a candidate $\mathbf{X}_t$, we first sample $N$ time points
$t_1, t_2, \ldots, t_N$ with a fixed interval $\Delta t=t_{i+1}-t_{i}$.
% (the sensitivity analysis of the starting point $t_1$ is discussed in~\cref{ext_exp:ratio}).
For each $t_i$, UCE takes $\mathbf{X}{t_i}$ as the input to the TSLM and computes its internal probability distribution $\hat{P}_{t_i}$ as defined in~\cref{eq:distribution}, yielding a sequence of internal distributions.
\begin{equation}\label{eq:distribution}
    \hat{P}_{t_i}= p_{\theta}\left(\cdot|X_{1},\ldots,X_{t_{i}}\right),i=1,\ldots,N.
\end{equation}
By~\cref{theorem:continuous} the internal probability distribution coincides
with the underlying Gaussian noise distribution, which is unimodal and
relatively probability-concentrated.
Therefore, UCE focuses on a neighborhood $\mathcal{U}$ around the mean.
Within $\mathcal{U}$, UCE computes three uncertainty measures to capture different aspects of the distribution
concentration.
\begin{enumerate}
\item \emph{Entropy}
  $E=-\sum_{x\in\mathcal{U}}\hat{P}\left(x\right)\log \hat{P}\left(x\right)$,
  which captures the spread of probability mass.
\item \emph{Max-Probability} $P_{\max}=\max_{x\in \mathcal{U}}\hat{P}(x)$,
  where larger values imply lower uncertainty.
\item \emph{Variance}
  $\mathrm{Var}=\sum_{x\in\mathcal{U}}\left(x-\mu\right)^2\hat{P}(x)$ with $\mu$
  the local mean to quantify concentration.
\end{enumerate}
% The detailed selection of hyperparameters is given in~\cref{apdx:hyperparameter}.
For a selected metric $s\in\{E, P_{\max},\mathrm{Var}\}$, UCE calculates the
metric sequence $s_{t_{1}}, s_{t_{2}},\ldots, s_{t_{N}}$, and the overall UCE score is their mean, formulated as
$\mathrm{UCE}=\frac{1}{N}\sum_{i=1}^{N}s_{t_i}$.

\section{Experiments}

\subsection{Experimental Settings}

\subsubsection{Model and Datasets.}
We evaluate all detection methods using Chronos-T5 (large), a
point forecasting TSLM that discretizes continuous values over a fixed token
vocabulary and produces probability distributions.
We conduct experiments on 32 datasets across diverse domains (energy, finance, transportation), categorized as (1) in-distribution datasets, including 12 used during Chronos training; and
(2) zero-shot datasets, consisting of 20 previously unseen datasets.
The zero-shot setting evaluates the generalization of detection to unseen datasets to eliminate the effects of data leakage, and validates the ``black-box'' detection capability (see \cref{subsec:cross_model}).
For each dataset, we follow the default setup in Chronos to generate forecasts of horizon $H=64$, and use the corresponding $H$ observations as ground truth.
We treat forecast series as positive (model-generated) samples and real series as negative samples.

\subsubsection{Baselines.}
We compare UCE with 10 text-based baselines adapted for time series:
(1) \textbf{DNA‑GPT WScore}~\cite{dnagpt}, which measures the average log‑likelihood gap between regenerated and original text given a fixed prefix;
(2) \textbf{DetectGPT}~\cite{detectgpt} and its efficient variant \textbf{Fast‑DetectGPT}~\cite{fast_detectgpt}, which detect generation by measuring probability changes under model‑guided perturbations;
(3) \textbf{DetectLLM}~\cite{detectllm}, using Log‑Likelihood Log‑Rank Ratio (LRR) and Normalized Perturbed log‑Rank (NPR) metrics;
(4) \textbf{Intrinsic Dimension}~\cite{Intrinsic_dimension}, which classifies sequences by estimating the topological “intrinsic dimension” of token embeddings;
(5) \textbf{FourierGPT}~\cite{FourierGPT}, performing spectral analysis on token probability sequence;
(6) \textbf{Binocular}~\cite{Binocular}, leveraging two models to reduce the intrinsic perplexity of prefixes;
and (7) \textbf{Traditional metrics}~\cite{gltr,automatic,Release_Strategies}, including log‑likelihood $\log p(x)$, rank, and log‑rank.
The details of the baselines are documented in the Appendix.

\subsubsection{Evaluation Metrics.}
We use the Area Under the ROC Curve (AUROC) to evaluate the detection performance of UCE and the baseline methods.
Since AUROC masks performance at low false positive rates (FPR), which are critical in generation detection tasks~\cite{1stprinc,paraphrasing,dnagpt}, we also report the true positive rate (TPR) at a fixed FPR of 1\%~\cite{curiouscaseAUROC}.

\subsection{Overall Result}\label{sec:overall result}

We evaluate detection performance under in-distribution and zero-shot settings.
UCE consistently achieves state-of-the-art AUROC and TPR across all scenarios, with key results shown in \cref{fig:exp_result}.

\subsubsection{In-Distribution.} 
% UCE-Entropy achieves an AUROC of $0.855$, outperforming the strongest baseline DetectLLM-LLR ($0.815$) by $0.040$ and the baseline average ($0.670$) by $0.183$, and a TPR of $0.447$, surpassing DetectLLM-LLR ($0.324$)  by $0.123$ and the baseline average ($0.118$) by $0.329$.
UCE-Entropy achieves an AUROC of 0.855, outperforming the strongest baseline, DetectLLM-LLR (0.815), by 0.040, and exceeding the baseline average (0.670) by 0.183. Its TPR reaches 0.447, surpassing DetectLLM-LLR (0.324) by 0.123 and the baseline average (0.118) by 0.329.
% Although UCE-Variance shows slightly reduced TPR, all variants maintain substantial advantages.

\subsubsection{Zero-Shot.} 
% UCE-Entropy achieves an AUROC of $0.731$, exceeding DetectLLM-LLR ($0.705$) by $0.026$ and the baseline average ($0.632$) by $0.099$, and a TPR of $0.286$, outperforming DetectLLM-LLR ($0.233$) by $0.053$ and the average baseline ($0.151$) by $0.135$.
In the zero-shot setting, UCE-Entropy achieves an AUROC of 0.731, outperforming DetectLLM-LLR (0.705) by 0.026 and the baseline average (0.632) by 0.099. It also attains a TPR of 0.286, exceeding DetectLLM-LLR (0.233) by 0.053 and the average baseline (0.151) by 0.135.

The other two UCE metrics (Max-Probability and Variance) also show strong performance, with UCE–MaxProb approaching UCE–Entropy in TPR despite a modest AUROC gap; UCE–Variance is slightly weaker in TPR.
Overall, UCE-Entropy remains the top performer, UCE-MaxProb is a close second, and UCE-Variance may offer greater robustness under alternative conditions.
% The full experimental results are provided in \cref{tab:In-domain-64-full-1} and \cref{tab:0-shot-64-full-1}, with additional experiments in \cref{apdx:ext_exp}.
% The full experimental results, with additional experiments, are provided in the Supplementary Material (Section D).
The full experimental results, with additional experiments, are provided in the Appendix.
The experimental results suggest potential future work on hybrid or adaptive variant selection.

\subsection{Cross-Model Detection Performance of UCE}
\label{subsec:cross_model}
In this section, we evaluate the generalization of UCE on time series generated by two alternative TSLMs, Timer and Time-MoE.
We conduct experiments on 9 datasets covering multiple lengths and report both AUROC and TPR.

UCE achieves strong overall performance on both models, with the Entropy variant exhibiting particularly notable results, consistently achieving high AUROC and TPR across both models, as shown in \cref{tab:Timer} and \cref{tab:TimeMoE}.
% , and outperforms both the Max-Probability and Variance variants by a substantial margin, as shown in \cref{tab:Timer} and \cref{tab:TimeMoE}.
The Max-Probability and Variance variants yield weaker performance, especially for Timer-generated sequences. 
Notably, UCE performs particularly well in Time-MoE, possibly due to its Mixture-of-Experts architecture with better long-forecasting performance.
In \cref{subsubsec:recursive variance reduction} we show that although Timer and Time-MoE are not probabilistic forecasting models, they also exhibit uncertainty contraction, and uncertainty metrics demonstrate discriminative power between real and model-generated series.
% \cref{fig:hist} and \cref{fig:hist_TimeMoE} illustrate the significant uncertainty differences between real and model-generated series in all three metrics, indicating the discriminative power of uncertainty.

% We note that although Timer and Time-MoE are not probabilistic forecasting models, their predictions still exhibit significant differences in uncertainty compared to real series, as illustrated in \cref{fig:hist} and \cref{fig:hist_TimeMoE}.
% This is because they are trained with MSE or similar point-wise losses, which dominate time series modeling.
% We formally show in \cref{corollary:MSE_converge} that, under such losses, the predictive uncertainty tends to converge rapidly.
% This infers the ubiquity of the uncertainty contraction effect (or equivalently $\gamma_t<1$) and demonstrates the discriminative power of uncertainty.

Furthermore, the performance of UCE on both zero-shot datasets and cross-model settings also implies its potential for the ``black-box'' detection analogous to DetectGPT.
Specifically, given a time series that originates from an unknown model and unknown source (or the real world), UCE can perform detection by leveraging a locally deployed probabilistic model to compute uncertainty-based signals.

% \cref{tab:Timer} reveals that when applied to Timer-generated time series, UCE-Entropy exhibits relatively strong overall detection performance in both AUROC and TPR.
% However, the Max-Probability and Variance variants consistently underperform.
% Interestingly, all variants degrade slightly as the horizon increases.
% Results for Time-MoE, reported in \cref{tab:TimeMoE}, show that all three variants of UCE demonstrate strong detection performances in both metrics, and UCE-Entropy outperforms the other two variants.

% These cross-model results, along with the uncertainty histograms in \cref{fig:hist} and \cref{fig:hist_TimeMoE}, confirm that UCE, especially the entropy-based variant, generalizes robustly as a zero-shot detector across various TSLM architectures.

\begin{table}[t]
\small
\centering
% \resizebox{\columnwidth}{!}{
% \begin{tabular}{@{}lcccccccc@{}}
% \toprule
%  \multicolumn{1}{l}{\multirow{2}{*}{\makecell{\textbf{Time Series}\\
%  \textbf{Length $H$}}}}
% &  \multicolumn{2}{c}{$H=96$ Avg.}
% &  \multicolumn{2}{c}{$H=192$ Avg.}
% &  \multicolumn{2}{c}{$H=384$ Avg.}
% &  \multicolumn{2}{c}{$H=768$ Avg.}
%  \\
%   \cmidrule(lr){2-3} \cmidrule(lr){4-5}\cmidrule(lr){6-7}\cmidrule(lr){8-9}

%    & \makecell{AUROC} & \makecell{TPR}
%    & \makecell{AUROC} & \makecell{TPR}
%    & \makecell{AUROC} & \makecell{TPR}
%    & \makecell{AUROC} & \makecell{TPR}\\

% \midrule

% UCE -Entropy
% &0.833	&0.301
% &0.771  &0.280
% &0.765  &0.305
% &0.788  &0.366\\

% \quad -Max Prob
% &0.556  &0.041
% &0.508	&0.050
% &0.484	&0.032
% &0.542  &0.096\\
% \quad -Variance
% &0.635	&0.108
% &0.576	&0.102
% &0.564	&0.093
% &0.602  &0.179\\

% \bottomrule
% \end{tabular}
\setlength{\tabcolsep}{1mm}
\begin{tabular}{@{}lcccccc@{}}
\toprule
 \multicolumn{1}{l}{\multirow{2}{*}{\makecell{\textbf{Time Series}\\
 \textbf{Length $H$}}}}
&  \multicolumn{2}{c}{UCE-Entropy}
&  \multicolumn{2}{c}{UCE-Max Prob}
&  \multicolumn{2}{c}{UCE-Variance}
 \\
 \cmidrule(lr){2-3} \cmidrule(lr){4-5}\cmidrule(lr){6-7}
    & \makecell{AUROC} & \makecell{TPR}
   & \makecell{AUROC} & \makecell{TPR}
   & \makecell{AUROC} & \makecell{TPR}\\
\midrule
$H=96$
& 0.833 & 0.301
& 0.556 & 0.041
& 0.635 & 0.108 \\

$H=192$
& 0.771 & 0.280
& 0.508 & 0.050
& 0.576 & 0.102 \\

$H=336$
& 0.765 & 0.305
& 0.484 & 0.032
& 0.564 & 0.093 \\

$H=768$
& 0.788 & 0.366
& 0.542 & 0.096
& 0.602 & 0.179 \\
\bottomrule
\end{tabular}
% }
\caption{AUROC and TPR (at 1\% FPR) for Timer generation detection on 9 datasets. }
\label{tab:Timer}
\end{table}

% In-domain datasets and 0-shot datasets are split in 2 tables
\begin{table}[t]
\small
\centering
% \resizebox{\columnwidth}{!}{
% \begin{tabular}{@{}lcccccccc@{}}
% \toprule
%  \multicolumn{1}{l}{\multirow{2}{*}{\makecell{\textbf{Time Series}\\
%  \textbf{Length $H$}}}}
% &  \multicolumn{2}{c}{$H=96$ Avg.}
% &  \multicolumn{2}{c}{$H=192$ Avg.}
% &  \multicolumn{2}{c}{$H=336$ Avg.}
% &  \multicolumn{2}{c}{$H=720$ Avg.}
%  \\
%   \cmidrule(lr){2-3} \cmidrule(lr){4-5}\cmidrule(lr){6-7}\cmidrule(lr){8-9}

%    & \makecell{AUROC} & \makecell{TPR}
%    & \makecell{AUROC} & \makecell{TPR}
%    & \makecell{AUROC} & \makecell{TPR}
%    & \makecell{AUROC} & \makecell{TPR}\\

% %  \midrule
% %  % \hdashline
% % Prediction Length $H$ &64\\
% % MASE &0.663\\
% \midrule

% UCE -Entropy
% &0.829	&0.320
% &0.890  &0.392
% &0.957  &0.611
% &0.950  &0.561\\

% \quad -Max Prob
% &0.717  &0.292
% &0.773	&0.420
% &0.810	&0.475
% &0.845  &0.540\\

% \quad -Variance
% &0.745	&0.316
% &0.806	&0.422
% &0.856	&0.511
% &0.863  &0.566\\

% \bottomrule
% \end{tabular}
\setlength{\tabcolsep}{1mm}
\begin{tabular}{@{}lcccccc@{}}
\toprule
 \multicolumn{1}{l}{\multirow{2}{*}{\makecell{\textbf{Time Series}\\
 \textbf{Length $H$}}}}
&  \multicolumn{2}{c}{UCE-Entropy}
&  \multicolumn{2}{c}{UCE-Max Prob}
&  \multicolumn{2}{c}{UCE-Variance}
 \\
 \cmidrule(lr){2-3} \cmidrule(lr){4-5}\cmidrule(lr){6-7}
    & \makecell{AUROC} & \makecell{TPR}
   & \makecell{AUROC} & \makecell{TPR}
   & \makecell{AUROC} & \makecell{TPR}\\
\midrule
$H=96$
& 0.829 & 0.320
& 0.717 & 0.292
& 0.745 & 0.316 \\

$H=192$
& 0.890 & 0.392
& 0.773 & 0.420
& 0.806 & 0.422 \\

$H=336$
& 0.957 & 0.611
& 0.810 & 0.475
& 0.856 & 0.511 \\

$H=720$
& 0.950 & 0.561
& 0.845 & 0.540
& 0.863 & 0.566 \\
\bottomrule
\end{tabular}

% }
\caption{AUROC and TPR (at 1\% FPR) for Time-MoE generation detection on 9 datasets. }
\label{tab:TimeMoE}
\end{table}

\section{Discussion}\label{sec:discussion}

Our methodology is based on the contraction hypothesis, which depends on certain idealized assumptions (see \cref{subsec:Contraction Hypothesis}).
Despite the empirical support for the hypothesis from the experimental results, we revisit the foundational assumptions to strengthen our theoretical analysis.

\subsection{Idealized Model Assumption}
To establish an architecture-agnostic detection methodology, we postulate a theoretically optimal Ideal Model.
This abstraction avoids model-specific details to identify general behaviors, but is mathematically unrealizable in practice.
Existing TSLMs are approximations of this ideal, which undermine both the guarantee of perfect prediction (see \cref{theorem:dirac}) and the model's ability to faithfully recover the true series distribution (see \cref{theorem:continuous}).
Despite the limitation, practical models also exhibit similar uncertainty contraction behaviors in forecasting.

In particular, TSLMs are trained on finite datasets with finite parameters.
Since the noise samples are limited and the model is trained to minimize expected loss over these few realizations, it cannot fully capture the underlying noise distribution or the true trend.
Therefore, both the estimated trend $\hat{T}_t$ and noise $\hat{n}_t$ inevitably deviate from the truth,
and the trend deviation $\Delta T_t=|T_t-\hat{T}_t|$ accumulates as forecasting progresses.
Given this non-zero deviation, \cref{theorem:evaluation} does not strictly hold, and reducing forecast variance $\hat{\sigma}_t^2$ does not always improve evaluation performance.
Nevertheless, a moderate reduction in uncertainty constrains the expected evaluation error within the relatively small bound $\Delta T_t$ by concentrating probability between $T_t$ and $\hat{T}_t$.
% Despite that the expectation of model evaluation function is not monotonously increasing in forecast variances $\hat{\sigma}_t^2$ given non-zero trend deviation $\Delta T_t$, lower forecast uncertainty decreases the expectation of $\mathrm{Eva}$ to the bound $\Delta T_t$.
% Intuitively, a lower forecast uncertainty allocates more probability in the relatively small interval between $T_t$ and $\hat{T}_t$.
As the deviation grows, greater uncertainty does not yield sufficiently better performance by correcting the error and may result in unstable forecasts.
% When this deviation grows pronounced as error accumulates, increasing forecast uncertainty does not yield sufficiently better performance by ``pulling back'' the error of $\Delta T_t$.
% Moreover, excessive uncertainty can lead to chaotic or unstable predictions to undermine the practical utility.
This indicates that practical TSLMs also inherently exhibit a forecasting uncertainty contraction (i.e., $\gamma_t<1$), especially under recursive noise accumulation, as elaborated in \cref{subsubsec:recursive variance reduction}.
% This aligns with empirical evidence: Chronos adopts contractive sampling, while Timer and Time-MoE use point-wise loss (e.g., MSE) which enforce contraction (see \cref{corollary:MSE_converge}).
Hence, the contraction hypothesis still holds for practical models, allowing UCE to track predictive uncertainty across recursive steps.

% This indicates that practical TSLMs also inherently exhibit a contraction in forecasting uncertainty and the scaling factor $\gamma_t<1$, and such contraction may be more evident due to recursive noise accumulation under recursive forecasting.
% The contraction hypothesis therefore holds for the practical models, and experiments on multiple TSLMs (e.g., Chronos, Timer and TimeMoE), shown in \cref{fig:Timer_uncertainty} and \cref{fig:TimeMoE_uncertainty}, demonstrate such uncertainty decay.
% As long as the model reasonably approximates both components, the recursive reduction in predictive uncertainty can be reliably tracked by UCE.

\subsection{Gaussian Noise Assumption}
In this work, we assume Gaussian noise for analytical simplification: the sum of independent Gaussian noises remains Gaussian, and the variance fully characterizes the distribution.
Crucially, our proofs do not depend on the exact form of the Gaussian but only require the noise to be unimodal (greatest probability for trend), symmetric (equivalence for both positive and negative biases), and having a finite second moment (existence of variance).
By the additivity of the variance for independent variables, our results naturally extend to such noise models.
These properties are shared by most noise models commonly assumed in real-world time series analysis (e.g., Laplacian), and the experimental results (see Appendix) validate the broad applicability of the contraction hypothesis under various noise types.
% Moreover, by the additivity of the variance for independent variables, our results naturally extend to any noise model satisfying the mild assumptions.
% The experimental result of variation in uncertainty (see Appendix) validates the broad applicability of the contraction hypothesis under various noise types.
% \cref{fig:noise_type} illustrates that, under various types of noise (Laplacian, Cauchy, etc.), the entropy of model‐generated series consistently decreases and remains significantly lower than real series, thus validating the broad applicability of the contraction hypothesis.

\subsection{Modality Difference Analysis}
In this work, we extend textual detection methods to time series via UCE. 
Although uncertainty metrics (e.g., entropy) are relatively simple, we investigate the unique properties of time series versus text modalities (see \cref{ss:analysis}) to prove the optimality of uncertainty in discriminative power.
% In \cref{apdx:reevaluate}, we examine the performance of alternative textual detection methods, confirming the effectiveness of uncertainty metrics.

\section{Conclusion}

We investigate detecting TSLM‑generated time series and hypothesize that they exhibit progressively decreasing uncertainty--unlike real data.
Building on this, we propose the Uncertainty Contraction Estimator (UCE), which captures uncertainty to
distinguish model-generated from real series and is validated both theoretically and empirically.
Future work will further analyze the hypothesis under diverse model architectures and extend UCE to multivariate and batch-forecasting settings.

\section*{Acknowledgements}
This work was supported in part by National Key Research and Development Plan in
China (2023YFC3306100) and National Natural Science Foundation of China (62272372).

\bibliography{aaai2026}
% \input{ReproducibilityChecklist}
% \newpage
\clearpage

\appendix

\section{Definitions}
\label{apdx:definition}

\subsection{Time Series and Dataset}
We begin by defining the notion of time series throughout this work, and then formalize the dataset constituted by the series.

In classical time series analysis, a time series is typically decomposed into seasonal, trend, and noise components\cite{TimeSeriesAnalysis}.
The seasonal and trend components capture the low-frequency deterministic structure of the series, whereas the noise term accounts for the high-frequency stochastic fluctuations.
Henceforth, we merge the seasonal and trend components into a single ``trend term'', representing the deterministic and predictable portion of the series, while the noise term captures the unpredictable component.

\begin{definition}[Time Series Decomposition]\label{def:decomposition}
  Given a bounded discrete time series $X_t, t\in\mathbb{Z}$, it can be decomposed into the \textbf{Trend Sequence} $T_t$ and the \textbf{Noise Process} $n_t$, formulated as \cref{eq:def_decompose}.
\begin{equation}\label{eq:def_decompose}
    X_t=T_t+n_t,
\end{equation}
where $T_t$ is a bounded deterministic sequence such that $\sup_t|T_t|\leq M$, and $n_t$ is an independent Gaussian process with mean $0$.
We model the variances of $n_t$ by \cref{eq:linear_combination}
\begin{equation}\label{eq:linear_combination}
    \sigma_{t}^2=\alpha_{1}\sigma_{t-1}^2+\alpha_2\sigma_{t-2}^2+\ldots+\alpha_l\sigma_{t-l}^2,\; \sum_{i=1}^{l}\alpha_{i}=1
\end{equation}
\end{definition}

We first explain the motivation behind the modeling of the variance sequence in \cref{eq:linear_combination}.
Note the notion of GARCH~\cite{garch}, which is formulated by \cref{eq:GARCH}
\begin{equation}
\label{eq:GARCH}
    \sigma_t^2=\sigma^2\left(1-\alpha-\beta\right)+\alpha \zeta_{t-1}^2+\beta\sigma_{t-1}^2, \alpha+\beta<1.
\end{equation}
GARCH is motivated by the (exponential) regression of variance to the long-run mean $\sigma$ after any arbitrary shock $\zeta_{t}$.
Following this rationale and omitting the shock term $\zeta_{t}$, we replace the ``prior variance'' $\sigma^2$, which requires prior knowledge of the entire variance sequence,  with a ``posterior variance'' defined by the sample average $\bar{\sigma}_t^2=\sum_{i=2}^{l}\alpha_{i}'\sigma_{t-i}^2$, which is a certain weighted mean of sufficiently long historical variances.
\cref{eq:linear_combination} can therefore be reformulated as
\begin{equation}
    \sigma_{t}^2=\alpha_{1}\sigma_{t-1}^2+(1-\alpha)\bar{\sigma}_t^2.
\end{equation}

Note that this decomposition is \emph{unique} as $T_t=\E{X_t}$, $n_t=X_t-T_t$ for all $t$.
In the following analysis, we often relax the Gaussianity of noise and impose the weaker conditions that the distributions are unimodel, symmetric, has well-defined moment and is continuous almost everywhere.

In \cref{def:decomposition}, we treat the trend and noise of time series as independent components.
Building on this, we consider all possible pairings of the trend with all noises.
Let the set of $T$ be $\mathcal{T}$, and denote the set of $n_t$ by $N$.
We have the following definition.

\begin{definition}[Ideal Trend Cluster of $T$]\label{def:Cluster}
  For a trend sequence $T$ and noise process set $N$, \emph{Ideal Trend Cluster} of $T$ is defined as \cref{eq:Cluster}.
  \begin{equation}\label{eq:Cluster}
    T+N=\{T_t+n_t\;|\;n_t\in N\}.
  \end{equation}
  Intuitively, \(T+N\) is the coset of all time series that share the same deterministic trend \(T\) but differ in noise $n_t$.
\end{definition}

Based on the separation of trend and noise, we begin from the perspective of trends to analyze which trend sequences can be used to form the dataset for forecasting.
In time series forecasting, the model leverages past observations (and its internal state) to predict future values.
Wold’s decomposition theory shows that the exact prediction of predictable (trend) components requires infinite past information~\cite{TimeSeriesAnalysis}, and in practical forecasting, this is approximated by leveraging sufficiently long finite histories.
Consequently, at the trend level, the forecast must be a deterministic function of the same history: identical historical inputs cannot produce different trend predictions~\cite{TimeSeriesAnalysis}.
This implies that each distinct trend must correspond to an infinite unique history, and we formalize this distinction in \cref{def:distinct}.

\begin{definition}[Distinct Trends]
\label{def:distinct}
For trend sequences $T_t$, $T'_t$, if for all $\tau\in\mathbb{Z}$ there exists $s\leq \tau$ such that $T_s\neq T'_s$, then $T_t$ and $T'_t$ are \textbf{distinct}.
Otherwise, $T_t$ and $T'_t$ are \textbf{historically equivalent}, denoted as $T_t \sim_{H} T'_t$.
\end{definition}
This indicates that for two distinct trends $T_t, T'_t$, there is no minimal $s$ such that $T_t=T'_t$ for all $t\leq s$.
The relations defined in \cref{def:distinct} allow us to partition the trend set $\mathcal{T}$, enabling the construction of an ``ideal dataset'' whose trend sequences are mutually distinct.

The historical equivalence relationship, by definition, satisfies reflexivity, symmetry, and transitivity and is an equivalence relationship in the trend set. Consequently, it partitions the trend set into disjoint equivalence classes.
\begin{equation}\label{eq:partition}
    \mathcal{T}/\sim_{H}=\bigl\{\left[T\right]\;|\;T\in\mathcal{T}\bigr\},\quad \left[T\right]=\bigl\{T'\; |\; T'\sim_{H} T\bigr\},
\end{equation}
where $\bigcup_{\left[T\right]\in\mathcal{T}/\sim_{H}}\left[T\right]=\mathcal{T}, \left[T\right]\cap \left[T'\right]=\varnothing$ if $\left[T\right]\neq \left[T'\right]$.
Note that any two trends in different equivalence classes are distinct. According to the \textbf{Axiom of Choice}, we select any representative element from each equivalence class to obtain a \textbf{maximally distinct set}, and use it to construct the \textbf{ideal dataset}.

\begin{definition}[Ideal Dataset]\label{def:Dataset}
     A set $S\subseteq \mathcal{T}$ is \textbf{maximally distinct} if
     \begin{enumerate}
         \item (Pairwise distinct) $\forall T_{1}, T_{2}\in S, T_{1}\neq T_{2}\Rightarrow$ $T_{1}$ and $T_2$ are distinct;
         \item (Maximality) $\forall T\in \mathcal{T}, \exists T'\in S$ s.t. $T\sim_{H}T'.$
     \end{enumerate}
     The set of ideal trend clusters for all trend sequences in a \emph{maximally distinct set} $S$ is an \textbf{Ideal Dataset}, which is formulated as \cref{eq:dataset}
  \begin{equation}\label{eq:dataset}
    D=\{T+N\;|\; T_t\in S\}.
  \end{equation}
\end{definition}
In the construction of the maximally distinct set and ideal dataset, we choose representative elements from an uncountably infinite family of non-empty equivalence classes.
Although this construction is non-constructive and not unique, the theoretical existence is guaranteed according to the Axiom of Choice.
Notably, for any two maximally distinct sets $S_1$ and $S_2$, we have $\forall T_t^{(1)}\in S_1\;\exists T_t^{(2)}\in S_{2}$ such that $T_{t}^{(1)}\sim_{H}T_{t}^{(2)}$.
This indicates that the maximally distinct set is \emph{unique}  modulo the equivalence relation $\sim_{H}$.

Given the dataset of intrinsically distinct time series defined above, we next introduce the set of sufficiently long historical sequences used for forecasting.

\begin{definition}[History Set]
    Let $H$ be the minimal length of the forecast history.
    Given an ideal dataset $D$, we define the $k-$\textbf{history set} ($k\geq H$) as \cref{eq:k-history}.
    \begin{equation}\label{eq:k-history}
        \mathrm{HS}_k=\bigl\{\mathbf{X}_{-k:0}\; |\; X_t\in D\bigr\}.
    \end{equation}
    Notably, since $X_t=T_t+n_t$, we have that any two distinct histories in $\mathrm{HS}_{k}$ differ in either trend or noise component.
    The \textbf{history set} of an ideal dataset $D$ is defined as
    \begin{equation}\label{eq:history}
        \mathrm{HS}=\bigcup_{k\geq H}\mathrm{HS}_k.
    \end{equation}
\end{definition}
We first note that choosing $0$ as the historical end point is without loss of generality.
For any $\mathbf{X}_{-k+s:s}$, it can be mapped to $\mathbf{X}'_{-k:0}$, and therefore we have $\bigl\{\mathbf{X}_{-k+s:s}|\; X_t\in D\bigr\}=\bigl\{\mathbf{X}_{-k:0}|\; X_t\in D\bigr\},\forall s$.
For notational consistency, we fix the historical end point as $0$.

Intuitively, the history set $\mathrm{HS}$ contains all sufficiently long distinct sub-series.
Given any history $\mathbf{X}_{-k:0}\in\mathrm{HS}_k$, once a one-step forecast $X_1$ is predicted and appended, we obtain a new history $\mathbf{X}_{-k:1}\in \mathrm{HS}_{k+1}$, thus enabling recursive forecasting over arbitrary horizons.

\subsection{Ideal Model and Loss}
In this section, we introduce the concept of ideal model for time series forecasting.
In the forecasting task, a model takes a sufficiently long history as input and generates the model's \emph{internal probability distribution} to sample prediction values.
For the Large Time Series Models that sample real values at equal intervals into discrete set, formulated as \cref{eq:vocabulary}
\begin{equation}\label{eq:vocabulary}
    \CV=\{v_i\ |\ v_{i+1}-v_{i}=\Delta, -R\leq v_i \leq R \},
\end{equation}
the internal distributions are presented by a discrete probability mass function $p_\theta$ over $\CV$.
More concretely, given any history $\mathbf{X}_{-H:t-1}\in \mathrm{HS}$, the internal probability distribution is $p_{\theta,\Delta}(z_{t}\;|\;\mathbf{X}_{-H:t-1})$.
In the limit $R\to \infty$, $\CV$ extends to a countably infinite discretization of $\mathbb{R}$.
We denote $\CV$ as the \emph{vocabulary} and each $v\in \CV$ is a token, which represents an interval $\left[v-\frac{\Delta}{2},v+\frac{\Delta}{2}\right)$.
We consider our forecasting models and the internal probability distributions over this extended vocabulary.

Formally, we measure the discrepancy between the true next-step outcome and the model’s internal distribution using cross-entropy.
In the following section, we will demonstrate that if losses such as MSE are used, we immediately draw the conclusion of uncertainty convergence, and therefore the selection of cross-entropy is without loss of generality.
With fixed $\Delta$, we first propose the loss function for model training.
\begin{definition}[$\Delta$-Loss Function]
\label{def:delta_loss}
Given history $\mathbf{X}_{-H:0}$ and a fixed forecast length $\tau$,
the \textbf{$\Delta$-Loss Function} of forecast $\mathbf{X}_{\tau}$ is formulated as \cref{eq:delta_cross_entropy}.
\begin{equation}\label{eq:delta_cross_entropy}
  \mathcal{L}_{\theta,\Delta}\left(\mathbf{X}_\tau\right)=
  \sum_{i=1}^{\tau}-\log p_{\theta,\Delta}\left(z_{i}=X_i|\mathbf{X}_{-H:i-1}\right).
\end{equation}
\end{definition}
It should be noted that, given the stochastic process nature of $X_t$, $\mathcal{L}_{\theta,\Delta}\left(\mathbf{X}_\tau\right)$ is a \emph{random variable} dependent on $\mathbf{X}_\tau$.

Note that for any two time points $i\neq j$, the internal distributions $p_{\theta,\Delta}(z_{i}|\mathbf{X}_{-H:i-1})$ and $p_{\theta,\Delta}(z_{j}|\mathbf{X}_{-H:j-1})$ are independent if their histories are given.
Consequently, the loss in \cref{def:delta_loss} decomposes into a sum of mutually independent per-step losses as \cref{eq:per_step_loss}, and the global loss $\mathcal{L}_{\theta,\Delta}$ can be minimized by minimizing each $\ell_{\theta,\Delta}(X_i)$.
\begin{equation}\label{eq:per_step_loss}
    \ell_{\theta,\Delta}(X_i)=-\log p_{\theta,\Delta}\left(z_{i}=X_i|\mathbf{X}_{-H:i-1}\right).
\end{equation}

Based on the loss in \cref{def:delta_loss}, we define the discrete ideal model.
\begin{definition}[$\Delta$-Ideal Model]
\label{def:delta_model}
    For any finite time series segment $\mathbf{X}_{\tau}$ of any ideal trend cluster in the ideal dataset, the mathematical expectation of the $\Delta$-loss function is minimal for the \textbf{$\Delta$-Ideal Model}.
  \begin{equation}\label{eq:Delta_E_loss}
    \theta=\arg\min_\theta \E{\mathcal{L}_{\theta,\Delta}\left({\mathbf{X}}_{\tau}\right)},
    \;
    \forall X\in C(T), \forall C(T) \in D.
  \end{equation}
\end{definition}

We assume the existence of such models, which is discussed in the following paragraphs.
We first demonstrate the following property of the $\Delta$-ideal model.
\begin{restatable}{theorem}{DiscreteID}\label{theorem:discrete identical distribute}
    Let $p_{\theta,\Delta}(z_t)$ denote the $\Delta$-ideal model's internal probability distribution with history $\mathbf{X}_{-H:t-1}$, and denote the probability density function of $X_{t}$ to be $f_t$, we have
    \begin{equation}
        p_{\theta,\Delta}(z_t=v)=\int_{v-\frac{\Delta}{2}}^{v+\frac{\Delta}{2}}f_t(z)\mathrm{d}z.
    \end{equation}
    Let $f_{\theta,\Delta}=\frac{p_{\theta,\Delta}(z_t)}{\Delta}$, we have when $\Delta\to0$,
    \begin{equation}
        f_{\theta,\Delta}\to f_t\; a.e.
    \end{equation}

\end{restatable}
The proof is given in \cref{proof:discrete identical distribute}.

\cref{theorem:discrete identical distribute} indicates that the $\Delta$-ideal model achieves a perfect prediction of the original time series with granularity $\Delta$.
More precisely, the probability mass function of the internal distribution, by endowing the probability of each token $v$ to the corresponding interval $\left[v-\frac{\Delta}{2},v+\frac{\Delta}{2}\right)$, is a simple-function approximation of $f_t$ with granularity $\Delta$, and as $\Delta\to 0$ it converges to $f_t$ in the sense of measures.

On the other hand, as $\Delta\to 0$, the vocabulary $\CV$ converges to $\mathbb{R}$ in the sense of Hausdorff metric. We now turn to the development of a continuous ideal model and its corresponding loss, and proceed to verify that this ``limit model'' is exactly the limit of $\Delta$-ideal model.

\begin{definition}[Loss Function]
\label{def:loss}
    Let $f_\theta(z_{t}\;|\;\mathbf{X}_{-H:t-1})$ be the probability density function of the internal probability distribution given $\mathbf{X}_{-H:t-1}$, the \textbf{Loss Function} is formulated as \cref{eq:cross_entropy}.
    \begin{equation}\label{eq:cross_entropy}
    \begin{split}
        \mathcal{L}_\theta\left(\mathbf{X}_\tau\right)
        &=\sum_{i=1}^{\tau}-\log f_\theta\left(z_{i}=X_{i}|\mathbf{X}_{-H:i-1}\right).
    \end{split}
\end{equation}
\end{definition}

Notably, the loss function is not the point-wise limit of $\Delta$-loss function.
Instead, it emerges through the limit of its normalized form, as shown in \cref{eq:loss_limit}
\begin{equation}\label{eq:loss_limit}
    \lim_{\Delta\to0}\sum_{i=1}^{\tau}-\log \frac{p_{\theta,\Delta}\left(z_i=X_i|\mathbf{X}_{-H:i-1}\right)}{\Delta}=\mathcal{L}_{\theta}(\mathbf{X}_{\tau}).
\end{equation}
With fixed $\Delta$ the left hand side of \cref{eq:loss_limit} is minimized in expectation simultaneously with $\Delta$-loss function.
Therefore, the loss function in \cref{def:loss} can be regarded as the ``limit'' of $\Delta$-loss function.

Similarly to \cref{def:delta_loss}, $\mathcal{L}_\theta\left(\mathbf{X}_\tau\right)$ is also a \emph{random variable}, and the loss alike can be minimized at each step.

\begin{definition}[Ideal Model]\label{def:model}
  For any finite time series segment $\mathbf{X}_{\tau}$ of any ideal trend cluster in the ideal dataset, the mathematical expectation of the loss function is minimal for the \textbf{Ideal Model}.
  \begin{equation}\label{eq:E_loss}
    \theta=\arg\min_\theta \E{\mathcal{L}_\theta\left({\mathbf{X}}_{\tau}\right)},
    \;
    \forall X\in C(T), \forall C(T) \in D.
  \end{equation}
\end{definition}
Assuming the existence of the ideal model, we show that it satisfies the following property.
We show that the ideal model \emph{can} be regarded as a ``limit model'' of the $\Delta$-ideal model in terms of internal probability distribution.
\begin{restatable}{theorem}{ID}\label{theorem:identical distribute}
    Let $f_{\theta}(z_t)$ denote the ideal model's internal probability distribution with history $\mathbf{X}_{-H:t-1}$, and denote the probability density function of $X_{t}$ to be $f_t$, we have
    \begin{equation}
        f_{\theta}(z_t)\equiv f_t,\;a.e.
    \end{equation}

\end{restatable}

% \begin{theorem}\label{theorem:identical distribute}
%     Let $f_{\theta}(z_t)$ denote the ideal model's internal probability distribution with history $\mathbf{X}_{-H:t-1}$, and denote the probability density function of $X_{t}$ to be $f_t$, we have
%     \begin{equation}
%         f_{\theta}(z_t)\equiv f_t,\;a.e.
%     \end{equation}
% \end{theorem}
The proof is given in \cref{proof:identical distribute}.
\cref{theorem:identical distribute} establishes that the internal distribution of ideal model coincides almost everywhere with the true probability distribution, which is exactly the limit distribution of $\Delta$-ideal models when $\Delta\to0$.
This indicates that the ideal model is exactly the ``limit model'' of $\Delta$-ideal model, and the definition is well-defined.

On the other hand, \cref{theorem:identical distribute} shows that, given history $\mathbf{X}_{-H:t-1}$ the ideal model can recover the true probability distribution, indicating the computability of $\sigma_{t}^2$.
Since it follows \cref{eq:linear_combination}, we conclude that the model has internalized this variance recursion.

\subsection{Model Evaluation Function}
In practice, once a forecasting model produces a candidate sequence $\hat{\mathbf{X}}_{\tau}$, we need a principled way to assess its accuracy against the true sequence $\mathbf{X}_{\tau}$, which measure the ``distance'' between $\hat{\mathbf{X}}_{\tau}$ and $\mathbf{X}_{\tau}$.
The notion of distance is grounded in the $\ell_p$ norm in finite dimensions, and drawing inspiration from MAE and MSE, we introduce a positive monotonic function to regulate the error.
We formalize the aforementioned concepts and generalize them as \cref{def:evaluation}.
\begin{definition}[Model Evaluation Function]\label{def:evaluation}
  For any finite time series segment $\mathbf{X}_{\tau}$ of any ideal trend cluster in the Ideal Dataset, and its corresponding forecast series $\mathbf{\hat{X}}_{\tau}$, the \textbf{evaluation function} is formulated as \cref{eq:evaluation}.
  \begin{equation}\label{eq:evaluation}
    \mathrm{Eva}=
    \left[\sum_{i=1}^{\tau}g\left(\left|X_i-\hat{X}_i\right|^{p}\right)\right]^\frac{1}{p},
  \end{equation}
  where $g$ is a monotonically increasing non-negative function and $p\in [1,
  +\infty]$.
\end{definition}
\cref{def:evaluation} introduces a general form of evaluation function that encompasses common metrics such as MAE and MSE.
When $g$ is a positive proportional function, that is, $g(x)=\alpha x, \alpha>0$, the evaluation function is proportional to the norm $\ell_p$.

\subsection{Theoretical Existence of Ideal Model}
\label{subsec:ideal model existence}
The perfect prediction property of ideal model, shown in \cref{theorem:identical distribute}, provides another view of the ideal model, and reduces the problem of loss minimization over all time series to the identification of each single deterministic function.

We first demonstrate that the ideal model, which is a mapping from sequences to distribution, can potentially be an injection, inferring that no two distinct input sequences are forced to produce the same distribution.
The length of history is arbitrary as $k\geq H$, and the histories therefore constitute a set $\bigcup_{H\geq L}\mathbb{R}^{H}$.
On the other hand, since the probability density is continuous almost everywhere, it is completely determined by its values on a countable dense subset of $\mathbb{R}$, and the space of such distributions can be identified with $\mathbb{R}^{\mathbb{N}}$.
The ideal model is therefore the function $f:\bigcup_{H\geq L}\mathbb{R}^{H}\to \mathbb{R}^\mathbb{N}$.
According to the Cantor-Bernstein theorem, we have
\begin{equation}
    \mathrm{Card}\left(\bigcup_{H\geq L}\mathbb{R}^{H}\right)=\mathrm{Card}\left(\mathbb{R}^\mathbb{N}\right),
\end{equation}
which provides the prerequisite for the injection condition.

We next provide that, at any prescribed precision, the ideal model's mapping $f$ is Type-2 computable.
Moreover, every finite-term approximation of $f$ (e.g., simple function approximation) is computable, and is applied in real-world model applications.
Since sequence modeling architectures such as RNNs~\cite{RNN_Turing,NNTuring} and Transformers~\cite{TransformerTuring,AttentionTuring} are Turing-complete and can approximate any computable or Type-2 computable functions with infinite memory, runtime and arbitrary precision.
This indicates the theoretical existence of the ideal model.

\section{Theoretical Analysis of Contraction Hypothesis}
\label{apdx:proof}

In \cref{ss:analysis}, we provided a qualitative outline of the proof strategy
for the Contraction Hypothesis.
In this section, we formally state this hypothesis and provide the proofs of all
critical propositions.

\subsection{Distributional Consistency}

In \cref{def:model} and \cref{theorem:identical distribute}, we propose the
concept of ideal model and generalize that it makes distribution-perfect
predictions given histories.
The concept of ideal model derives directly from its discrete form (see
\cref{def:delta_model}) with a corresponding property (see
\cref{theorem:discrete identical distribute}).
We now prove the properties of the ideal models in both discrete and continuous
scenarios.
\DiscreteID*
\begin{proof}[Proof of \cref{theorem:discrete identical distribute}]
\label{proof:discrete identical distribute}
    We begin from a lemma.
    \begin{lemma}\label{lemma:discrete}
  Given positive integer $n\geq 1$ and fixed positive reals $a_1,\ldots, a_n$
  such that $\sum_{i=1}^{n}a_{i}=1$, for positive reals $b_{1},\ldots, b_{n}$
  such that $\sum_{i=1}^{n}b_{i}=1$, the following equation holds:
  \begin{equation}
    \sum_{i=1}^{n}a_{i}\ln b_{i}\leq \sum_{i=1}^{n}a_{i}\ln a_{i},
  \end{equation}
  and equation holds if and only if $a_{i}=b_{i}, \forall i$.
\end{lemma}

\begin{proof}[Proof of \cref{lemma:discrete}]
  Under the constraint of $\sum_{i=1}^{n}b_{i}=1$, the Lagrange function is
  formulated as follows:
  \begin{equation}
    L= \sum_{i=1}^{n}a_{i}\ln b_{i}-\lambda\left(\sum_{i=1}^{n}b_{i}-1\right).
  \end{equation}
  The partial derivatives of the function are
  \begin{equation}
    \frac{\partial L}{\partial b_{i}} =\frac{a_{i}}{b_{i}}-\lambda, \quad
    \frac{\partial^2 L}{\partial b_{i}^2} =-\frac{a_{i}}{b_{i}^2}, \quad
    \frac{\partial^2 L}{\partial b_{i}\partial b_{j}} =0, \quad i\neq j,
  \end{equation}
  and the extremum condition is satisfied when $\frac{a_{i}}{b_{i}}=\lambda$.
  According to the constraint, we have $\lambda=1, a_{i}=b_{i}$.
  At this point, the Hessian matrix is negative definite, indicating that a
  local maximum is obtained.

  Due to the strict concavity of the logarithmic function, it follows that $L$
  is strictly concave and therefore the maximum value is achieved at the local maximum.
\end{proof}
We have
\begin{equation}
\begin{split}
    &\E{\mathcal{L}}
    \\&=\sum_{v\in\mathcal{V}}P\left\{v-\frac{\Delta}{2}\leq X_{i}\leq v+\frac{\Delta}{2}\right\}\left(-\log p_{\theta,\Delta}\left(z_{i}=v\right)\right)
    \\&=\sum_{v\in\mathcal{V}}\int_{v-\frac{\Delta}{2}}^{v+\frac{\Delta}{2}}f_{t}(z)\mathrm{d}z\left(-\log p_{\theta,\Delta}\left(z_{i}=v\right)\right).
\end{split}
\end{equation}
According to \cref{lemma:discrete}, the expectation is minimized when
\[p_{\theta,\Delta}(z_t=v)=\int_{v-\frac{\Delta}{2}}^{v+\frac{\Delta}{2}}f_t(z)\mathrm{d}z\]
holds for all $v\in\CV$.
Notably, we have $f_{\theta,\Delta}$ is measurable for all $\Delta>0$ and $f_t$ is measurable.
    According to Lebesgue differential theorem, we have
    \[\lim_{\Delta\to0}f_{\theta,\Delta}(v)=\lim_{\Delta\to 0}\frac{1}{\Delta}\int_{v-\frac{\Delta}{2}}^{v+\frac{\Delta}{2}}f_{t}(z)\mathrm{d}z=f_{t}(v)\]
    for almost every $v\in\mathbb{R}$.
\end{proof}

\ID*
\begin{proof}[Proof of \cref{theorem:identical distribute}]
\label{proof:identical distribute}
We begin with a lemma.
\begin{lemma}\label{lemma:continuous}
  Let $f$ be a fixed positive measurable function such that
  $\int_{-\infty}^\infty  f(x)\dd x=1$.
  For a positive measurable function $g$ such that $\int_{-\infty}^\infty
  g(x)\dd x=1$, then the following equation holds:
  \[
    \int_{-\infty}^\infty f(x)\ln g(x)\dd x
    \leq \int_{-\infty}^\infty f(x)\ln f(x)\dd x.
  \]
  The maximum value occurs if and only if $f=g$ almost everywhere.
\end{lemma}

\begin{proof}[Proof of \cref{lemma:continuous}]
  We construct the Lagrange functional as follows:
  \[
    L(g)=\int_{-\infty}^\infty f(x)\ln g(x)\dd x-\lambda\left(\int_{-\infty}^\infty g(x)\dd x-1\right)
  \]
  The variations of this functional are
  \begin{equation}
  \begin{split}
     &\delta L(g)=
    \int_{-\infty}^\infty \left(\frac{f(x)}{g(x)}-\lambda\right)\delta b(x)\dd x,
    \\&\delta^2 L(g)=
    \int_{-\infty}^\infty -\frac{f(x)}{g(x)^2}\left(\delta b(x)\right)^2\dd x.
  \end{split}
  \end{equation}
  The functional extremum condition holds when $\delta L(g)\equiv 0,\forall b$, indicating $\frac{f(x)}{g(x)}-\lambda=0$ almost everywhere.
  According to the constraint, we have $\lambda=1$, and $f(x)=g(x)$ holds almost everywhere.
  At this point, the second variation is always less than $0$, indicating a local maximum.

  Note that for any positive measurable function $g_1, g_2$ and any $\alpha\in
  \left[0,1\right]$, we have
  \begin{align*}
    &L\left(\alpha g_1(x)+(1-\alpha)g_2(x)\right) \\
    &=\int_{-\infty}^\infty f(x)\ln\left(\alpha g_1(x)+(1-\alpha)g_2(x)\right)\dd x \\
    &\leq \int_{-\infty}^\infty \alpha f(x)\ln g_1(x)+f(x)(1-\alpha)\ln g_2(x)\dd x \\
    &=\alpha L(g_1(x)) + (1-\alpha)L(g_2(x)).
  \end{align*}
  Thus, the functional $L$ is concave and the maximum value is obtained when $f=g$ holds almost everywhere.
\end{proof}
    For any time series, we have
  \begin{equation}
    \begin{split}
      \E{\mathcal{L}}=-\int_{-\infty}^\infty f_{t}(v)\log f_\theta(z_t=v)\dd v
    \end{split}
  \end{equation}
  According to \cref{lemma:continuous}, the expectation of loss is minimized when $f_\theta=f_{t}$ holds almost everywhere.
\end{proof}

Specifically, we propose a corollary of \cref{theorem:identical distribute} (see \cref{theorem:dirac}) that if the true sequence is entirely deterministic (i.e., has zero uncertainty), ideal models make value-perfect predictions.
\MainBodyDirac*
\begin{proof}[Proof of \cref{theorem:dirac}]
\label{proof:dirac}
    \cref{theorem:dirac} is a special case of \cref{theorem:continuous} where $f_t=\delta$.
    In this proof, we analyze the discrete form, as the continuous ideal model is the limit of the discrete $\Delta$-ideal model.

  According to Chebyshev's inequality $P\left\{X_t=T_t\right\}=1$.
  We calculate the loss in this time point.
  \begin{equation}
      \begin{split}
          \E{\mathcal{L}(X_t)}
      &=\sum_{v\in\CV}P\left\{X_t=v\right\}\cdot\left(-\log p_\theta\left(z_t=v\right)\right)
      \\&=-\log p_\theta\left(z_t=T_t\right)\geq 0
      \end{split}
  \end{equation}
  According to \cref{lemma:discrete}, the equation holds if and only if $\arg\max_{v}p_\theta(z_t=v)=T_t,
  p_\theta(z_t=T_t)=1$, with minimal loss expectation.
\end{proof}

\subsection{Sampling-Induced Variance Scaling}
When given any history ${\mathbf{X}}_{-H:t-1}$ to predict $\hat{X}_t$, TSLMs utilize the internal probability distributions $f_{\theta}$ to obtain the sampling distributions $\hat{f}_t$.
The modification of $f_{\theta}$ derives from sampling strategies, including scaling and symmetric truncation.
We first prove the following property of the strategies.
\begin{theorem}\label{theorem:modification}
    The sampling strategies preserves the expectation of $f_{\theta}$.
\end{theorem}
\begin{proof}[Proof of \cref{theorem:modification}]
    Denoting $\int_{-\infty}^{\infty}f(z)\mathrm{d}z=E_z$.
        For any symmetric probability density function $f$ over $\mathbb{R}$, the scaling transformation can be formulated by \cref{eq:scaling}
        \begin{equation}\label{eq:scaling}
            f_{T}(z)=\frac{f(z)^{\frac{1}{T}}}{\int_{-\infty}^{\infty}f(z)^{\frac{1}{T}}\mathrm{d}z}.
        \end{equation}
        Since $f(E+z)=f(E-z)$ holds for all $z\in\mathbb{R}$, we therefore have
        \begin{equation}
        \begin{split}
            &\int_{-\infty}^{\infty}\frac{zf(z)^{\frac{1}{T}}}{\int_{-\infty}^{\infty}f(x)^{\frac{1}{T}}\mathrm{d}x}\mathrm{d}z
            \\&=\frac{\int_{-\infty}^{\infty}E_zf(z)^{\frac{1}{T}}+\int_{-\infty}^{\infty}\left(z-E_z\right)f(z)^{\frac{1}{T}}}{\int_{-\infty}^{\infty}f(x)^{\frac{1}{T}}\mathrm{d}x}
            =E_z.
        \end{split}
        \end{equation}
        The symmetric truncation can be formulated by
        \begin{equation}
            f_{a}(z)=\left\{\begin{matrix}
                \frac{f(z)}{\int_{E_z-a}^{E_z+a}f(x)\mathrm{d}x},&E_z-a\leq z\leq E_z+a;\\
                0,&\text{otherwise}.
            \end{matrix}\right.
        \end{equation}
        We have
        \begin{equation}
            \int_{-\infty}^{\infty}zf_{a}(z)\mathrm{d}z=\int_{E_z-a}^{E_z+a}\frac{zf(z)}{\int_{E_z-a}^{E_z+a}f(x)\mathrm{d}x}=E_z.
        \end{equation}
\end{proof}

Denoting the uncertainty modification of the sampling strategies from internal distribution variance $\tilde{\sigma}_{t}^2$ to sampling distribution variance $\hat{\sigma}_t^2$ as
\begin{equation}
    \hat{\sigma}_{t}^2=\gamma_t\cdot\tilde{\sigma}_{t}^2,
\end{equation}
and in recursive forecasting we have
\begin{theorem}\label{theorem:expectation}
    Given history $\mathbf{X}_{-H:0}$, we have that for any $\tau>0$, there is
    \begin{equation}
        \E{\hat{X}_t}=T_t,\; \forall 1\leq t\leq \tau.
    \end{equation}
\end{theorem}

\begin{proof}[Proof of \cref{theorem:expectation}]
    We define $\hat{X}_0=X_0$ as they share the same history.
    For the forecast of $\hat{X}_1$ with $\mathbf{X}_{-H:0}$, we have $\int_{-\infty}^{\infty}f_{\theta}(z_1)\mathrm{d}z_1=\int_{-\infty}^{\infty}f_{1}(z)\mathrm{d}z=T_{1}$ due to \cref{theorem:continuous}.
    According to \cref{theorem:modification}, we have
    \begin{equation}
        \E{\hat{X}_1}=\int_{-\infty}^{\infty}f_{\theta}(z_1)\mathrm{d}z_1=T_{1}.
    \end{equation}
    Given that $\E{\hat{X}_t}=T_t$ holds for all $t\leq s-1$, we have that with history $\mathbf{X}_{-H:0},\hat{\mathbf{X}}_{\tau-1}$.
    With knowledge of this and according to \cref{theorem:continuous}, the internal probability distribution at time $\tau$ is of expectation $T_{\tau}$.
    According to \cref{theorem:modification}, we have
    \begin{equation}
    \label{eq:expectation success}
        \E{\hat{X}_\tau}=\int_{-\infty}^{\infty}f_{\theta}(z_\tau)\mathrm{d}z_\tau=T_{\tau}.
    \end{equation}
    By induction, \cref{eq:expectation success} holds for all $1\leq t\leq \tau$ for all possible $\tau$.
\end{proof}

Given the trend-perfect forecast for arbitrary sampling strategies or $\gamma_t$ selection, we show that when $f_{\theta}$ is fixed, the modification $\gamma_t$ influences the forecast and its evaluation.

\MainBodyEva*
\begin{proof}[Proof of \cref{theorem:evaluation}]
\label{proof:evaluation}
Let $Z_t=X_t-\hat{X}_t\sim \mathcal{N}\left(0,\sigma_t^2+\hat{\sigma}_t^2\right)$.
We begin by presenting a definition and its associated properties.

\begin{definition}[First-Order Stochastic Dominance]
\label{def:dominance}
    For variables $z^1$ and $z^2$, if
    \begin{equation}
        P\{|z^1|\geq z\}\leq P\{|z^2|\geq z\},\;\forall z\in\mathbb{R},
    \end{equation}
    then $z^2$ \textbf{first-order stochastically dominates} $z^1$, denoted by $z^{1}\leq_{st} z^2$
\end{definition}
\begin{lemma}[Gaussian First-Order Stochastic Dominance]
\label{lemma:gaussian dominance}
    For variables $z^{1}\sim \CN(0,\sigma_1^2)$ and $z^2\sim\CN(0,\sigma_2^2)$, if $\sigma_1^2\leq \sigma_2^2$, then $|z^{1}|\leq_{st} |z^2|$.
\end{lemma}
\begin{proof}[Proof of \cref{lemma:gaussian dominance}]
    Given two Gaussian variables $z^1$, $z^2$ with their corresponding folded normal variables $|z^1|, |z^2|$ and probability distributions  $f_{|z^{1}|}, f_{|z^{2}|}$. Since
    \begin{equation}
    \begin{split}
        &P\{|z^1|\geq z\}=2\left[1-\Phi(\frac{z}{\sigma_1})\right]=\frac{2}{\sqrt{\pi}}\int_{\frac{z}{\sqrt{2}\sigma_1}}^{\infty}\mathrm{e}^{-t^2}\mathrm{d}t,
        \\&P\{|z^2|\geq z\}=\frac{2}{\sqrt{\pi}}\int_{\frac{z}{\sqrt{2}\sigma_2}}^{\infty}\mathrm{e}^{-t^2}\mathrm{d}t,
    \end{split}
    \end{equation}
    holds for all $z$, we have $P\{|z^1|\geq z\}\leq P\{|z^2|\geq z\}$ for all $z$.
    This indicates that $|z^1|\leq_{st}|z^2|$.
\end{proof}

\begin{lemma}[Preservation under Monotonic Transformations]
\label{lemma:perservation}
    For a monotonically increasing function $g$ and variables $z^1$ and $z^2$, if $z^1\leq_{st} z^2$, then $g(z^1)\leq_{st} g(z^2)$.
\end{lemma}
\begin{proof}[Proof of \cref{lemma:perservation}]
    For all $z$, we have
    \begin{equation}
    \begin{split}
        P\{g(z^1)\geq z\}&=P\{z^1\geq g^{-1}(z)\}
        \\&\leq P\{z^2\geq g^{-1}(z)\}= P\{g(z^2)\geq z\},
    \end{split}
    \end{equation}
    indicating $g(z^1)\leq_{st} g(z^2)$.
    Notably, here $g^{-1}$ is the \emph{generalized inverse} defined as
    \begin{equation}
        g^{-1}(y)=\inf\{x\;|\;g(x)\leq y\}.
    \end{equation}
\end{proof}

\begin{lemma}[Additivity over Independent Variables]
\label{lemma:additivity}
    For variables $z^1$, $z^2$ and $z^3$, if $z^1\leq_{st} z^2$ and $z^1$, $z^2$ are independent of $z^3$, then $z^1+z^3\leq_{st} z^2+z^3$.
\end{lemma}
\begin{proof}[Proof of \cref{lemma:additivity}]
    For all $z$, we have
    \begin{equation}
        \begin{split}
            &P\left\{z^{1}+z^{3}\geq z\right\}=\int P\{z^{1}\geq z-v\}\mathrm{d}F_{z^3}(v)
            \\&\leq \int P\{z^{2}\geq z-v\}\mathrm{d}F_{z^3}(v)
            =P\{z^2+z^3\geq z\},
        \end{split}
    \end{equation}
    indicating $z^1+z^3\leq_{st} z^2+z^3$.
\end{proof}

\begin{lemma}[Expectation Dominance over Utility]
\label{lemma:utililty expectation}
For variables $z^1$ and $z^2$, if $z^1\leq_{st} z^2$, then for every monotonically increasing function $u$, we have
\begin{equation}
    \E{u(z^1)}\leq \E{u(z^2)}
\end{equation}
if both expectations exist.
\end{lemma}
\begin{proof}[Proof of \cref{lemma:utililty expectation}]
Let
\[u^{+}(z)=\max\left\{u(z),0\right\},\; u^{-}(z)=\max\left\{-u(z),0\right\},\]
therefore $u=u^{+}-u^{-}$ with both $u^{+}, u^{-}\geq 0$ and $u^{+},-u^{-}$ increasing monotonically.
If and only if $\E{|u(z)|}$ exists, $\E{u(z)}$ exists with \[\E{u(z)}=\E{u(z^{+})}-\E{u(z^{-})}\]
potentially for $z=z^1$ or $z=z^2$. Therefore, without loss of generality, let $u(z)\geq 0$ (because for $u^{-}$ there is a similar conclusion).

Therefore, for non-negative variables $u(z^1)$, $u(z^2)$, we have
\begin{equation}
\begin{split}
    &\E{u(z^1)}=\int_{0}^{\infty}P\{u(z^1)\geq u\}\mathrm{d}u,
    \\& \E{u(z^2)}=\int_{0}^{\infty}P\{u(z^2)\geq u\}\mathrm{d}u.
\end{split}
\end{equation}
Given $z^1\leq_{st} z^2$ and according to \cref{lemma:perservation}, we have $u(z^1)\leq_{st} u(z^2)$, indicating $P\{u(z^1)\geq u\}\leq P\{u(z^2)\geq u\},\forall u>0$.
Therefore,
\begin{equation}
\begin{split}
    &\E{u(z^1)}-\E{u(z^2)}
    \\&=\int_{0}^{\infty}P\{u(z^1)\geq u\}-P\{u(z^2)\geq u\}\mathrm{d}u\leq 0.
\end{split}
\end{equation}
\end{proof}

Given the real time series segment $\mathbf{X}_{\tau}$ and any forecast series $\hat{\mathbf{X}}_\tau$, we define $\mathbf{Z}_{\tau}=\mathbf{X}_{\tau}-\hat{\mathbf{X}}_\tau$ where each $Z_t \sim \CN\left(0,\sigma_t^2+\hat{\sigma}_t^2\right)$ is independently distributed.
With an arbitrary choice of $1\leq i\leq \tau$, we consider a varying variance of $Z_i$ and set all other variances of $Z_{t} (t\neq i)$ fixed.
We define $Z:=\sum_{j\neq i}g(|Z_{j}|^p)$, which is independent from $Z_i$.
If it is replaced by $Z'_i\sim \CN\left(0,\sigma_t^2+\hat{\sigma}_t'^2\right)$ such that $\hat{\sigma}_t'^2\geq \hat{\sigma}_t^2$, according to \cref{lemma:gaussian dominance} we have
$|Z_i|\leq_{st}|Z'_{i}|$.
According to \cref{lemma:perservation}, we have $g(|Z_i|^{p})\leq_{st}g(|Z'_{i}|^p)$.
According to \cref{lemma:additivity}, we have
\begin{equation}
\begin{split}
    &\sum_{j=1}^{\tau}g\left(\left|X_j-\hat{X}_j\right|^{p}\right)
    =g\left(\left|Z_i\right|^p\right)+\sum_{j\neq i}g\left(\left|Z_j\right|^p\right)
    \\&=g\left(\left|Z_i\right|^p\right)+Z\leq_{st}g\left(\left|Z'_i\right|^p\right)+Z
    \\&=g\left(\left|Z'_i\right|^p\right)+\sum_{j\neq i}g\left(\left|Z_j\right|^p\right).
\end{split}
\end{equation}
According to \cref{lemma:utililty expectation} and let $u(z)=z^{\frac{1}{p}}$, we have
\begin{equation}
    \begin{split}
    &\E{\left[\sum_{j=1}^{\tau}g\left(\left|Z_j\right|^p\right)\right]^{\frac{1}{p}}}
        \\&\leq\E{\left[g\left(\left|Z'_i\right|^p\right)+\sum_{j\neq i}g\left(\left|Z_j\right|^p\right)\right]^{\frac{1}{p}}}.
    \end{split}
\end{equation}
This indicates that for all $1\leq i\leq \tau$, we have $\E{\mathrm{Eva}}$ is monotonically increasing in $\hat{\sigma}_t^2$.
\end{proof}

Notably, we draw a corollary from \cref{theorem:evaluation} that, if metrics such as MSE are used as training loss, the conclusion of uncertainty convergence is immediate.
\begin{corollary}
    \label{corollary:MSE_converge_apdx}
    If the loss function is defined as the model evaluation function $\mathrm{Eva}$, which is
    \begin{equation}
        \mathcal{L}=
    \left[\sum_{i=1}^{\tau}g\left(\left|X_i-z_i\right|^{p}\right)\right]^\frac{1}{p},
    \end{equation}
    we have
    \begin{equation}
        f_{\theta}(z=v)=\delta(v-T_t)
    \end{equation}
    for all $\sigma^2\geq0$.
\end{corollary}
\begin{proof}[Proof of \cref{corollary:MSE_converge_apdx}]
    According to \cref{theorem:evaluation}, the expectation of $\mathcal{L}$ is monotonically increasing in the variance of $X_i-z_i$, it is minimized if and only if $\tilde{\sigma}_t^2=0$ for all $i$ and the expectation to be $T_t$, and therefore we obtain that
    \begin{equation}
        f_{\theta}(z=v)=\delta(v-T_t), a.e.
    \end{equation}
\end{proof}

Since $\hat{\sigma}_t^2=\gamma_t\tilde{\sigma}_t^2=\gamma_t\sigma_t^2$, we have that the lower each $\gamma_t$, the lower the model evaluation function's value, indicating better performance.
Typically, $\gamma_t<1$ is universally adopted (e.g., lowering temperature) at each time $t$.

\subsection{Recursive Variance Reduction}
Building on the intuition of $\gamma_t$, we further analyze the variance evolution in the process of forecasting.
According to \cref{theorem:continuous}, the ideal model reproduces any distribution with variance satisfying $\sigma_{t}^2=\sum_{i=1}^{l}\alpha_{i}\sigma_{t-i}^2$.
When the forecasts $\hat{X}_1, \hat{X}_2,\ldots$ are fed back as histories, their variances $\hat{\sigma}_{1}^2,\hat{\sigma}_2^2,\ldots$ involve in this recurrence, and we have \cref{eq:ap_recursion}.
\begin{equation}\label{eq:ap_recursion}
    \tilde{\sigma}_t^2=\sum_{i=1}^{l}\alpha_{i}\hat{\sigma}_{t-i}^2=\sum_{i=1}^{l}\alpha_{i}\gamma_{t-i}\tilde{\sigma}_{t-i}^2,\;\sum_{i=1}^{l}\alpha_i=1.
\end{equation}
If $\gamma_t\leq q<1$ holds for all $t$, we consider the linear recurrence sequence defined by
\begin{equation}
    \tilde{\sigma}_t^2\leq \sum_{i=1}^{l}q\alpha_{i}\tilde{\sigma}_{t-i}^2,
\end{equation}
whose general term can be written by
\begin{equation}\label{eq:recurrence_general}
    \tilde{\sigma}_t^2=P_1r_1^t+P_2r_2^t+\ldots+P_jr_j^t,
\end{equation}
where $r_1,\ldots,r_j$ are distinct roots of the characteristic polynomial and $\mathrm{deg}(P_1)+1,...,\mathrm{deg}(P_j)+1$ are their respective multiplicities.
Given \cref{eq:recurrence_general}, we draw the following conclusion.
\begin{theorem}\label{theorem:root_below_1}
    $|r_i|<1$ holds for all $1\leq i \leq l$.
\end{theorem}
\begin{proof}[Proof of \cref{theorem:root_below_1}]
    Since $r_i$ is the root of the characteristic polynomial
    \begin{equation}
        r^{l}-\sum_{i=1}^{l}q\alpha_{i}r^{l-i},
    \end{equation}
    if there exists $i$ such that $|r_k|\geq1$, we have
    \begin{equation}
        \begin{split}
            |r_k|^{l}&=\left|\sum_{i=1}^{l}q\alpha_{i}r_k^{l-i}\right|
            \leq \sum_{i=1}^{l}q\alpha_{i}\left|r_k\right|^{l-i}
            = \sum_{i=1}^{l}q\alpha_{i}\left|r_k\right|^{-i}
            \\&\leq|r_k|^l\sum_{i=1}^{l}q\alpha_{i}<|r_k|^{l},
        \end{split}
    \end{equation}
    which leads to a contradiction, thereby confirming the original conclusion.
\end{proof}

According to \cref{eq:recurrence_general}, we have that when $t\to\infty$, $\tilde{\sigma}_t^2$ converges to $0$.
By contrast, when $\gamma_t\geq q>1$ holds for all $t$, we have the opposite conclusion that $\exists i$ such that $|r_i|>1$ and $\tilde{\sigma}_t^2$ diverges to infinity.

When $\gamma_t\equiv 1$, we have \cref{eq:ap_recursion} to be refomulated by
\begin{equation}
    \tilde{\sigma}_t^2=C_1\cdot1+P_2r_2^t+\ldots+P_jr_j^t,
\end{equation}
indicating that $1$ is a simple root of the characteristic equation.
Similarly to \cref{theorem:root_below_1}, we have $|r_i|\leq 1$ for all $2\leq i\leq j$.
Therefore, $\tilde{\sigma}_t^2$ is either periodical or converging to a constant $C_1$.
This represents the ``best‑case scenario'' for an idealized model: regardless of how the original series fluctuates, it perfectly captures the true trend and noise distribution and faithfully propagates them to the next time step.
In practical models, however, this property often leads to noise accumulation, which in turn distorts long‑horizon forecasts.

\section{Baselines}\label{apdx:baselines}
In this section, we introduce a variety of representative baselines that span likelihood-based, perturbation-based, rank-based, and geometry-based detection strategies.
These methods are carefully adapted to our time series scenario and we use consistent settings to ensure comparability, thereby providing a robust foundation for benchmarking.

\subsection{DNA-GPT (WScore)}
DNA-GPT WScore~\cite{dnagpt} measures the average similarity between the regenerated text and the original suffixes of the truncated input sequences.
The approach is based on the hypothesis that, conditioned by a fixed prefix, the log-likelihood of the machine-generated remaining text substantially exceeds that of human-generated text.
Operationally, DNA-GPT truncates each input sequence in a specified ratio, generates multiple continuations from the retained prefix, and evaluates the average difference in log-probability between the regenerated and original suffixes as the detection signal.
In this study, we adopt the white-box version of DNA-GPT as a baseline since the black-box metric is inapplicable in the context of our problem.

We follow the original parameter configuration reported in~\cite{dnagpt}.
In particular, we set the truncation ratio (i.e., the proportion of regenerated tokens) to $\alpha=0.5$, and the regeneration iterations to be $10$.
% Later

\subsection{DetectGPT}
DetectGPT~\cite{detectgpt} builds on the hypothesis that LLM-generated text tends to lie in regions of negative curvature in the model’s log-probability landscape.
As a result, such text is more sensitive to perturbations, exhibiting greater reductions in log-likelihood when rephrased or slightly modified, compared to human-written text.
Building on this hypothesis, DetectGPT generates multiple semantically equivalent perturbations of the input sequence and compares the log-probability of the original sample with each perturbed sample.
The perturbations are typically produced using pre-trained models such as T5, via operations like paraphrasing or word deletion that aim to preserve the input's semantics.
Note that the ``black-box'' version of DetectGPT assumes no access to the original generation model but still requires access to the logit distributions of a surrogate language model. In practice, this setup aligns with the zero-shot scenario considered in our experiments. Therefore, we treat it as a white-box method in this context.

In our adaptation of DetectGPT to time series data, perturbations are introduced by randomly selecting a subset of tokens and substitute them with alternatives from their local temporal neighborhood (within a window size of $2$) to preserve semantics based on similarity.
We set the perturbation ratio (i.e., the proportion of perturbed tokens) to $\beta=0.3$ and generate $10$ perturbations for each input sequence, consistent with the configuration used in the original work.

\subsection{Fast-DetectGPT}
Fast-DetectGPT~\cite{fast_detectgpt} is a computationally efficient variant of DetectGPT that significantly reduces inference cost by computing log probabilities only for the perturbed tokens, rather than for the entire sequence.
This selective computation strategy preserves and even improves detection performance while improving efficiency.
In our experiment, we adopt the same parameter configuration as DetectGPT, with the perturbation ratio $\beta=0.3$ and $10$ perturbations generated for each input sequence.

\subsection{DetectLLM}
DetectLLM~\cite{detectllm} concentrates on the probability rank of tokens, and consists of two metrics, namely Log-Likelihood Log-Rank Ratio (LRR) and  Normalized Perturbed log-Rank (NPR).
LRR measures the average log probabilities of all tokens in the sample against the average log rank.
NPR is based on the log rank difference between original and perturbed samples, where multiple semantically equivalent perturbations are generated, and the log rank differences between the original sequence and the perturbations are estimated.

In our adaptation of DetectLLM to time series data, perturbations are introduced in the same manner as DetectGPT for NPR method, and we follow the same parameter configuration as DetectGPT, with the perturbation ratio $\beta=0.3$ and $10$ perturbations generated for each input sequence.
\begin{figure*}[t!]
    \centering
    \subfloat[Entropy]{
    \includegraphics[width=.32\linewidth]{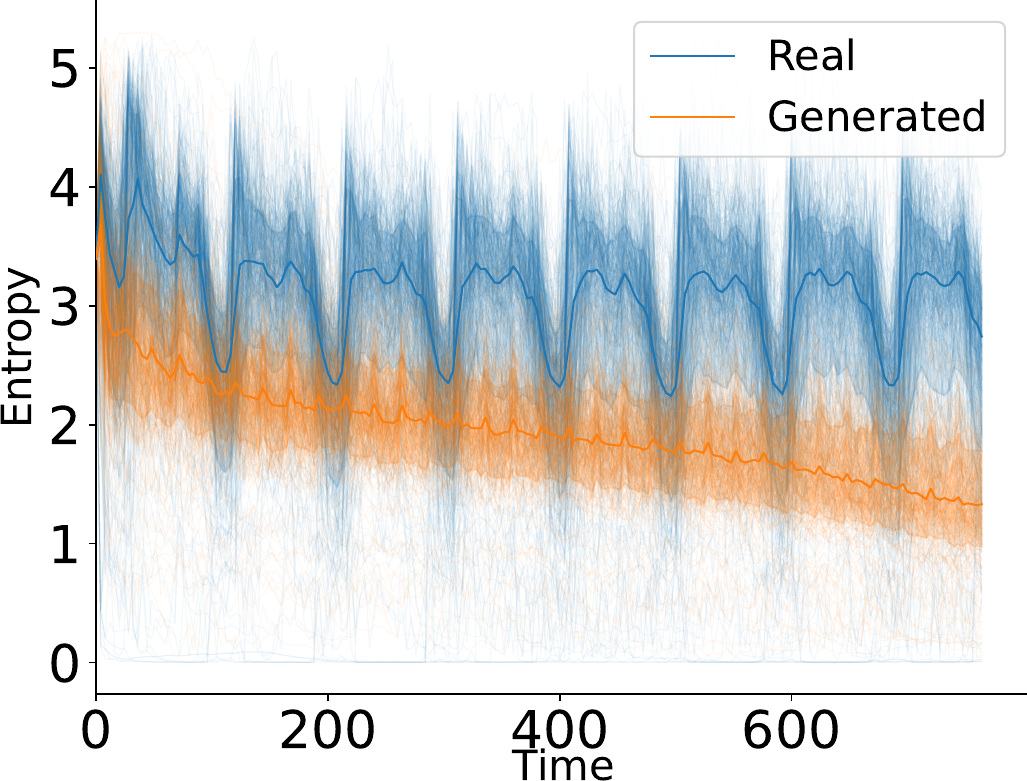}
    }
    \subfloat[Max-Probability]{
    \includegraphics[width=.32\linewidth]{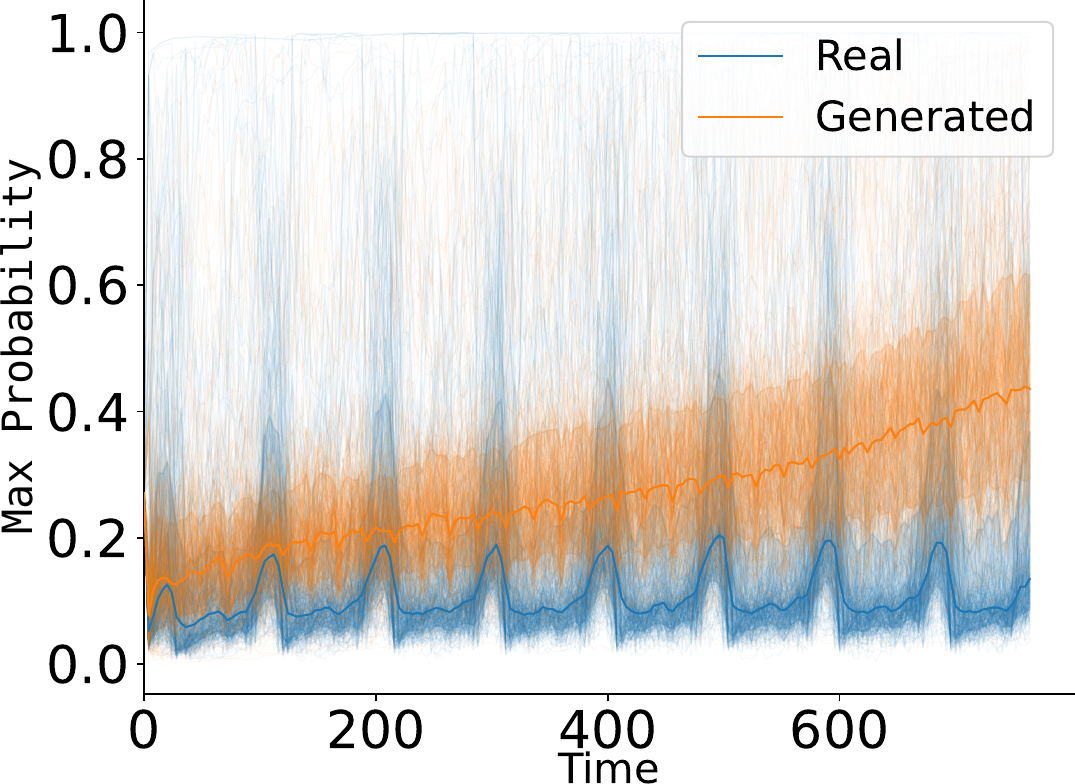}
    }
    \subfloat[Variance]{
    \includegraphics[width=.32\linewidth]{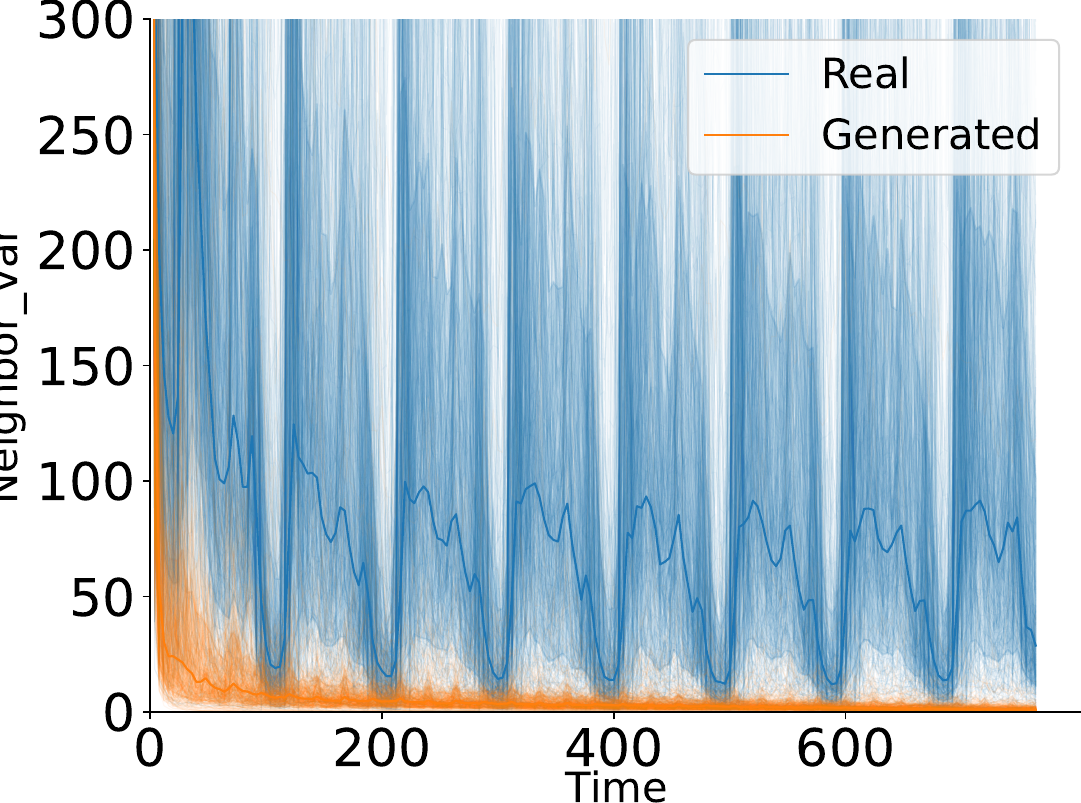}
    }
    \caption{Uncertainty reduction for Timer-generated time series.}
    \label{fig:Timer_uncertainty}
\end{figure*}

\begin{figure*}[t!]
    \centering
    \subfloat[Entropy]{
    \includegraphics[width=.32\linewidth]{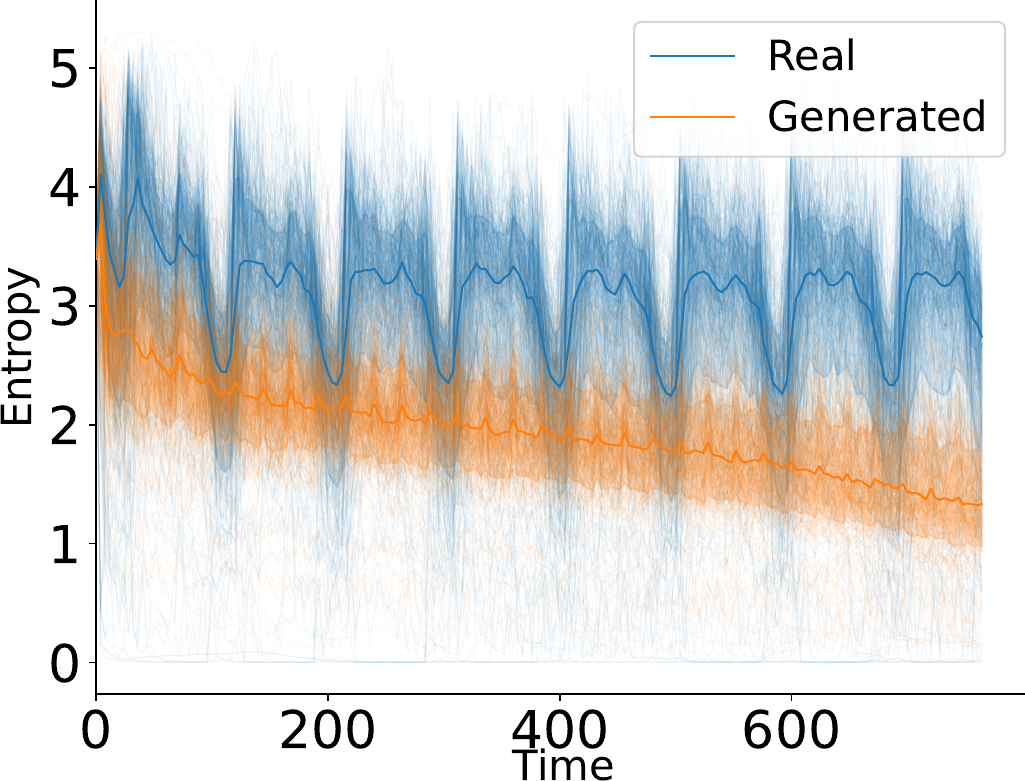}
    }
    \subfloat[Max-Probability]{
    \includegraphics[width=.32\linewidth]{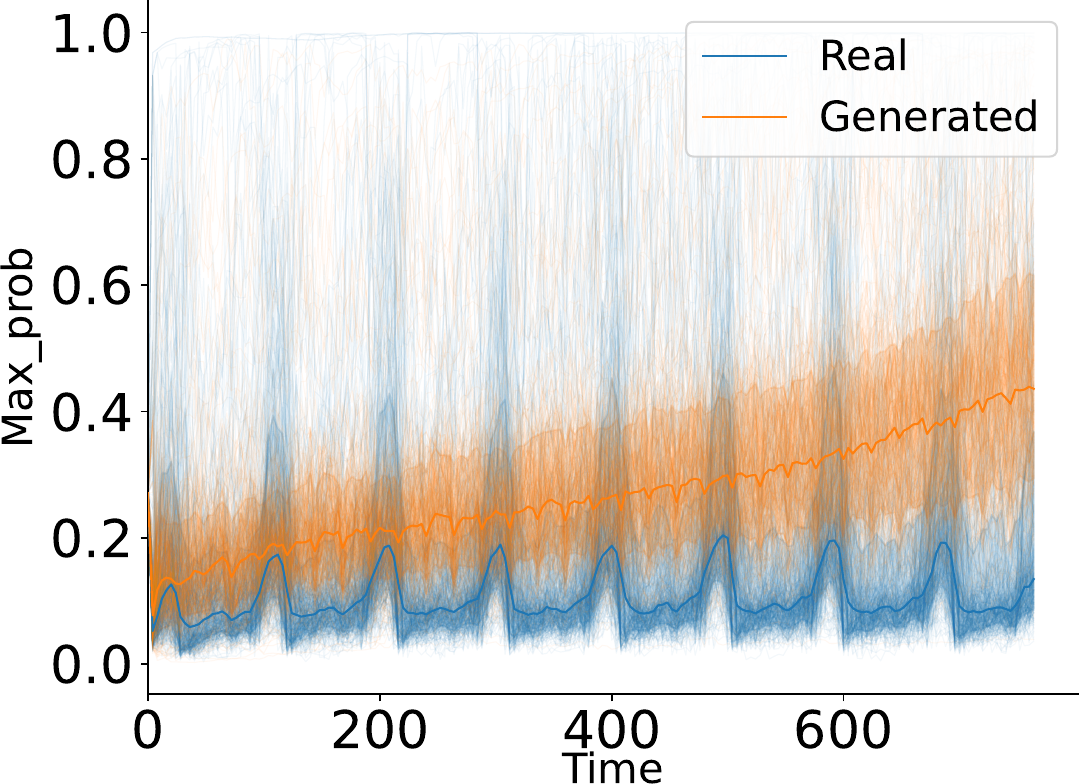}
    }
    \subfloat[Variance]{
    \includegraphics[width=.32\linewidth]{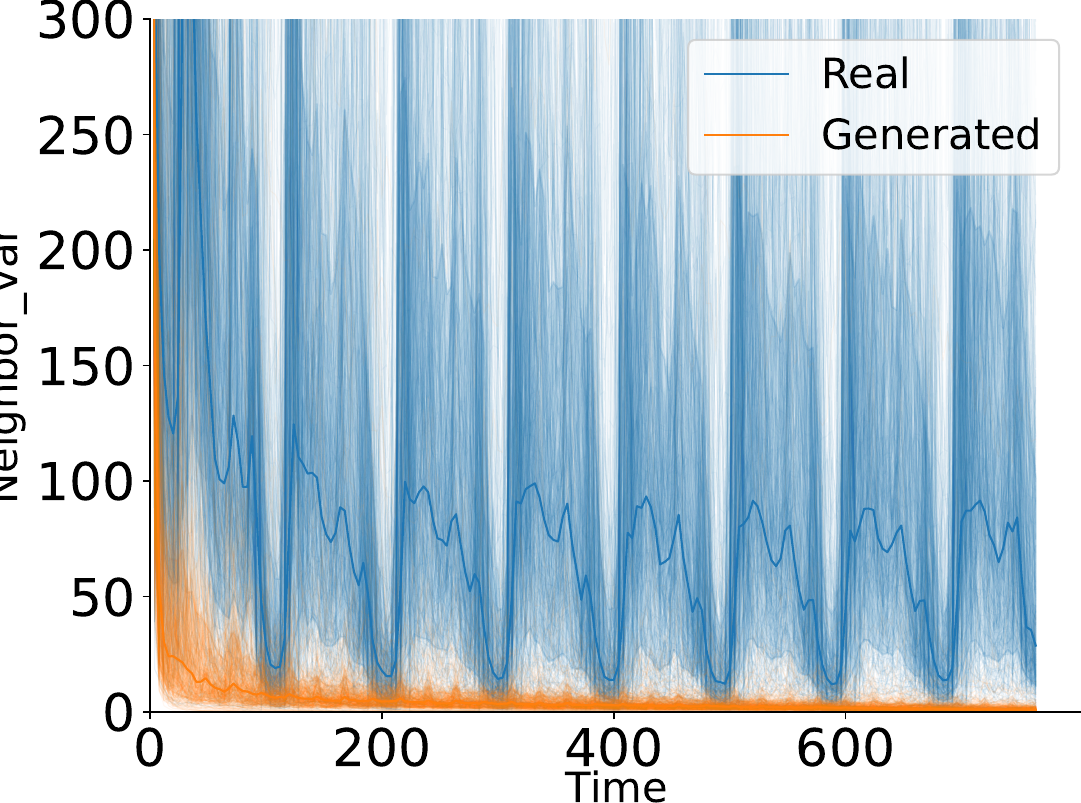}
    }
    \caption{Uncertainty reduction for TimeMoE-generated time series.}
    \label{fig:TimeMoE_uncertainty}
\end{figure*}

\subsection{Intrinsic Dimension}
Intrinsic Dimension method~\cite{Intrinsic_dimension} treats a sequence of tokens as a point cloud on a low-dimensional manifold in the embedding space and estimates its ``persistent homology dimension'' via topological data analysis (TDA).
In particular, it measures the 0-dimensional persistence lifetimes by computing the total edge-length of the point cloud’s minimum spanning tree (MST).

% The life span of the 0-dimensional persistent homology is utilized for intrinsic dimension estimation, which is measured by the total length of its minimal spanning tree (MST).

Concretely, each token is mapped to a high-dimensional vector using a pre-trained Transformer encoder.
In our adaptation of Intrinsic Dimension method to time series, we feed the entire prefix of a token into Chronos-T5 (large) and take its pooled embedding, since token semantics depend on the full context.
We then apply the original parameter settings from~\cite{Intrinsic_dimension}:
\begin{itemize}
    \item \textbf{Weight exponent} $\alpha=1.0$: each MST edge length $L$ contributes as $L^\alpha=L$.
    \item \textbf{Distance metric}: Euclidean distance in $\mathbb{R}^n$ is used to build the MST.
    \item  \textbf{Number of scales} $k=8$: the range of sample sizes is partitioned into eight equally spaced values.
    \item \textbf{Subsampling per scale} $J=7$: for each scale, seven random subsets are drawn and their MST lengths computed, with the median retained.
    \item \textbf{Repeat runs} $n=3$: the slope is calculated three times to obtain the average slope and recover the intrinsic dimension.
\end{itemize}

\subsection{FourierGPT}
FourierGPT~\cite{FourierGPT} generalizes from the uniform information density (UID) hypothesis from psycholinguistic, which posits that human language tends to distribute information evenly across utterance.
Unlike conventional white-box detectors that operate on raw token probabilities, FourierGPT analyze the spectrum of relative token likelihoods to better capture this linguistic property.
This method posits that human-generated text exhibits lower low-frequency power in relative likelihood sequence than model-generated text.
Leveraging this finding, both supervised and unsupervised classifiers are proposed for detection.

In our adaptation of FourierGPT to time series data, we perform the same normalization and frequency‑domain decomposition to the likelihood sequence, and follow the unsupervised FourierGPT's configuration to use the sum of the first ten low‑frequency components’ magnitudes as the classifier.

\subsection{Binocular}
Binocular~\cite{Binocular} proposes that the hidden context or prompt of text segments can introduce additional perplexity into the passage, thereby confusing the probability-based detectors and impairing their performance.
To mitigate the contextual conditioning of text distributions, the Binocular method introduces an observer model.
By measuring the cross-perplexity of texts under the probability distributions of both models, it reconstructs the text's original perplexity independently of contextual bias.

In our adaptation of Binocular, we use Chronos-T5 (tiny) as the observer model and Chronos-T5 (large) as the performer model.

\subsection{Traditional Metrics}
Traditional metrics, including log-probability, rank and log rank~\cite{gltr,automatic,Release_Strategies}, leverage the model's logit distribution to quantify how likely each token is under the model.

% \noindent\textbf{Traditional metrics}~\cite{gltr,automatic,Release_Strategies}. These metrics utilize the model to extract statistic features of the tokens, such as Log Likelihood ($\log p$), Rank and Log Rank.

\section{Supplementary Experiments and Results}\label{apdx:ext_exp}
\subsection{Compute Resources and Hyperparameter Details}\label{apdx:hyperparameter}
In this section we provide the computational environment and the selection of hyperparameters of the UCE method.
All inference experiments were run on a single NVIDIA Tesla V100 GPU (approximately 3000 MiB of dedicated memory) with no additional training required.

Given a fixed time series length $H$, which typically takes the value of $64$ in the main experiment, the value of the starting point $t_1$ is discussed in \cref{ext_exp:ratio}, and typically $t_1=0.75\cdot H$.
We set $\Delta t=1$, and therefore $N=H-t_1$.
For $H=64$, there is $N=16$.
When confining the neighborhood $\mathcal{U}$, we set the center of the neighborhood to be the mean value of the internal probability distribution $\mu=\sum_{x}x\cdot \hat{P}(x)$, and the radius of $\mathcal{U}$ to be $50$ under the size of the vocabulary to be $\left|\mathcal{V}\right|=4096$.
The selection of hyperparameters is highly flexible according to length $H$ and vocabulary size $\left|\mathcal{V}\right|$.

\subsection{Runtime Efficiency}
We compare the runtime of different detection methods by measuring their average processing time per instance, as shown in \cref{fig:runtime}.
Among white-box methods, UCE achieves the second fastest runtime ($0.47\ \mathrm{s}$), significantly outperforming the baseline approaches such as DetectGPT (23.18 s), DetectLLM-NPR (22.93 s) and DNA-GPT (16.30 s).
These methods require multiple resampling or perturbations, which introduce additional computation.
Traditional statistical methods such as Log Likelihood (8.32 s), FourierGPT (8.35 s) and DetectLLM-LLR (8.32 s) that utilize Log Rank show moderate latency.
These methods infer the logit distribution for every token for a single sequence, while Fast-DetectGPT (9.32 s) only infers the logit distribution over subsets of tokens across multiple perturbed sequences.
Binocular requires simultaneous inference of the observer model and the performer model. With both models loaded in parallel (which is the optimal scenario), the runtime of this method (9.98 s) is determined by the slower of the two.
Although both categories achieve substantial speedups, they lag behind UCE, which only computes a small portion (around 1/8) of token logits in a single sequence.
Noting that the Intrinsic Dimension method requires a relatively small proportion of embeddings to estimate the persistent homology dimension, which leads to the fastest runtime (0.33 s). However, for stability in performance, multiple reruns are often required (typically 3-10 times).
This significantly increases the effective runtime to 1-3 seconds, reducing its nominal advantage.
Therefore, even after adjusting for practical usage scenarios, UCE retains a clear advantage in both efficiency and real-world applicability.

\begin{figure*}[t!]
    \centering
    \subfloat[Entropy\label{fig:hist_entropy}]{%
      \includegraphics[width=.32\linewidth]{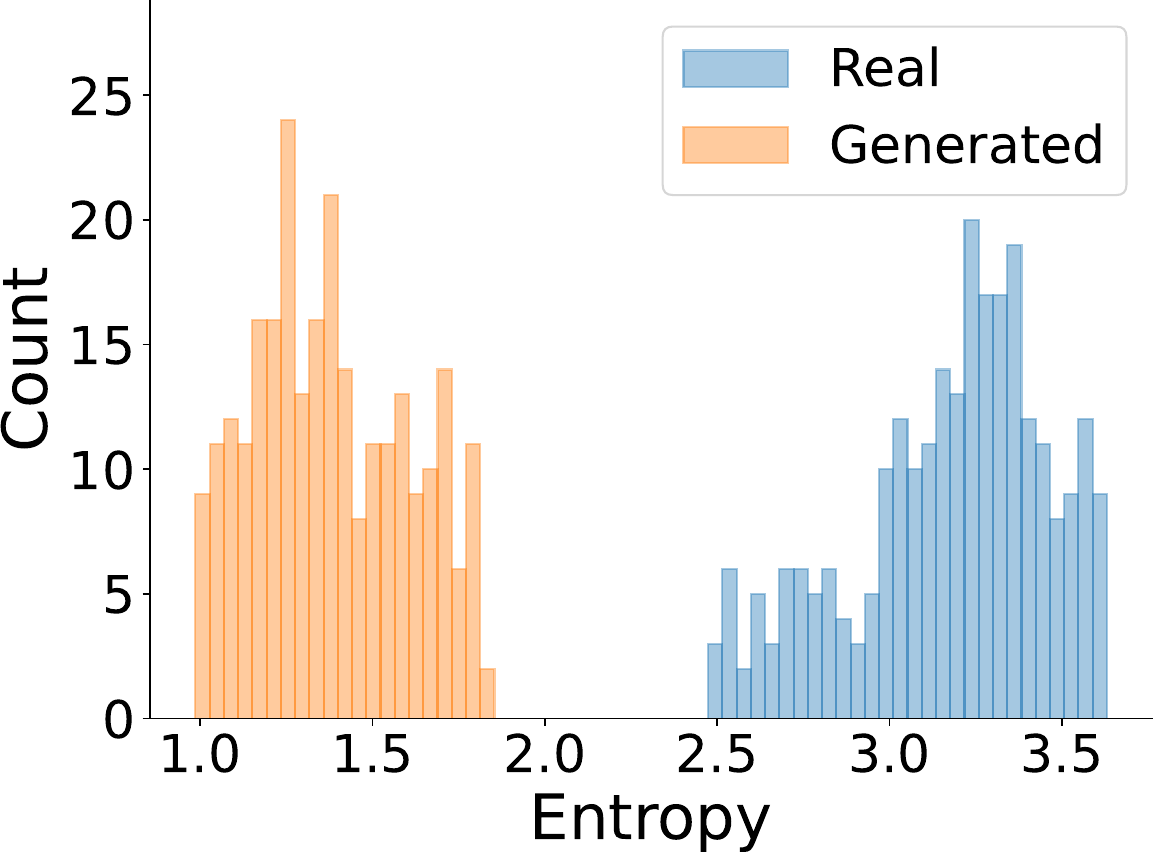}}
    \subfloat[Max-probability\label{fig:hist_max_prob}]{%
      \includegraphics[width=.32\linewidth]{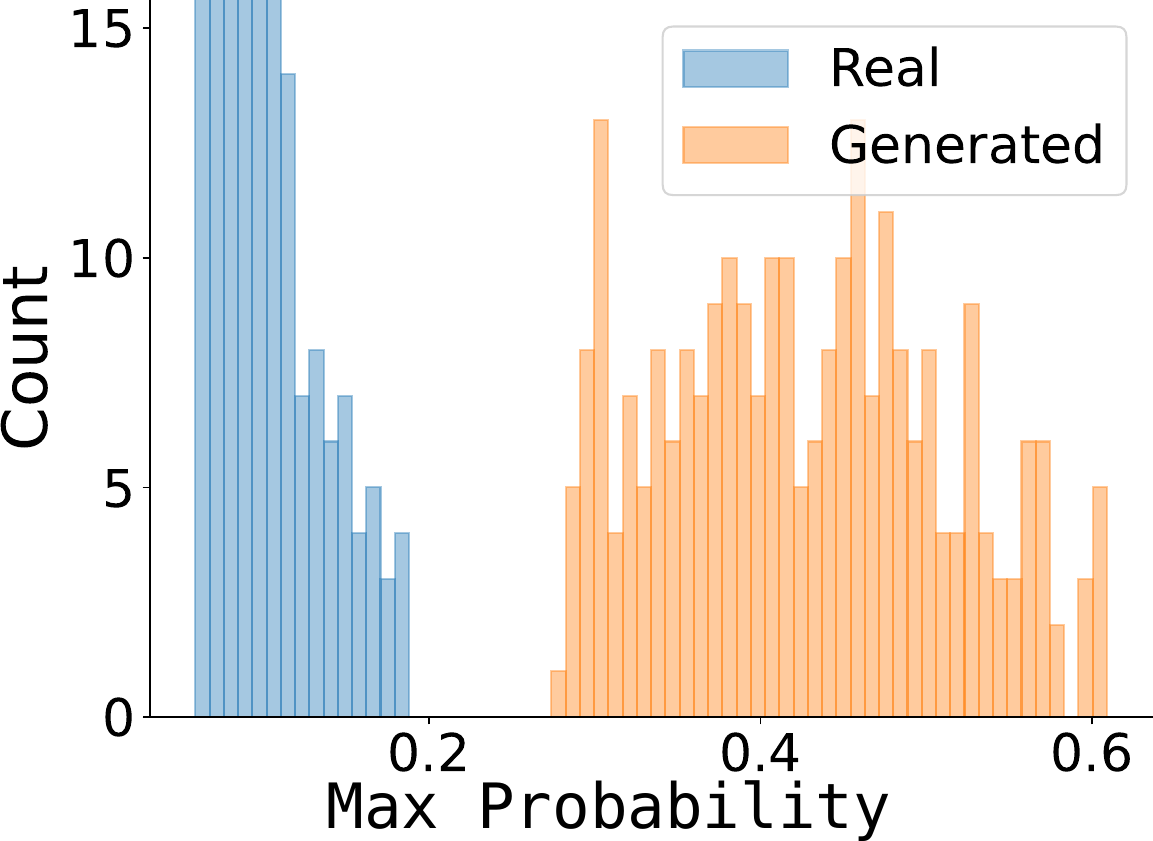}}
    \subfloat[Variance\label{fig:hist_var}]{%
      \includegraphics[width=.32\linewidth]{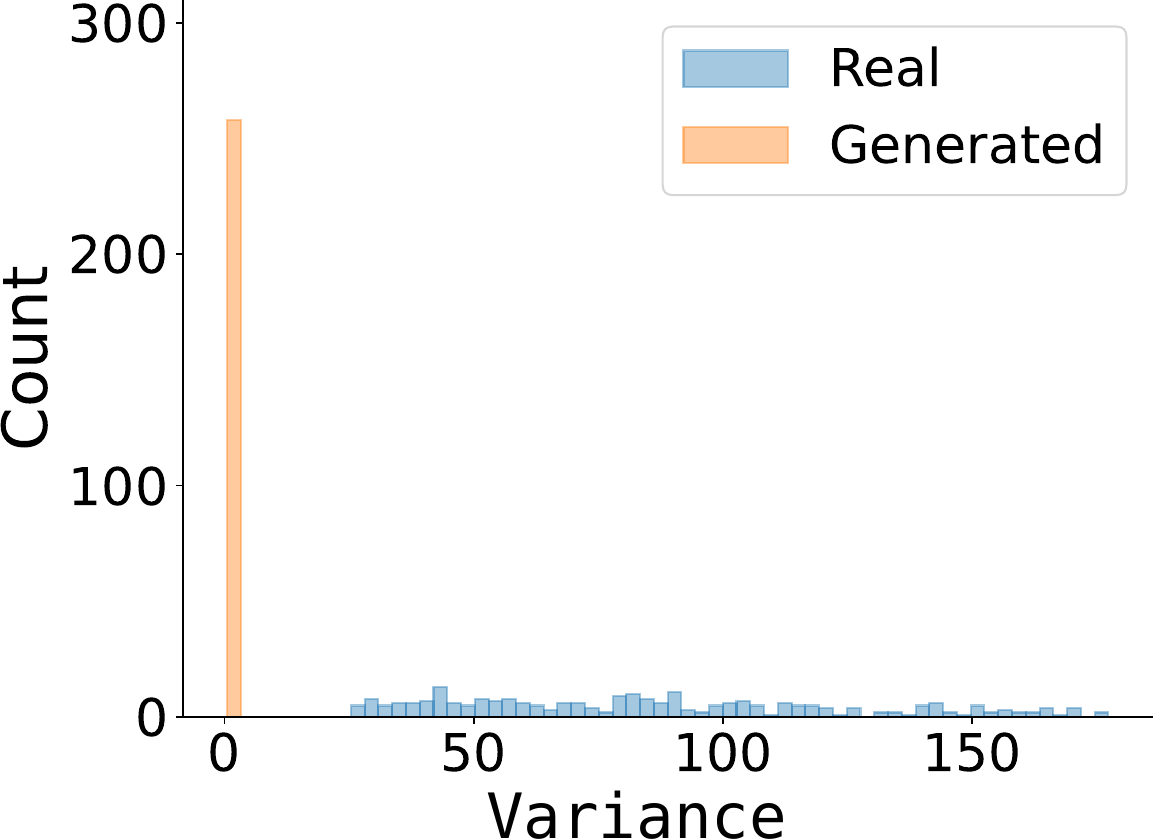}}
    \caption{Uncertainty discrepancy for {real} (blue) and
      Timer-generated (orange) time series}
      \label{fig:hist}
\end{figure*}
\begin{figure*}[t!]
  \centering
  \subfloat[Entropy\label{fig:hist_entropy_TimeMoE}]{%
    \includegraphics[width=.32\linewidth]{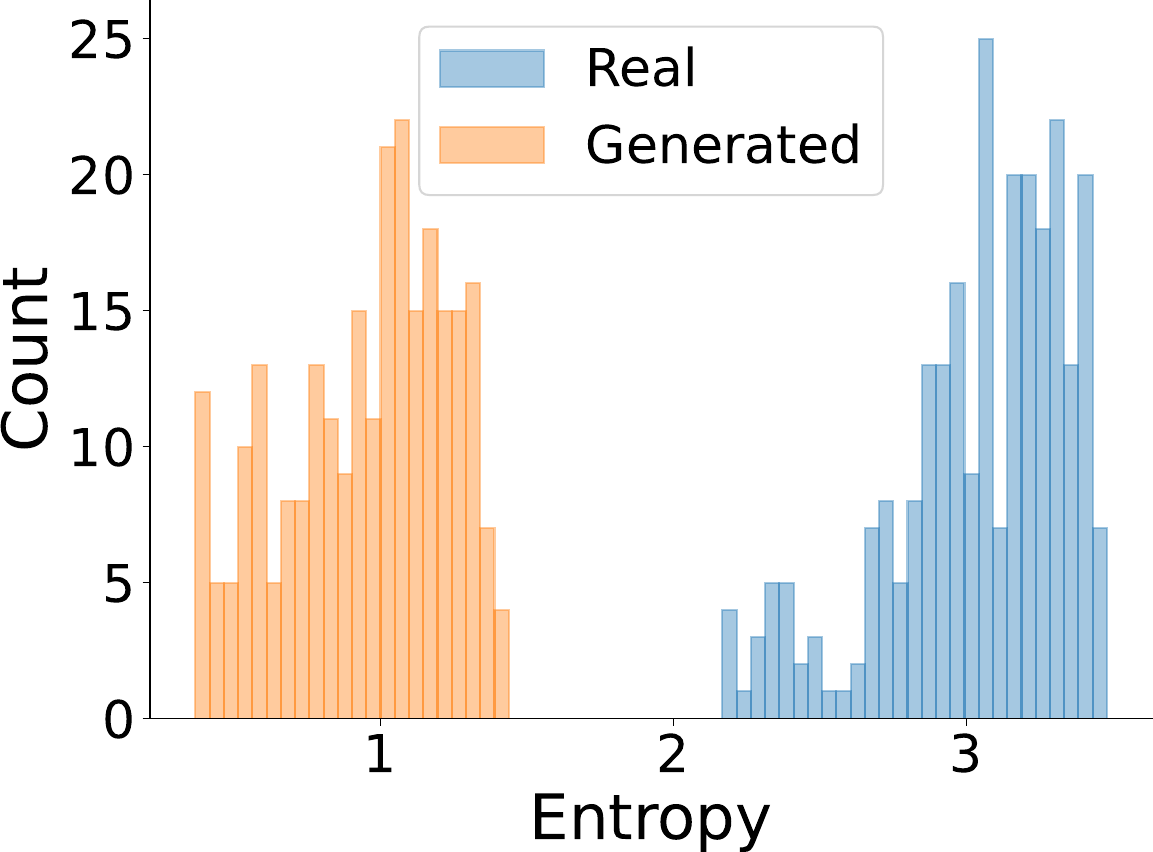}}
  \subfloat[Max-probability\label{fig:hist_max_prob_TimeMoE}]{%
    \includegraphics[width=.32\linewidth]{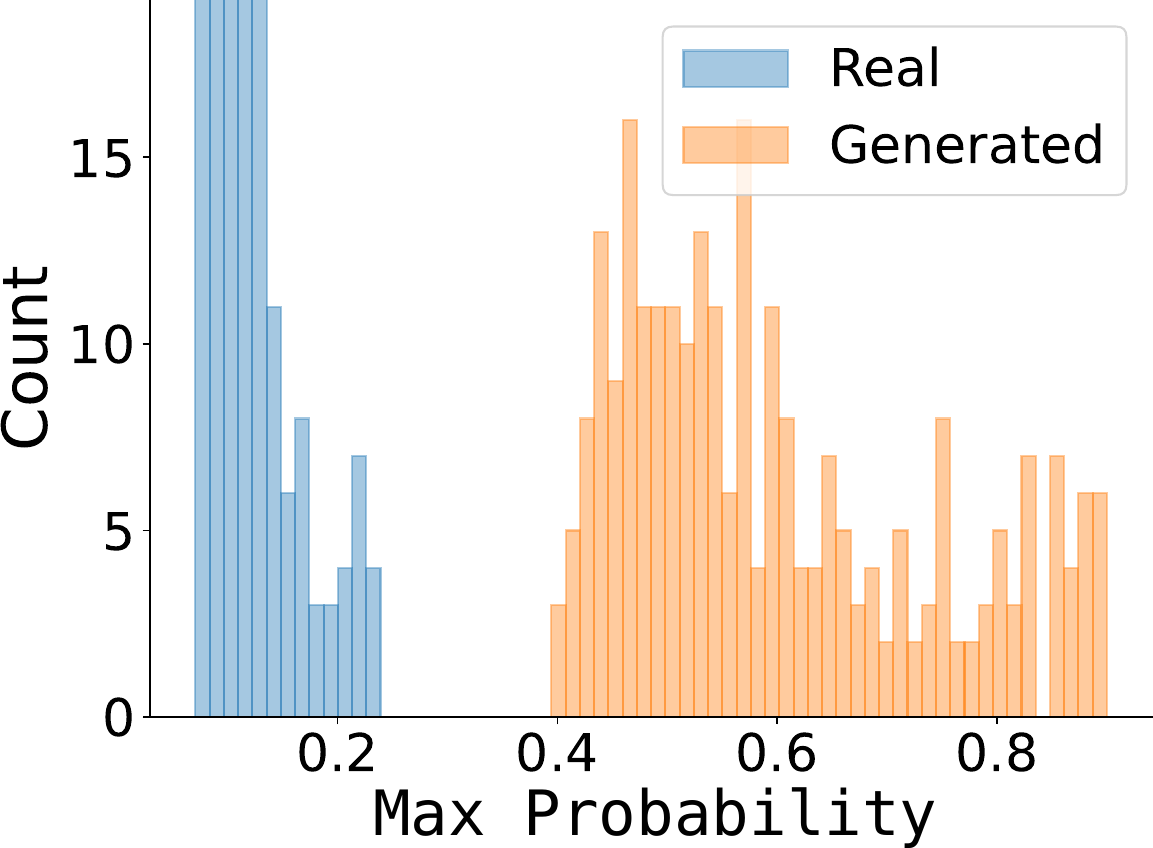}}
  \subfloat[Variance\label{fig:hist_var_TimeMoE}]{%
    \includegraphics[width=.32\linewidth]{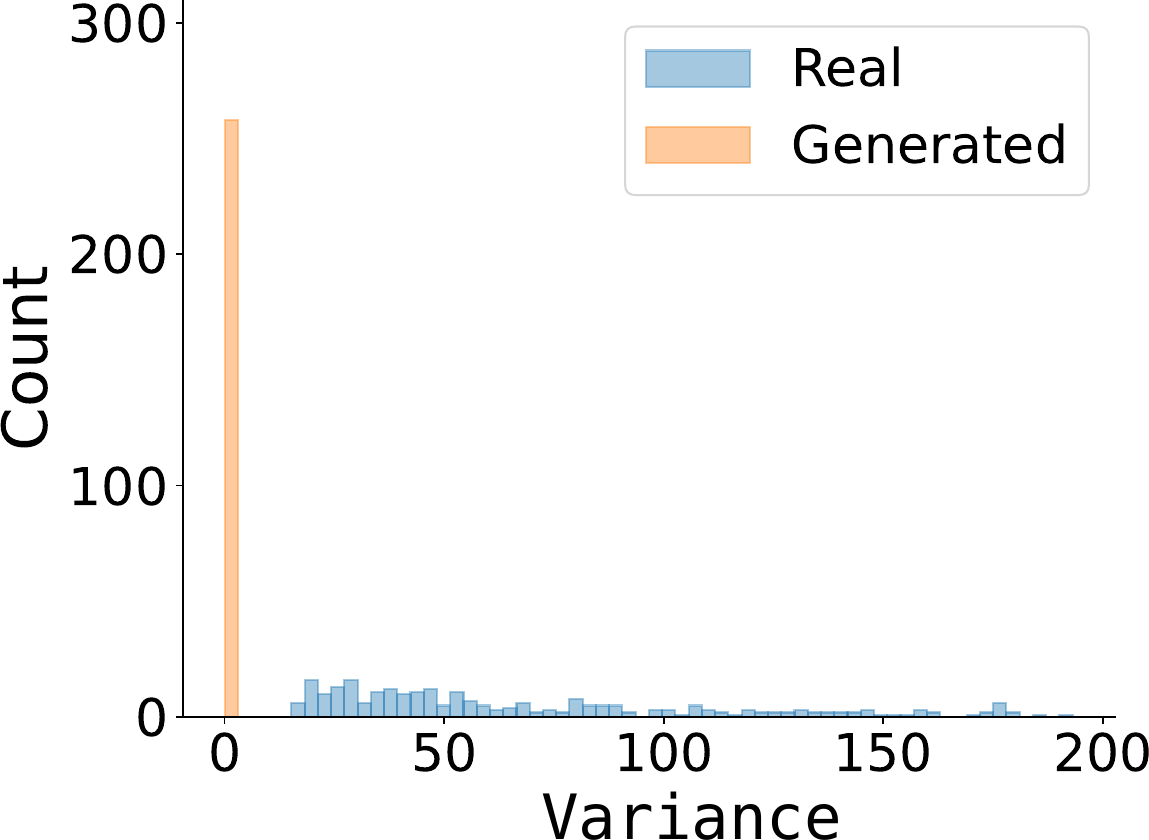}}
  \caption{Uncertainty discrepancy for {real} (blue) and
    {TimeMoE-generated} (orange) time series.}
  \label{fig:hist_TimeMoE}
\end{figure*}

\begin{figure}[t!]
    \centering
    \includegraphics[width=.9\linewidth]{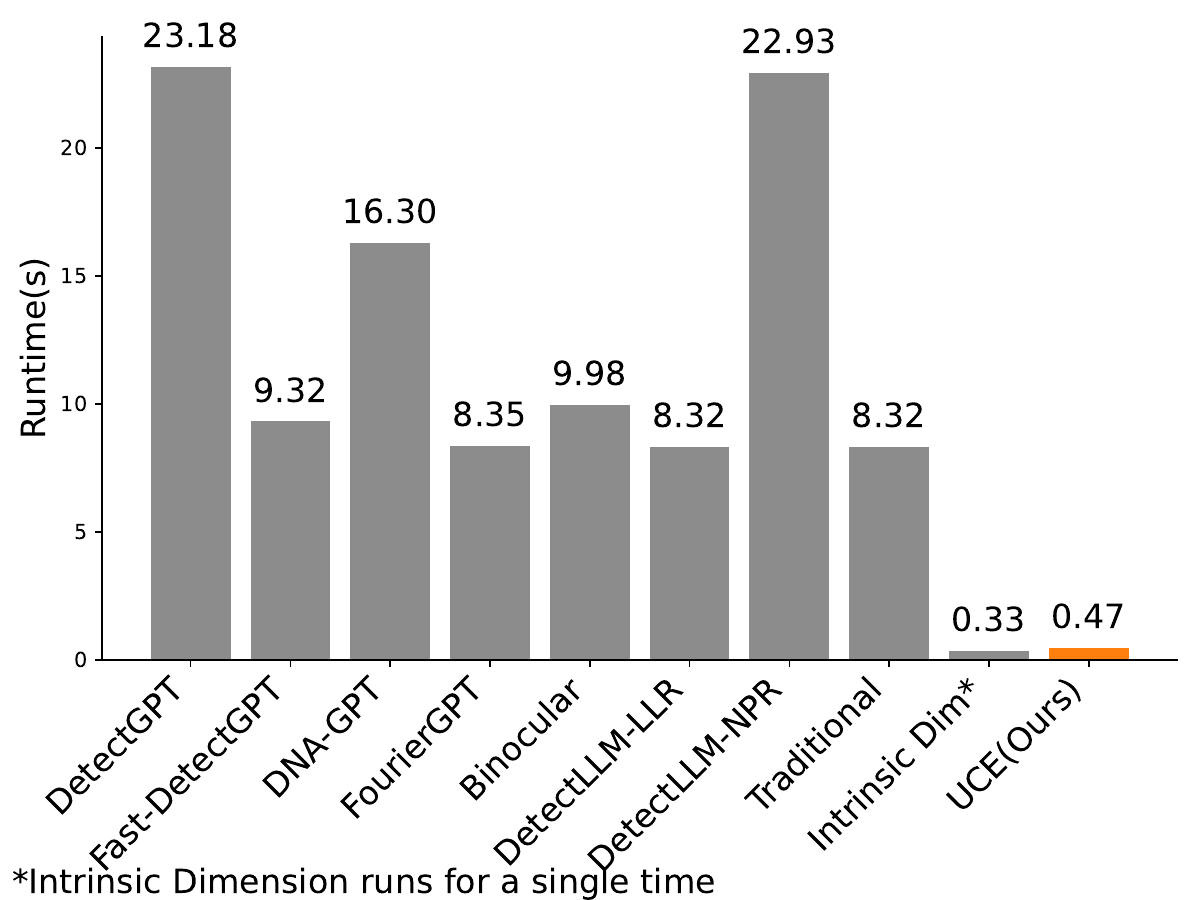}
    \caption{Runtime (in seconds) of different detection methods on a single input sequence. Although Intrinsic Dimension is shown with the shortest runtime, it typically requires multiple repeated runs to achieve stable results in practice. Overall, UCE is far more efficient, particularly when compared with white-box approaches.}
    \label{fig:runtime}
\end{figure}

\subsection{Detection with Shorter Time Series Length}
\label{ext_exp:short length}
To test the robustness of UCE on relatively shorter input sequences, we conduct detection experiments on sequences with an average length of $H=24$.
According to the Contraction Hypothesis, shorter time series exhibit a smaller overall reduction in uncertainty, thereby degrading detection performance.
This phenomenon aligns with previous findings in text generation detection that shorter texts are less discriminative.

As shown in \cref{tab:short length}, all methods experience a general drop in performance under this setting.
Nevertheless, UCE maintains a consistent advantage in TPR for in-distribution datasets, while achieving comparable AUROC scores to the baselines.
On zero-shot datasets, UCE sustains a slight but stable lead in both AUROC and TPR.
These results demonstrate that UCE remains effective and relatively stable across varying sequence lengths, even under constrained temporal contexts.

\begin{table}[t]
\small
\centering

\resizebox{\linewidth}{!}{
\begin{tabular}{@{}lcccc@{}}
\toprule
 \multicolumn{1}{l}{\multirow{2}{*}{\makecell{\\ \textbf{Dataset}}}}
&  \multicolumn{2}{c}{In-Domain Datasets Avg. }& \multicolumn{2}{c}{Zero-Shot Datasets Avg.}
 \\
  \cmidrule(lr){2-3} \cmidrule(lr){4-5}

   & \makecell{AUROC} & \makecell{TPR}  & \makecell{AUROC} & \makecell{TPR}\\

%  \midrule
%  % \hdashline
% Prediction Length $H$ &64\\
% MASE &0.663\\
\midrule
$\log p(x)$
                            &0.730   &0.101
                            &{0.640}   &{0.171}
\\
Rank
                            &0.734   &0.115
                            &0.615   &0.147
\\
LogRank
                            &\underline{\textbf{0.767}}   &0.119
                            &{0.644}   &0.171
\\
% Entropy
%                             &0.63   &0.05
%                             &0.49   &0.06
% \\
DetectGPT
                            &0.620   &0.032
                            &0.631   &0.105
\\
Fast-DetectGPT
                            &0.690   &0.047
                            &0.611   &0.091
\\
DetectLLM-LLR
                            &\underline{0.760}   &{0.213}
                            &{0.645}   &0.159
\\
DetectLLM-NPR
                            &0.685   &0.065
                            &0.626   &0.121
\\
DNA-GPT
                            &0.476   &0.065
                            &0.462   &0.069
\\
FourierGPT                  &0.528  &0.030
                            &0.505  &0.013

\\
Binocular                   &0.523  &0.035
                            &0.597  &0.037

\\
Intrinsic Dimension
                            &0.600   &0.105
                            &0.536   &0.041

\\
\midrule
\OurMethod -Entropy
                            &0.751   &\underline{0.264}
                            &\underline{0.692}   &\underline{0.199}
\\

% \hdashline
\quad -Max Prob
                            &{0.749}   &\underline{\textbf{0.277}}
                            &\underline{\textbf{0.703}}   &\underline{\textbf{0.211}}
\\
% \OurMethod (max Prob slope)\\

% \OurMethod (Entropy slope)\\
\quad -Variance
                            &0.738   &0.116
                            &0.614   &0.115
\\

\bottomrule
\end{tabular}
}
\caption{AUROC and TPR (at 1\% FPR) for detecting samples with shorter lengths.
The \textbf{\underline{best}} and \underline{second-best} results are highlighted.
% Methods achieving the first (bold underline) and second (underline) best score have been highlighted.
There are declines in performance as the length decreases.}
\label{tab:short length}
\end{table}

% In-domain datasets and 0-shot datasets are split in 2 tables

\subsection{Starting Point Analysis}\label{ext_exp:ratio}
The UCE methods evaluate the variation of uncertainty in time series with unknown history.
This naturally give rise to two important properties: (1) early prediction distributions tend to be unstable due to insufficient history;
(2) the contraction of uncertainty is increasingly pronounced at later time points.
Motivated by  these observations, we adjust the evaluation window by excluding the initial prefix and delaying the starting point of prediction. Specifically, we set the forecasting start time to $t_0 = \alpha \cdot T$, and evaluate several values of $\alpha \in \{0, 0.25, 0.5, 0.75, 0.99\}$ to study its impact on detection performance.
Note that setting $\alpha = 0.99$ corresponds to using only the final token to obtain the predictive distribution to calculate UCE metrics.
The results, summarized in \cref{fig:ratio}, show that as $\alpha$ increases, AUROC remains stable or improves slightly while TPR shows modest fluctuation.
\begin{figure}[t]
  \centering
  \includegraphics[width=0.85\linewidth]{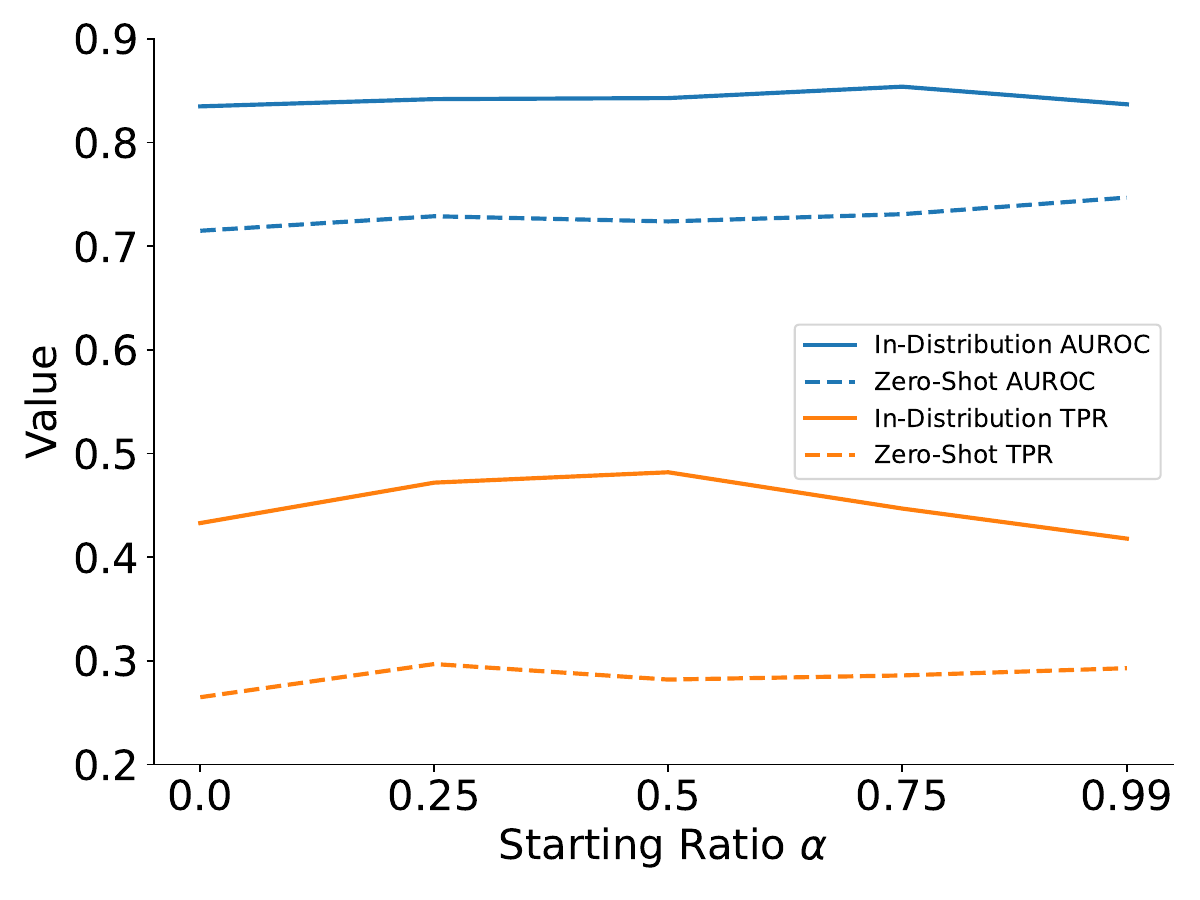}
  \caption{Effect of starting points $t_0=\alpha\cdot T$ on AUROC and TPR in both dataset categories.}
    \label{fig:ratio}
\end{figure}

Overall, detection performance is stable across a wide range of $\alpha$ values, suggesting that UCE is relatively robust to the choice of starting position.  We therefore set $\alpha = 0.75$ in our main experiments to balance AUROC and TPR performance.

\begin{figure*}[t!]
  \centering
  \subfloat[Entropy (Gaussian)\label{fig:normal_E}]{%
    \includegraphics[width=.30\textwidth]{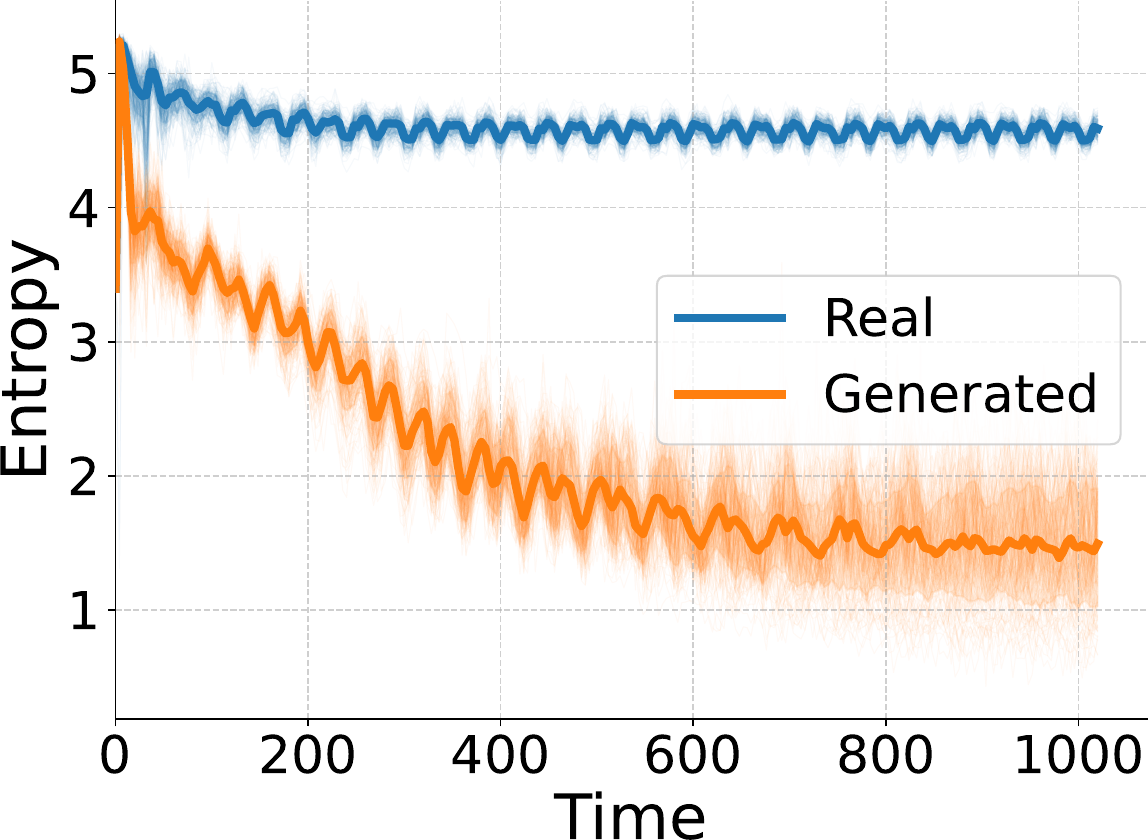}}
    \subfloat[Max-Probability (Gaussian)\label{fig:normal_M}]{%
    \includegraphics[width=.30\textwidth]{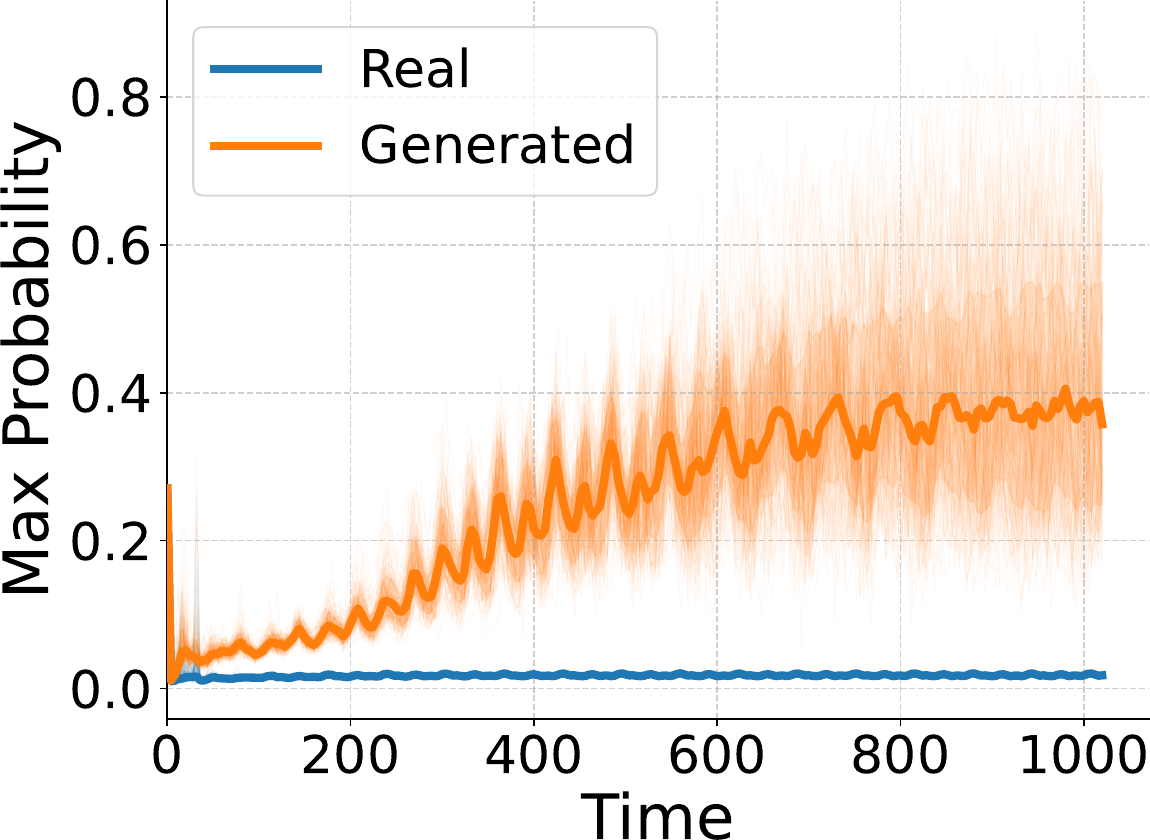}}
    \subfloat[Variance (Gaussian)\label{fig:normal_V}]{%
    \includegraphics[width=.30\textwidth]{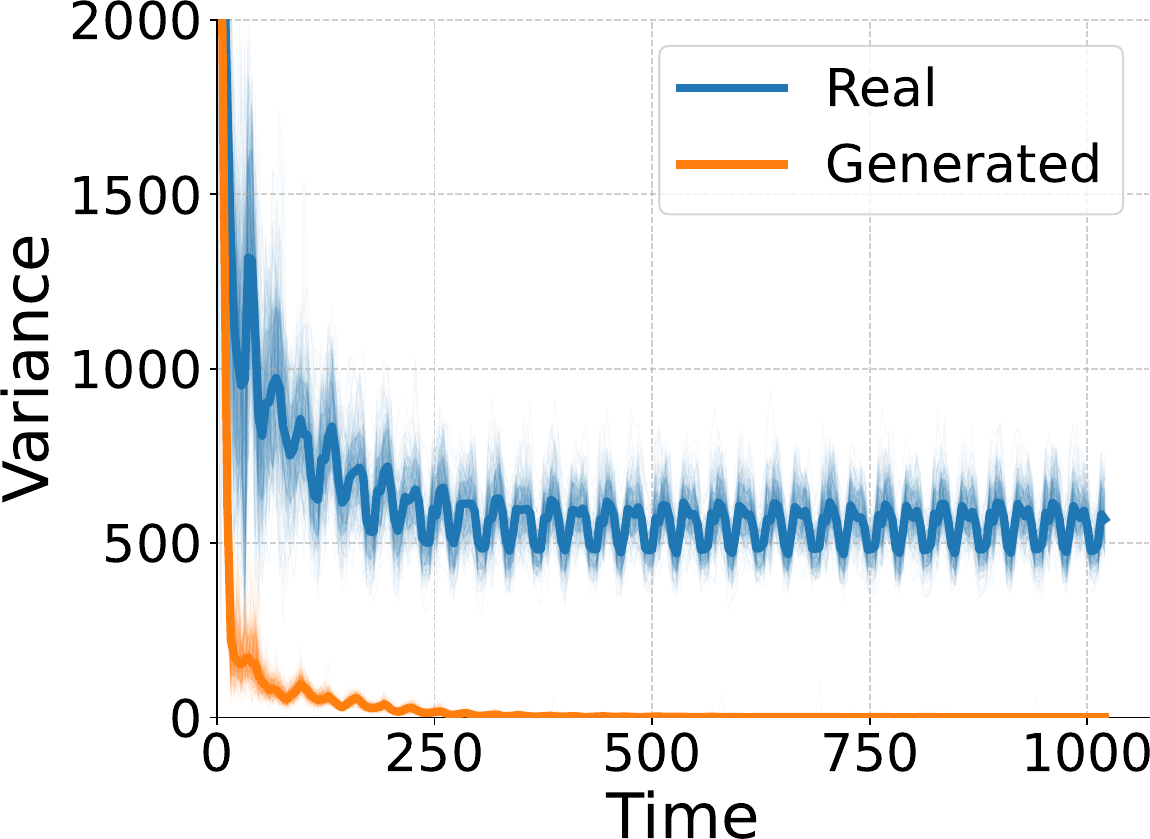}}
    \\
  \subfloat[Entropy (Laplacian)\label{fig:laplace_E}]{%
    \includegraphics[width=.30\textwidth]{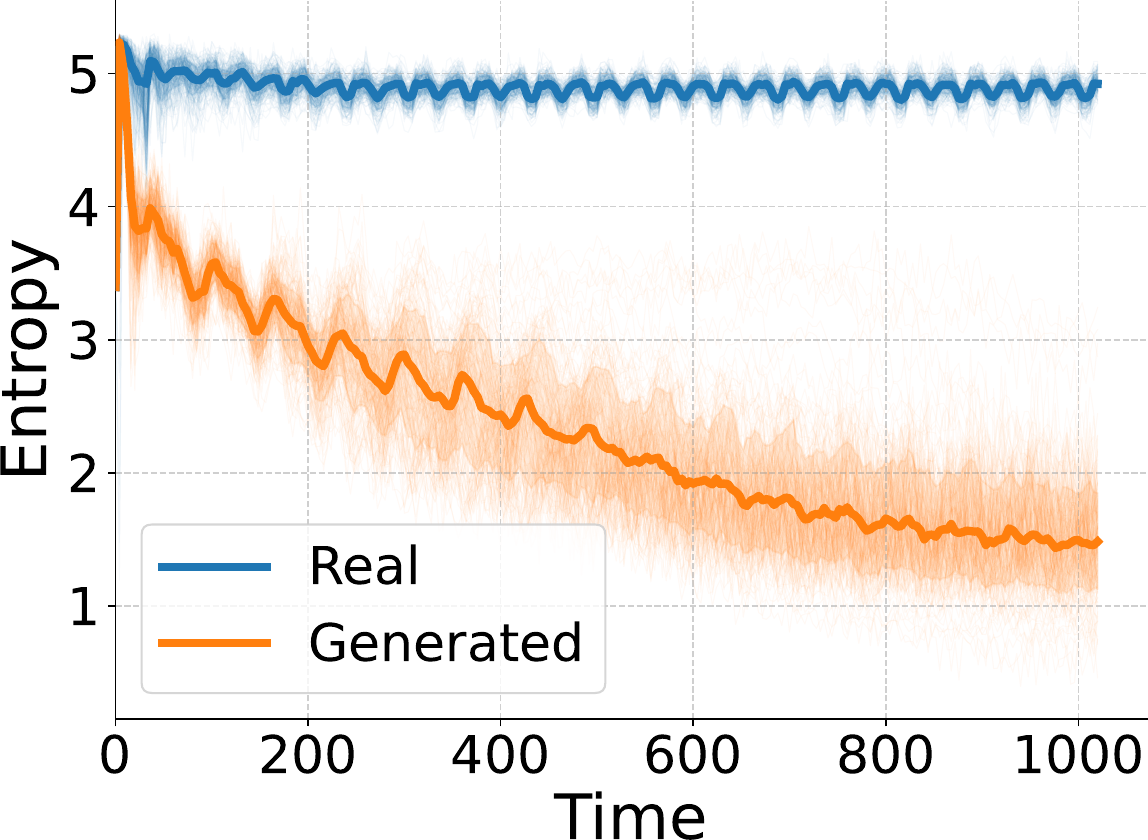}}
    \subfloat[Max-Probability (Laplacian)\label{fig:laplace_M}]{%
    \includegraphics[width=.30\textwidth]{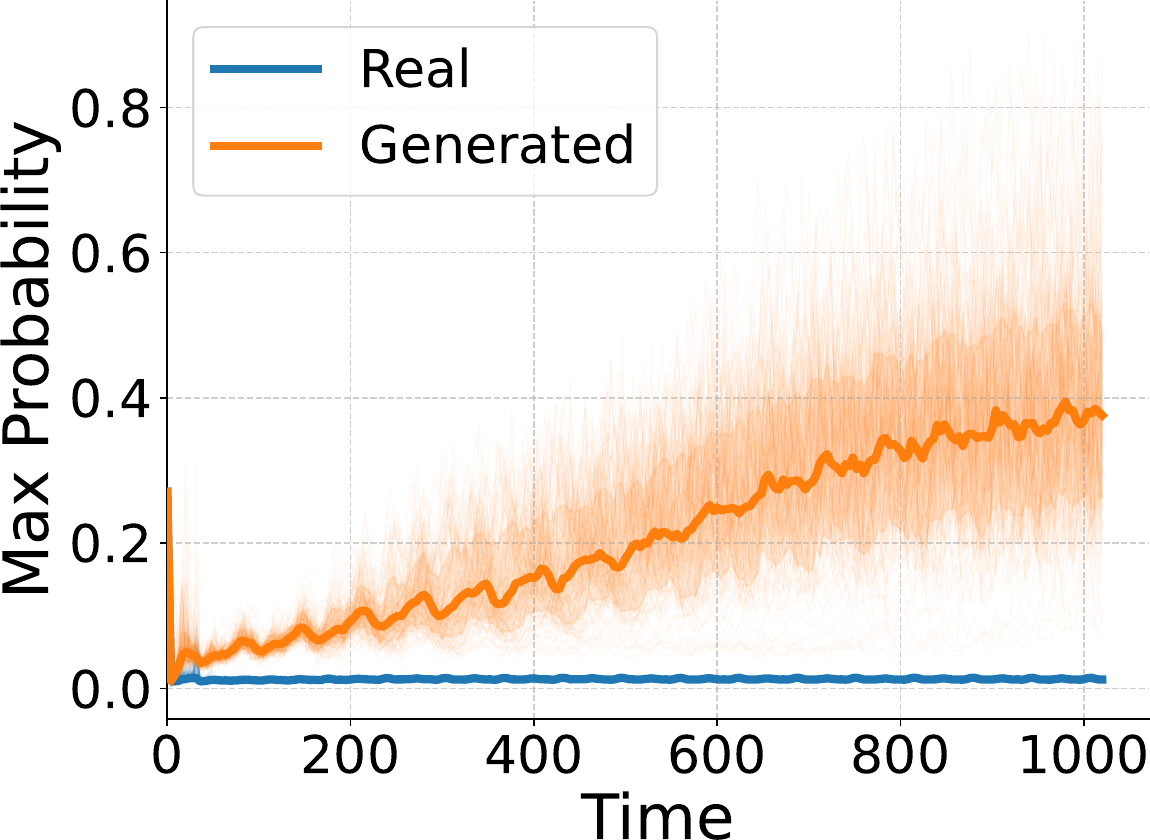}}
    \subfloat[Variance (Laplacian)\label{fig:laplace_V}]{%
    \includegraphics[width=.30\textwidth]{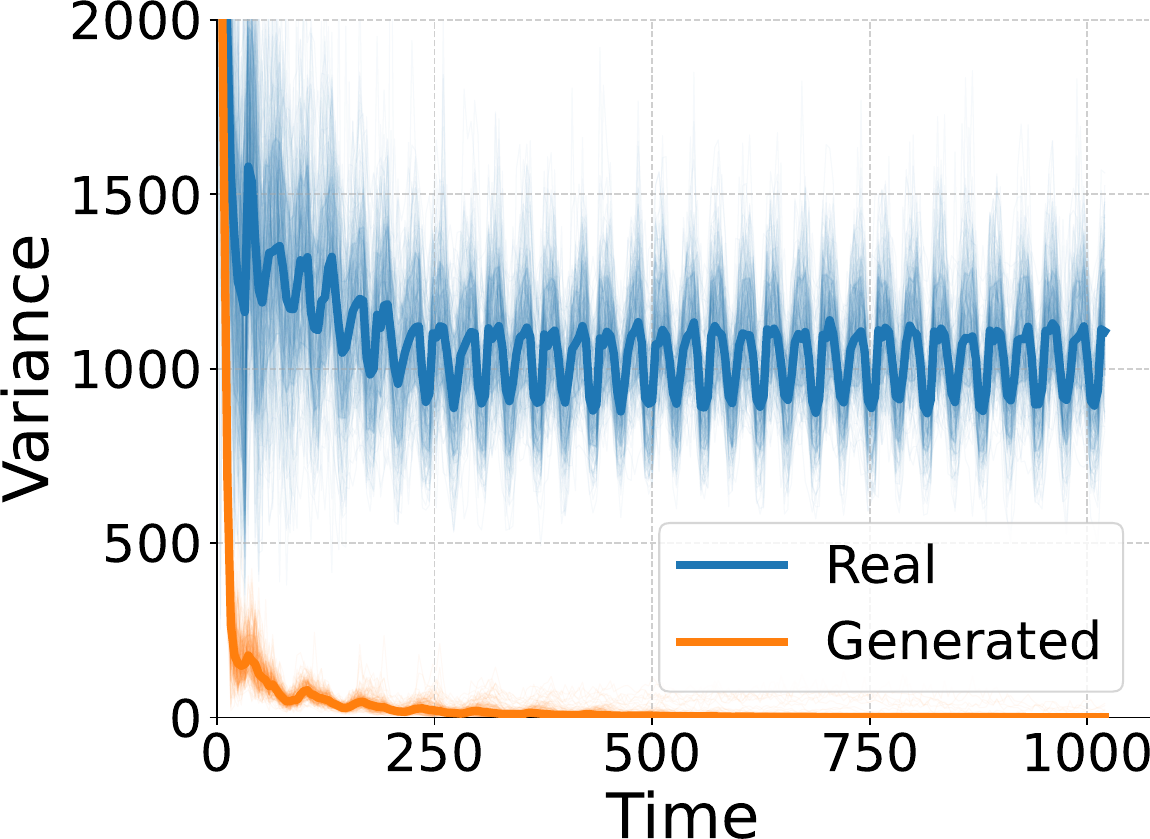}}
    \\
  % \subfloat[Entropy (Triangular)\label{fig:triangular_E}]{%
  %   \includegraphics[width=.30\textwidth]{pics/Triangular_entropy_combined.png}}
  %   \subfloat[Max-Probability (Triangular)\label{fig:triangular_M}]{%
  %   \includegraphics[width=.30\textwidth]{pics/Triangular_max_prob_combined.png}}
  %   \subfloat[Variance (Triangular)\label{fig:triangular_V}]{%
  %   \includegraphics[width=.30\textwidth]{pics/Triangular_var_combined.png}}
  %   \\
    \subfloat[Entropy (Cauchy)\label{fig:cauchy_E}]{%
    \includegraphics[width=.30\textwidth]{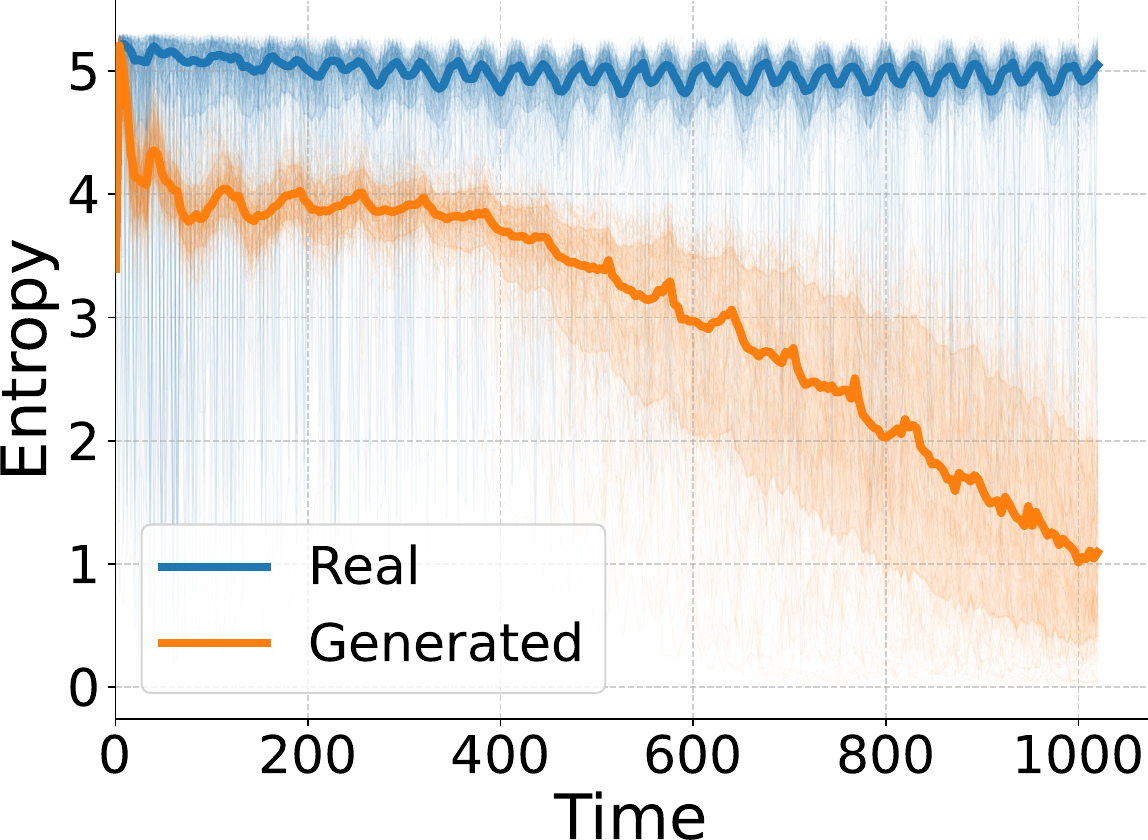}}
    \subfloat[Max-Probability (Cauchy)\label{fig:cauchy_M}]{%
    \includegraphics[width=.30\textwidth]{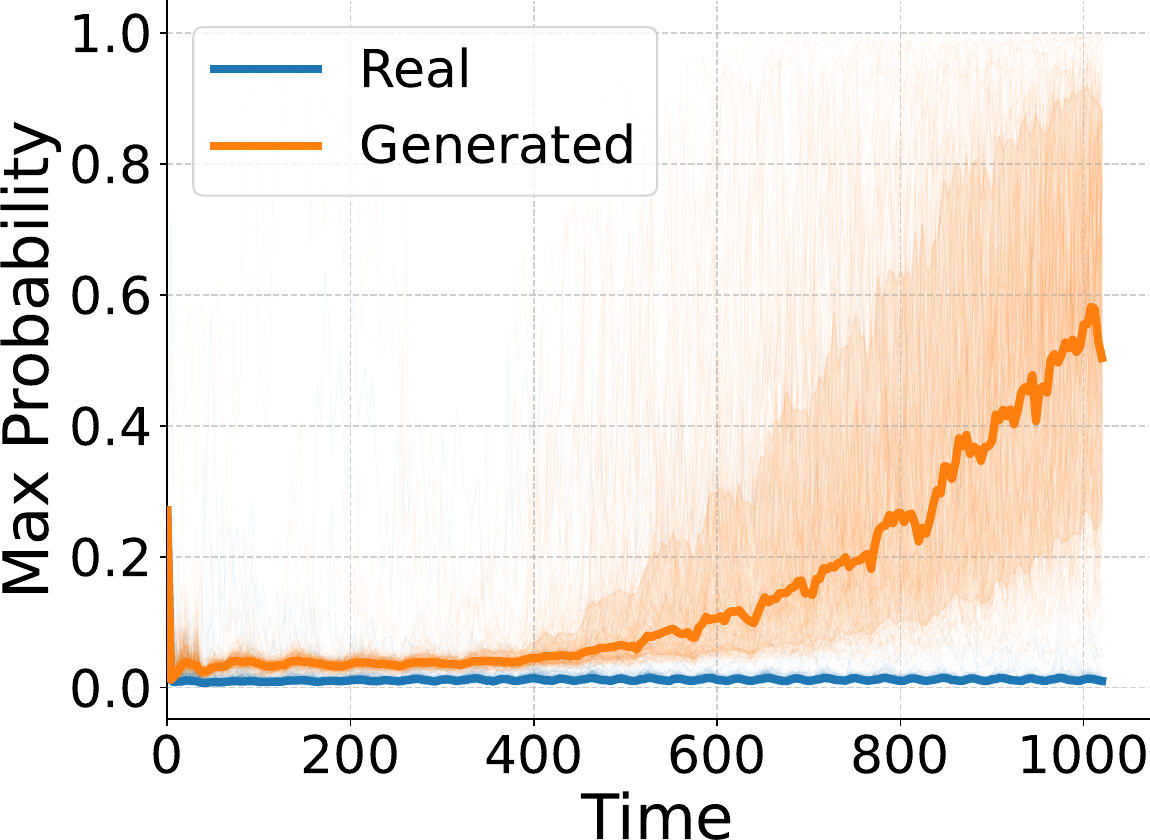}}
    \subfloat[Variance (Cauchy)\label{fig:cauchy_V}]{%
    \includegraphics[width=.30\textwidth]{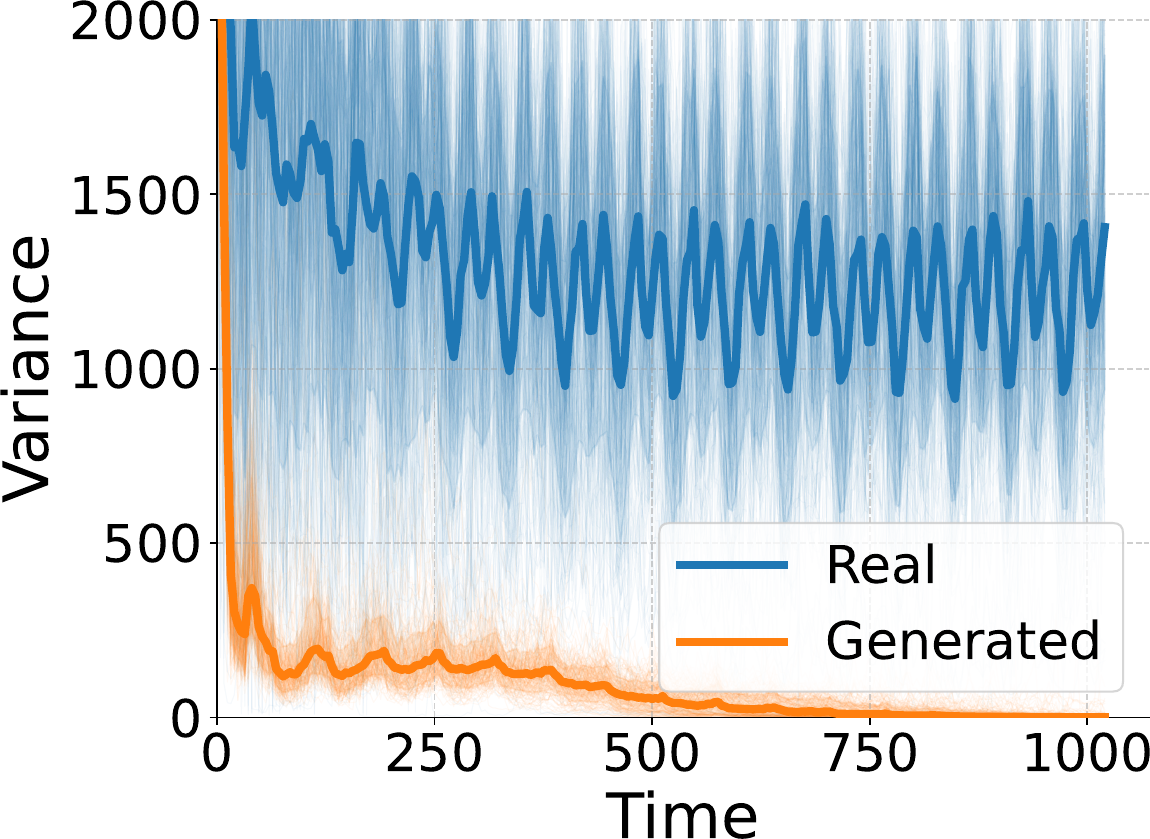}}
    \\
    \subfloat[Entropy (Student's T)\label{fig:student_E}]{%
    \includegraphics[width=.30\textwidth,height=0.25\textheight,keepaspectratio]{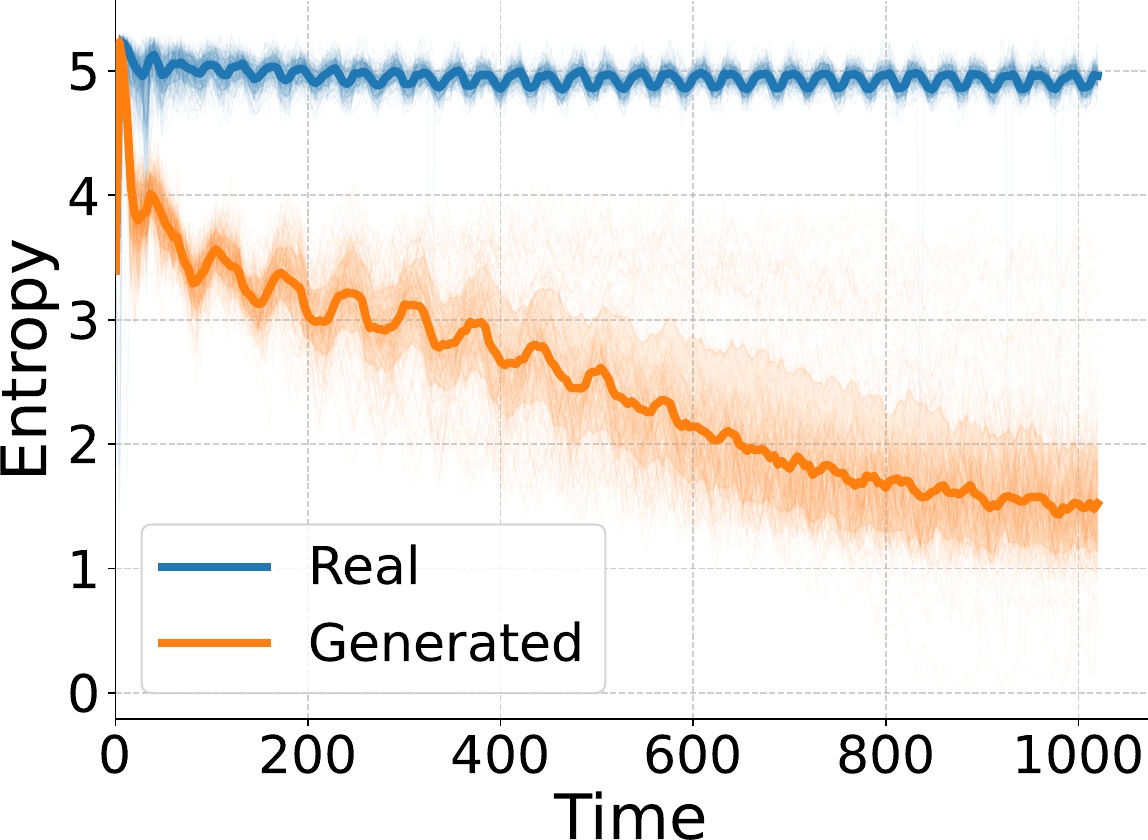}}
    % \subfloat[Entropy (Student's T)\label{fig:student_E}]{%
    % \includegraphics[width=.30\textwidth]{pics/Student_entropy_combined.png}}
    \subfloat[Max-Probability (Student's T)\label{fig:student_M}]{%
    \includegraphics[width=.30\textwidth]{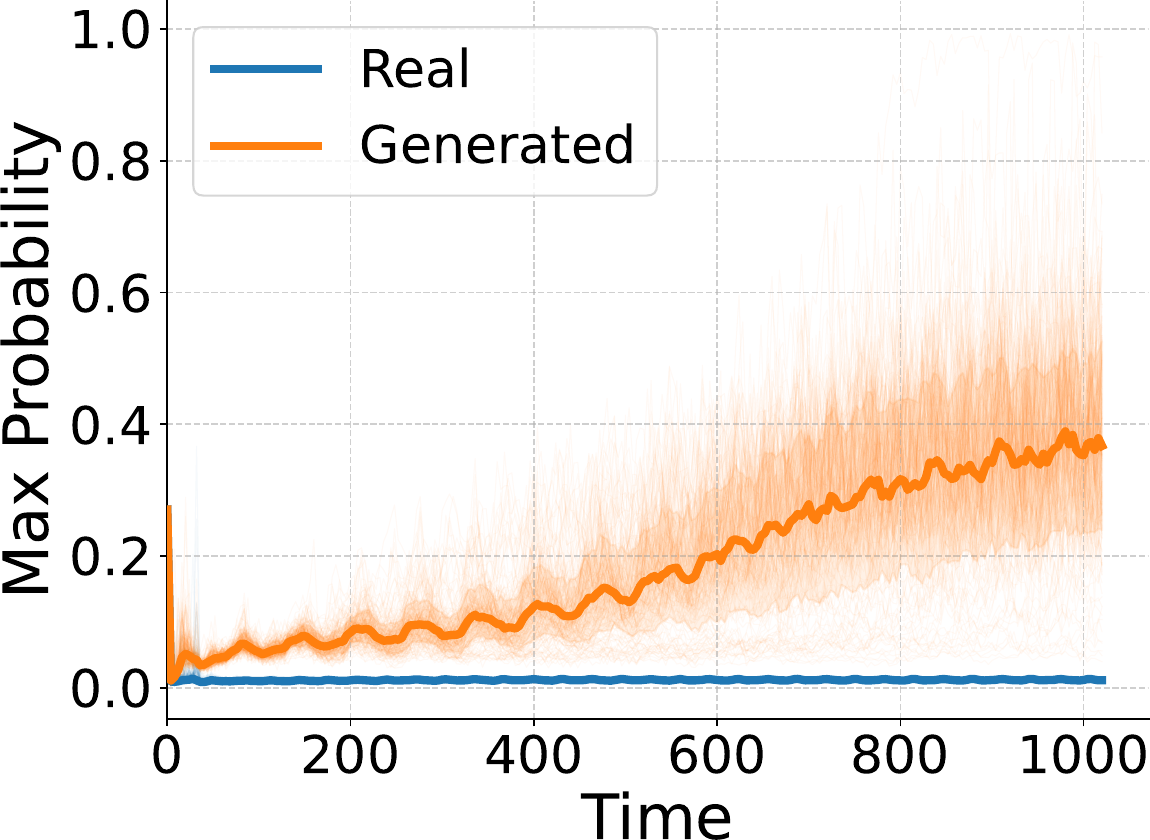}}
    \subfloat[Variance (Student's T)\label{fig:student_V}]{%
    \includegraphics[width=.30\textwidth]{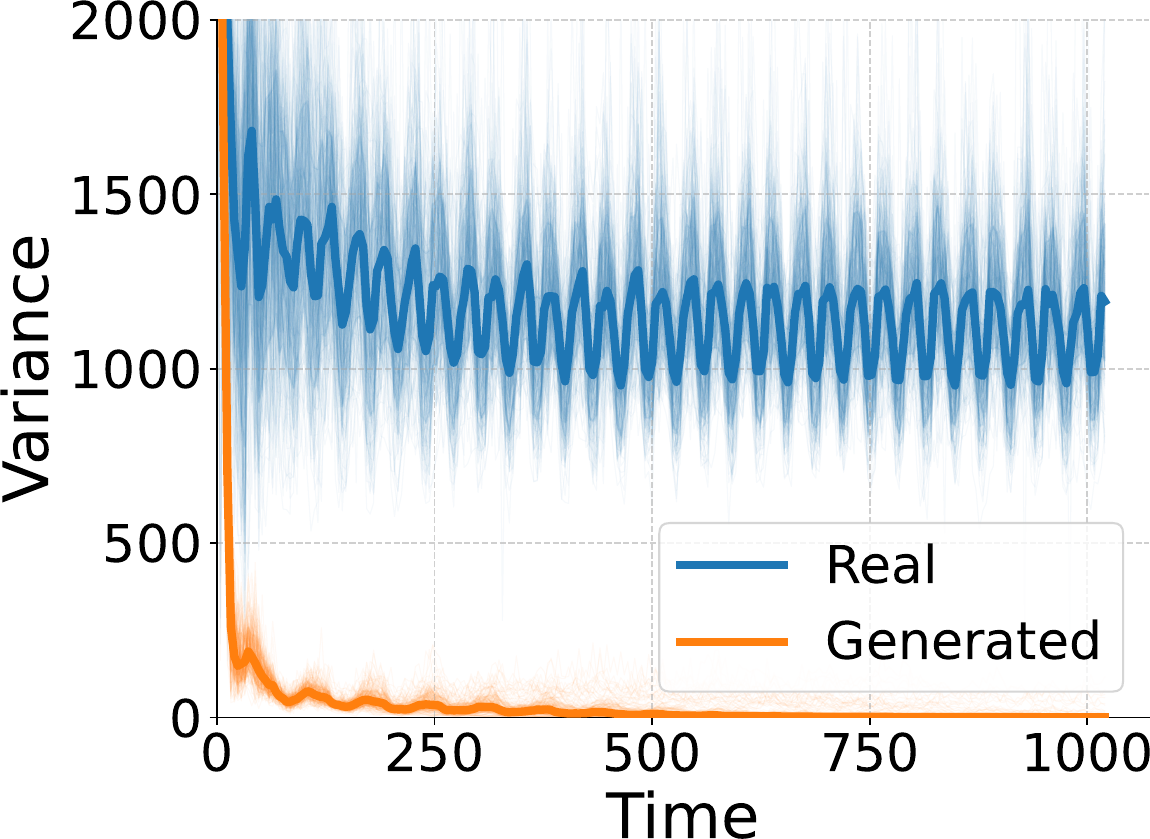}}

  \caption{Difference of entropy between {real} and
    {model-generated} data for various types of noise distributions on sine waves.}
    \label{fig:noise_type}
\end{figure*}

% In-domain datasets and 0-shot datasets are split in 2 tables

\begin{table*}[t!]
% \small
\centering

\tiny                                % 最小字体
\setlength\tabcolsep{2pt}            % 最小左右间距
\renewcommand{\arraystretch}{0.65}   % 紧凑行高
\captionsetup{skip=3pt}              % 减少 caption 与表格间距

\resizebox{\linewidth}{!}{
\begin{tabular}{@{}lcccccccccccc@{}}
\toprule
\multicolumn{1}{l}{\multirow{2}{*}{\makecell{\\ \textbf{Dataset}}}}
 &  \multicolumn{2}{c}{{Electricity (15M)}}&  \multicolumn{2}{c}{{Electricity (H)}}&  \multicolumn{2}{c}{{Electricity (W)}}&  \multicolumn{2}{c}{{KDD Cup}}& \multicolumn{2}{c}{{LDN SmtMeter}}& \multicolumn{2}{c}{{M4 (H)}}

 \\
  \cmidrule(lr){2-3} \cmidrule(lr){4-5}\cmidrule(lr){6-7}\cmidrule(lr){8-9}\cmidrule(lr){10-11}\cmidrule(lr){12-13}

   & \makecell{AUROC} & \makecell{TPR}  & \makecell{AUROC} & \makecell{TPR}
   & \makecell{AUROC} & \makecell{TPR}  & \makecell{AUROC} & \makecell{TPR}
   & \makecell{AUROC} & \makecell{TPR}& \makecell{AUROC} & \makecell{TPR}\\
\midrule
$\log p(x)$
&0.449   &0.014
&0.819   &\underline{0.140}
&0.667   &0.125
&0.472   &0.000
&0.822   &0.017
&0.730   &\underline{0.614}
\\
Rank
&0.473 & 0.035
&0.663   &0.081
&0.727   &0.137
&0.583   &0.000
&0.863   &0.029
&\textbf{\underline{0.800}}   &0.597
\\
LogRank
&0.529   &0.019
&0.800   &0.115
&0.691   &0.156
&0.639   &0.167
&0.869   &0.046
&0.773   &0.614
\\
DetectGPT
&0.404   &0.016
&0.453   &0.025
&0.760   &0.072
&0.444   &0.000
&0.676   &0.013
&0.507   &0.027
\\
Fast-DetectGPT
&0.463   &0.014
&0.733   &0.087
&0.691   &0.165
&0.417   &0.000
&0.796   &0.017
&0.515   &0.220
\\
DetectLLM-LLR
&0.727   &\underline{0.141}
&\underline{0.856}   &0.093
&0.637   &0.087
&0.667   &0.167
&0.863   &0.209
&\underline{0.798}   &0.594
\\
DetectLLM-NPR
&0.515   &0.024
&0.844   &0.072
&0.676   &0.047
&0.417   &0.000
&0.748   &0.039
&0.744   &0.548
\\
DNA-GPT
&0.599   &0.038
&0.808   &\textbf{\underline{0.162}}
&0.772   &0.022
&0.500   &{0.500}
&0.807   &0.029
&0.595   &0.107
\\
Intrinsic Dimension
&0.545   &0.068
&0.867   &0.018
&0.733   &0.006
&0.800   &0.200
&0.603   &0.018
&0.461   &0.415
\\
FourierGPT
&0.485	&0.019
&0.360	&0.000
&0.246	&0.006
&0.778	&\underline{0.667}
&0.534	&0.014
&0.623	&0.094

\\
Binocular
&0.278	&0.005
&0.284	&0.000
&0.708	&0.009
&0.389	&0.000
&0.762	&0.022
&0.491	&0.002

\\
\midrule
UCE-Entropy
&\underline{0.766}   &0.070
&0.797   &0.037
&0.803   &\underline{0.349}
&\textbf{\underline{0.972}}   &\textbf{\underline{0.833}}
&\underline{0.924}   &\textbf{\underline{0.467}}
&0.686   &0.384
\\
\quad\ -Max Prob
&0.728   &0.084
&0.756   &0.019
&\underline{0.818}   &0.343
&\textbf{\underline{0.972}}   &\textbf{\underline{0.833}}
&0.879   &0.341
&0.682   &0.382
\\

\quad\ -Variance
&\textbf{\underline{0.826}}	&\textbf{\underline{0.154}}
&\textbf{\underline{0.883}}	&0.118
&\textbf{\underline{0.836}}	&\textbf{\underline{0.352}}
&\underline{0.806}	&{0.500}
&\textbf{\underline{0.935}}	&\underline{0.351}
&0.661	&\textbf{\underline{0.616}}

\\

\bottomrule
\end{tabular}
}

% part 2
\vspace{1ex}

\resizebox{\linewidth}{!}{
\begin{tabular}{@{}lcccccccccccc|cc@{}}
\toprule
 \multicolumn{1}{l}{\multirow{2}{*}{\makecell{\\ \textbf{Dataset}}}}
 &  \multicolumn{2}{c}{{Pedestrian}}&  \multicolumn{2}{c}{{Rideshare}}&  \multicolumn{2}{c}{{Taxi}}&  \multicolumn{2}{c}{{Temperature}}& \multicolumn{2}{c}{{Uber TLC (D)}}& \multicolumn{2}{c}{{Uber TLC (H)}}&\multicolumn{2}{c}{Top-2 Count}
 \\
  \cmidrule(lr){2-3} \cmidrule(lr){4-5}\cmidrule(lr){6-7}\cmidrule(lr){8-9}\cmidrule(lr){10-11}\cmidrule(lr){12-13}\cmidrule(lr){14-15}

   & \makecell{AUROC} & \makecell{TPR}  & \makecell{AUROC} & \makecell{TPR}
   & \makecell{AUROC} & \makecell{TPR}  & \makecell{AUROC} & \makecell{TPR}
   & \makecell{AUROC} & \makecell{TPR}& \makecell{AUROC} & \makecell{TPR}
   & \makecell{AUROC} & \makecell{TPR}\\
\midrule
$\log p(x)$
&0.985	&0.848
&0.648	&0.000
&0.854	&0.039
&0.864	&0.001
&0.872	&0.000
&0.778	&0.031
&0      &2
\\
Rank
&0.896	&0.561
&0.636	&0.000
&0.897	&0.147
&0.868	&0.002
&0.851	&0.000
&0.821	&0.015
&1      &0
\\
LogRank
&0.981	&0.894
&0.652	&0.000
&0.921	&0.186
&0.877	&0.011
&0.879	&0.000
&\underline{0.871}	&0.034
&1      &0
\\
DetectGPT
&0.174	&0.015
&0.617	&0.000
&0.653	&0.008
&0.858	&0.000
&0.311	&0.000
&0.666	&0.057
&0      &0
\\
Fast-DetectGPT
&0.954	&0.591
&0.636	&0.000
&0.812	&0.010
&0.879	&0.000
&0.877	&0.000
&0.774	&0.050
&0      &0
\\
DetectLLM-LLR
&\textbf{\underline{0.992}}	&\textbf{\underline{0.970}}
&0.659	&\textbf{\underline{0.090}}
&0.944	&0.425
&0.867	&0.405
&0.892	&0.271
&\textbf{\underline{0.883}}	&0.431
&4      &3
\\
DetectLLM-NPR
&0.984	&0.788
&0.619	&0.007
&0.722	&0.003
&0.836	&0.003
&0.846	&0.000
&0.772	&0.019
&0      &0
\\
DNA-GPT
&0.747	&0.091
&0.417	&0.000
&0.444	&0.000
&0.528	&0.000
&0.592	&0.000
&0.460	&0.027
&0      &1
\\
Intrinsic Dimension
&0.932	&0.343
&0.521	&0.015
&0.683	&0.013
&0.572	&0.005
&0.724	&0.252
&0.561	&0.074
&0      &0
\\
FourierGPT
&0.366	&0.000
&0.590	&0.012
&0.434	&0.006
&0.626	&0.011
&0.637	&0.008
&0.667	&0.038
&0      &1
\\
Binocular
&0.439	&0.000
&0.569	&0.004
&0.446	&0.002
&0.842	&0.144
&0.662	&0.004
&0.626	&0.015
&0      &0
\\
\midrule
UCE-Entropy
&\underline{0.986}	&0.909
&\underline{0.700}	&0.039
&\textbf{\underline{0.963}}	&\underline{0.470}
&\textbf{\underline{0.960}}	&\textbf{\underline{0.817}}
&\underline{0.926}	&\underline{0.500}
&0.773	&\underline{0.489}
&8      &6
\\
\quad\ -Max Prob
&0.985	&0.909
&0.649	&\underline{0.088}
&\underline{0.960}	&\textbf{\underline{0.483}}
&\underline{0.935}	&\underline{0.711}
&\textbf{\underline{0.930}}	&\textbf{\underline{0.576}}
&0.790	&\textbf{\underline{0.519}}
&5      &6
\\

\quad\ -Variance
&0.980	&\underline{0.924}
&\textbf{\underline{0.733}}	&0.026
&0.936	&0.377
&0.924	&0.418
&0.898	&0.015
&0.770	&0.046
&6      &5
\\

\bottomrule
\end{tabular}
}

\caption{
Full result of AUROC and TPR (at 1\% FPR) for detecting samples on In-Distribution Datasets. The \textbf{\underline{best}} and \underline{second-best} results are highlighted.
}
\label{tab:In-domain-64-full-1}
\end{table*}

\begin{table*}[t!]
% \small
\centering

\tiny                                % 最小字体
\setlength\tabcolsep{2pt}            % 最小左右间距
\renewcommand{\arraystretch}{0.65}   % 紧凑行高
\captionsetup{skip=3pt}              % 减少 caption 与表格间距

\resizebox{\linewidth}{!}{
\begin{tabular}{@{}lcccccccccccccc@{}}
\toprule
\multicolumn{1}{l}{\multirow{2}{*}{\makecell{\\ \textbf{Dataset}}}}
 &  \multicolumn{2}{c}{{Aus El Demand}}&  \multicolumn{2}{c}{{CIF 2016}}&  \multicolumn{2}{c}{{Covid Death}}&  \multicolumn{2}{c}{{ERCOT}}& \multicolumn{2}{c}{{ETTh}}& \multicolumn{2}{c}{{ETTm}}& \multicolumn{2}{c}{{Exchange Rate}}

 \\
  \cmidrule(lr){2-3} \cmidrule(lr){4-5}\cmidrule(lr){6-7}\cmidrule(lr){8-9}\cmidrule(lr){10-11}\cmidrule(lr){12-13}\cmidrule(lr){14-15}

   & \makecell{AUROC} & \makecell{TPR}  & \makecell{AUROC} & \makecell{TPR}
   & \makecell{AUROC} & \makecell{TPR}  & \makecell{AUROC} & \makecell{TPR}
   & \makecell{AUROC} & \makecell{TPR}  & \makecell{AUROC} & \makecell{TPR}
   & \makecell{AUROC} & \makecell{TPR}\\
\midrule
$\log p(x)$
&0.480	&\underline{0.400}
&0.872	&0.361
&{0.576}	&{\underline{0.098}}
&0.734	&0.375
&0.653	&0.286
&0.209	&0.071
&0.859	&0.000

\\
Rank
&\textbf{\underline{0.880}}	&\underline{0.400}
&0.861	&0.569
&{0.576}	&0.026
&0.742	&0.375
&0.497	&0.214
&0.273	&0.071
&\underline{0.883}	&0.250

\\
LogRank
&0.600	&0.200
&0.889	&0.639
&0.574	&0.030
&\underline{0.805}	&\textbf{\underline{0.625}}
&0.658	&0.286
&0.260	&0.143
&0.875	&0.250

\\
DetectGPT
&0.560	&\underline{0.400}
&0.697	&0.017
&0.573	&{0.045}
&0.672	&0.250
&0.505	&0.000
&0.230	&0.000
&0.859	&0.000

\\
Fast-DetectGPT
&0.520	&\underline{0.400}
&\underline{0.908}	&\textbf{\underline{0.700}}
&0.561	&0.026
&0.641	&0.250
&0.684	&0.357
&0.209	&0.000
&0.828	&0.000

\\
DetectLLM-LLR
&0.680	&\underline{0.400}
&0.905	&0.600
&0.539	&0.015
&\textbf{\underline{0.828}}	&0.375
&\underline{0.923}	&0.357
&0.413	&0.143
&0.859	&\textbf{\underline{0.750}}

\\
DetectLLM-NPR
&0.520	&\underline{0.400}
&\textbf{\underline{0.915}}	&0.500
&{{0.586}}	&0.004
&0.781	&\underline{0.500}
&\textbf{\underline{0.934}}	&\underline{0.571}
&0.173	&0.000
&0.844	&0.125

\\
DNA-GPT
&0.280	&0.000
&0.481	&0.032
&0.374	&0.004
&0.516	&0.250
&0.811	&\textbf{\underline{0.643}}
&\textbf{\underline{0.684}}	&\textbf{\underline{0.358}}
&0.250	&0.000

\\
Intrinsic Dimension
&0.600	&\underline{0.400}
&0.744	&0.250
&0.474	&0.022
&0.440	&0.000
&0.630	&0.300
&0.540	&0.100
&0.400	&0.000
\\
FourierGPT
&0.720	&\underline{0.400}
&0.644	&0.014
&\textbf{\underline{0.653}}	&0.004
&0.719	&0.125
&0.414	&0.000
&0.541	&0.071
&0.797	&0.250
\\
Binocular
&0.200	&0.000
&0.641	&0.028
&\underline{0.618}	&\textbf{\underline{0.139}}
&0.781	&\textbf{\underline{0.625}}
&0.480	&0.214
&0.541	&0.071
&0.266	&0.125

\\
\midrule
UCE-Entropy
&0.680	&\underline{0.400}
&\textbf{\underline{0.915}}	&\underline{0.625}
&0.540	&0.023
&0.586	&0.125
&0.770	&\underline{0.571}
&\underline{0.561}	&\underline{0.214}
&\textbf{\underline{0.891}}	&\underline{0.375}

\\
\quad\ -Max Prob
&\underline{0.760}	&\textbf{\underline{0.600}}
&0.876	&0.611
&0.543	&0.023
&0.555	&0.125
&0.694	&0.500
&0.546	&\underline{0.214}
&0.859	&\underline{0.375}

\\

\quad\ -Variance
&0.640	&\underline{0.400}
&0.902	&0.431
&0.451	&0.038
&0.555	&0.125
&0.760	&\textbf{\underline{0.643}}
&0.536	&\underline{0.214}
&0.016	&0.000
\\

\bottomrule
\end{tabular}
}

% part 2
\vspace{1ex}
\resizebox{\linewidth}{!}{
\begin{tabular}{@{}lcccccccccccccc@{}}
\toprule
\multicolumn{1}{l}{\multirow{2}{*}{\makecell{\\ \textbf{Dataset}}}}
 &  \multicolumn{2}{c}{{Fred MD}}&  \multicolumn{2}{c}{{Hospital}}&  \multicolumn{2}{c}{{M1 (M)}}&  \multicolumn{2}{c}{{M1(Q)}}& \multicolumn{2}{c}{{M3 (M)}}& \multicolumn{2}{c}{{M3 (Q)}}& \multicolumn{2}{c}{{M4(Q)}}

 \\
  \cmidrule(lr){2-3} \cmidrule(lr){4-5}\cmidrule(lr){6-7}\cmidrule(lr){8-9}\cmidrule(lr){10-11}\cmidrule(lr){12-13}\cmidrule(lr){14-15}

   & \makecell{AUROC} & \makecell{TPR}  & \makecell{AUROC} & \makecell{TPR}
   & \makecell{AUROC} & \makecell{TPR}  & \makecell{AUROC} & \makecell{TPR}
   & \makecell{AUROC} & \makecell{TPR}  & \makecell{AUROC} & \makecell{TPR}
   & \makecell{AUROC} & \makecell{TPR}\\
\midrule
$\log p(x)$
&\textbf{\underline{0.709}}	&0.075
&\underline{0.915}	&0.527
&0.783	&0.000
&0.372	&0.000
&\textbf{\underline{0.758}}	&0.118
&0.278	&0.000
&0.235	&0.000

\\
Rank
&\underline{0.706}	&0.121
&0.882	&0.437
&0.735	&0.000
&0.301	&0.000
&0.665	&0.054
&0.230	&0.000
&\textbf{\underline{0.993}}	&\textbf{\underline{0.790}}

\\
LogRank
&0.697	&0.121
&\textbf{\underline{0.917}}	&0.562
&0.788	&0.000
&0.317	&0.000
&0.717	&0.148
&0.242	&0.000
&0.238	&0.000

\\
DetectGPT
&0.687	&0.028
&0.872	&0.029
&0.726	&0.006
&{0.645}	&0.000
&\underline{0.736}	&0.092
&0.398	&0.000
&0.000	&0.000

\\
Fast-DetectGPT
&\textbf{\underline{0.709}}	&0.084
&0.909	&0.544
&\textbf{\underline{0.829}}	&0.008
&0.610	&0.000
&0.720	&0.197
&0.329	&0.000
&0.000	&0.000

\\
DetectLLM-LLR
&0.682	&\textbf{\underline{0.252}}
&0.906	&0.437
&0.779	&{0.029}
&0.399	&0.000
&0.624	&0.160
&0.298	&0.000
&0.421	&0.000

\\
DetectLLM-NPR
&0.694	&0.075
&0.888	&0.182
&0.806	&0.000
&0.558	&0.000
&0.707	&0.117
&0.307	&0.000
&0.230	&0.000

\\
DNA-GPT
&0.309	&0.009
&0.265	&0.000
&0.557	&0.000
&0.398	&0.000
&0.469	&0.009
&0.408	&0.000
&0.446	&0.000

\\
Intrinsic Dimension
&0.607	&0.127
&0.436	&0.026
&\underline{0.807}	&0.071
&\textbf{\underline{0.898}}	&\textbf{\underline{0.248}}
&0.729	&0.105
&\textbf{\underline{0.941}}	&\textbf{\underline{0.453}}
&{0.670}	&{0.227}
\\
FourierGPT
&0.520	&0.000
&0.584	&0.017
&0.535	&0.006
&0.358	&0.005
&0.600	&0.071
&0.323	&0.000
&\underline{0.981}	&\underline{0.623}
\\
Binocular
&0.443	&0.028
&0.386	&0.004
&0.594	&\textbf{\underline{0.044}}
&\underline{0.728}	&0.005
&0.593	&0.130
&\underline{0.747}	&\underline{0.012}
&0.877	&0.000

\\
\midrule
UCE-Entropy
&0.582	&\underline{0.234}
&0.903	&\underline{0.660}
&0.785	&0.008
&0.585	&0.005
&0.616	&\underline{0.233}
&0.449	&0.000
&0.345	&0.000

\\
\quad\ -Max Prob
&0.567	&0.215
&0.898	&\textbf{\underline{0.681}}
&0.768	&{\underline{0.032}}
&0.624	&\underline{0.025}
&0.595	&\textbf{\underline{0.256}}
&{0.491}	&{0.001}
&0.356	&0.000

\\

\quad\ -Variance
&0.449	&0.121
&0.894	&0.627
&0.744	&0.000
&0.325	&0.000
&0.614	&0.124
&0.167	&0.000
&0.382	&0.000

\\

\bottomrule
\end{tabular}
}

% part 3
\vspace{1ex}

\resizebox{\linewidth}{!}{
\begin{tabular}{@{}lcccccccccccc|cc@{}}
\toprule
 \multicolumn{1}{l}{\multirow{2}{*}{\makecell{\\ \textbf{Dataset}}}}
 &  \multicolumn{2}{c}{{M5}}&  \multicolumn{2}{c}{{NN5 (W)}}&  \multicolumn{2}{c}{{Tourism (M)}}&  \multicolumn{2}{c}{{Tourism (Q)}}& \multicolumn{2}{c}{{Traffic}}& \multicolumn{2}{c}{{Weather}}&\multicolumn{2}{c}{Top-2 Count}
 \\
  \cmidrule(lr){2-3} \cmidrule(lr){4-5}\cmidrule(lr){6-7}\cmidrule(lr){8-9}\cmidrule(lr){10-11}\cmidrule(lr){12-13}\cmidrule(lr){14-15}

   & \makecell{AUROC} & \makecell{TPR}  & \makecell{AUROC} & \makecell{TPR}
   & \makecell{AUROC} & \makecell{TPR}  & \makecell{AUROC} & \makecell{TPR}
   & \makecell{AUROC} & \makecell{TPR}& \makecell{AUROC} & \makecell{TPR}
   & \makecell{AUROC} & \makecell{TPR}\\
\midrule
$\log p(x)$
&0.695	&0.009
&\textbf{\underline{1.000}}	&\textbf{\underline{1.000}}
&0.902	&0.005
&0.875	&0.002
&0.567	&0.074
&0.641	&0.002
&4       &3
\\
Rank
&0.748	&0.005
&0.901	&0.622
&0.856	&0.161
&0.854	&0.000
&0.624	&0.079
&0.651	&0.001
&4      &2
\\
LogRank
&0.758	&0.006
&\textbf{\underline{1.000}}	&\textbf{\underline{1.000}}
&0.909	&0.232
&0.888	&0.000
&0.592	&0.082
&0.661	&0.001
&3      &2
\\
DetectGPT
&0.649	&0.004
&\textbf{\underline{1.000}}	&\underline{0.991}
&0.912	&0.003
&0.895	&0.007
&0.636	&0.048
&0.646	&0.003

&2      &2
\\
Fast-DetectGPT
&0.659	&0.005
&\underline{0.999}	&\underline{0.991}
&0.913	&0.011
&0.903	&0.012
&0.629	&0.013
&0.647	&0.005

&4     &3
\\
DetectLLM-LLR
&\textbf{\underline{0.816}}	&\textbf{\underline{0.145}}
&0.979	&0.730
&0.909	&0.123
&0.807	&0.027
&0.526	&0.052
&0.811	&0.066

&2      &4
\\
DetectLLM-NPR
&0.654	&0.005
&0.998	&0.946
&0.880	&0.005
&0.821	&0.000
&0.496	&0.049
&0.644	&0.006

&2      &2
\\
DNA-GPT
&0.510	&0.003
&0.694	&0.081
&0.725	&0.000
&0.657	&0.000
&0.818	&0.065
&0.217	&0.002

&1      &2
\\
Intrinsic Dimension
&0.572	&0.001
&0.662	&0.050
&0.635	&0.070
&0.643	&0.028
&0.695	&0.014
&\textbf{\underline{0.940}}	&\textbf{\underline{0.595}}

&4      &4
\\
FourierGPT
&0.621	&0.015
&0.452	&0.000
&0.583	&0.008
&0.525	&0.005
&0.409	&0.001
&0.836	&0.025
&2      &2
\\
Binocular
&0.682	&0.009
&0.717	&0.063
&0.614	&0.003
&0.790	&\underline{0.084}
&0.670	&0.082
&0.874	&0.004
&3      &5

\\
\midrule
UCE-Entropy
&0.726	&\underline{0.119}
&\textbf{\underline{1.000}}	&\textbf{\underline{1.000}}
&\textbf{\underline{0.973}}	&\underline{0.620}
&\underline{0.921}	&{0.033}
&\textbf{\underline{0.895}}	&\underline{0.375}
&\underline{0.899}	&0.108

&8      &12
\\
\quad\ -Max Prob
&0.740	&0.118
&\textbf{\underline{1.000}}	&\textbf{\underline{1.000}}
&\underline{0.970}	&0.604
&\textbf{\underline{0.936}}	&\textbf{\underline{0.199}}
&\textbf{\underline{0.895}}	&\textbf{\underline{0.421}}
&\underline{0.899}	&\underline{0.119}

&6      &11
\\

\quad\ -Variance
&\underline{0.807}	&0.018
&\textbf{\underline{1.000}}	&\textbf{\underline{1.000}}
&\textbf{\underline{0.973}}	&\textbf{\underline{0.658}}
&0.882	&0.000
&\underline{0.859}	&0.147
&0.721	&0.001

&4      &5
\\

\bottomrule
\end{tabular}
}

\caption{
Full result of AUROC and TPR (at 1\% FPR) for detecting samples on Zero-Shot Datasets. The \textbf{\underline{best}} and \underline{second-best} results are highlighted.
}
\label{tab:0-shot-64-full-1}
\end{table*}

\section{Comparison between Textual and Temporal Data in Model Generation}
Time Series Large Models share foundational similarities with LLMs as they both rely on Transformer based recursive prediction, yet their tokens differ fundamentally.
Text tokens represent discrete semantic words or subwords whose meanings is derived through their semantic mapping to real-world entities or abstract concepts, and is reflected by the contextual relationships within texts.
In contrast, continuous real-valued observations are discretized into a bounded ordered set, which preserves both the numerical ordering (e.g., token ``1000'' is 300 units greater than ``700'') and metric similarity (e.g., token ``490'' is closer to ``500'' than ``560'').

Time series tokens inherit the order and metric structure of real values, thereby carrying non-arithmetic semantic information.
TSLMs implicitly learn these metric-based semantic relationships between tokens from the continuous nature of their training data, as high continuity of real-valued sequences preserves this underlying semantics, enabling the models to “internalize” and exploit this information in their forecasts.
This means that each temporal token has many ``semantic neighbors'' with high metrical similarity to define a neighborhood.
For instance, temperature $25.1^\circ \mathrm{C}$, $25.2^\circ \mathrm{C}$ and $25.3^\circ \mathrm{C}$ all lie within a small numeric range.
Their high similarity implies low information difference within this neighborhood.
In other words, although a massive vocabulary may span thousands of discrete values, only a small subset of these local neighborhoods defines clear semantic boundaries, which illustrates the sparse information density.

\section{Reevaluating Generation Detection for Time Series via Uncertainty}
\label{apdx:reevaluate}
In this section, we reevaluate the baseline methods, which share two key traits
with UCE: performance improves on longer series and on in distribution data.
This suggests a lens of uncertainty from which to assess these baselines.

Point-wise probabilities fail to capture overall uncertainty in time series
forecasts, and relevant methods (e.g. $\log p$ and DNA-GPT) tend to
underperform.
The probability of individual output token is less effective in reflecting the
model's entire internal probability distribution.
The arbitrary nature of time series values causes a smoothed internal probability distribution at each time point, and the token-wise probability
differences are low.
In terms of uncertainty, a low $p(x_t)$ may reside in the tail of a peaked
distribution or the body of a flat one, offering minimal discriminative power.
However, a higher probability is more informative by approximating the maximum probability.

Perturbation-based methods (DetectGPT, Fast DetectGPT, NPR) detect generation by
measuring how output scores change when a few tokens are swapped with semantic
neighbors.
In time series, replacing a small fraction of values adds noise of variance
$\sigma_{N}^{2}$ on top of the inherent series variance $\sigma_{I}^{2}$.
The resulting score change—whether a likelihood ratio or a log rank
difference—depends critically on $\sigma_{N}^{2}$ against $\sigma_{I}^{2}$.
If $\sigma_N^2\ll\sigma_{I}^2$, the perturbation is too weak to detect; if
$\sigma_N^2\gg \sigma_I^2$, the noise overwhelms the original pattern.
Moreover, varing $\sigma_I^2$ across datasets makes it difficult to choose a
single perturbation strength that works universally.

Rank-based metrics invert token probability: a lower numeric rank indicates
lower uncertainty.
In principle, rank captures how concentrated the predictive distribution is
around its top candidates.
However, in time series forecasting most token probabilities are uniformly low, so ranks cluster at high values with little variation, making rank alone a weak
signal for distinguishing generated from real data.

\end{document}